\documentclass[10pt,a4paper]{article}
\usepackage[utf8]{inputenc}

\usepackage[a4paper,top=3cm,bottom=2cm,left=2.5cm,right=2.5cm,marginparwidth=1.75cm]{geometry}

\usepackage{amsmath}
\usepackage{amsthm}
\usepackage{amssymb}
\usepackage{natbib}
\usepackage{hyperref}

\usepackage{algorithm}
\usepackage{algorithmic}

\usepackage{thmtools}
\usepackage{thm-restate}

\usepackage{subfigure}
\usepackage{enumitem}

\usepackage[usenames,dvipsnames]{xcolor}
\usepackage{tcolorbox}

\usepackage{xspace}
\DeclareRobustCommand{\eg}{e.g.,\@\xspace}
\DeclareRobustCommand{\ie}{i.e.,\@\xspace}

\DeclareRobustCommand{\wrt}{w.r.t.\@\xspace}

\def\eqdef{:=}

\def\Regret{\mathrm{Reg}}

\newtheorem{proposition}{Proposition}
\newtheorem{corollary}[proposition]{Corollary}
\newtheorem{lemma}[proposition]{Lemma}
\newtheorem{theorem}[proposition]{Theorem}

\newtheorem{remark}{Remark}
\newtheorem{assumption}{Assumption}

\usepackage{bbm}

\newcommand{\wt}[1]{\widetilde{#1}}
\newcommand{\wh}[1]{\widehat{#1}}
\newcommand{\wb}[1]{\overline{#1}}
\newcommand{\Olog}{\wt{\mathcal{O}}}

\usepackage{mathtools}
\DeclarePairedDelimiter\br{(}{)}
\DeclarePairedDelimiter\brs{[}{]}
\DeclarePairedDelimiter\brc{\{}{\}}

\DeclarePairedDelimiter\abs{\lvert}{\rvert}
\DeclarePairedDelimiter\norm{\lVert}{\rVert}
\DeclarePairedDelimiter\inner{\langle}{\rangle}

\newcommand{\E}{\mathbb{E}}

\newcommand{\F}{\mathcal{F}}
\newcommand{\N}{\mathcal{N}}
\newcommand{\A}{{\bf A}}
\newcommand{\ind}{\mathbbm{1}}

\newcommand{\uv}{{\bf u}}
\newcommand{\tv}{{\bf t}}
\newcommand{\xv}{{\bf x}}

\newcommand{\bv}{{\bf b}}

\newcommand{\gv}{{\bf g}}
\newcommand{\fopt}{f_{\mathrm{opt}}}
\newcommand{\qopt}{q_{\mathrm{opt}}}

\DeclareMathOperator{\aset}{\mathcal{A}}
\DeclareMathOperator{\sset}{\mathcal{S}}

\DeclareMathOperator*{\argmin}{\arg\,\min}

\def\filt{\mathcal{F}_{k-1}}

\usepackage[colorinlistoftodos]{todonotes}

\DeclareRobustCommand{\cucrl}{\texttt{OptCMDP}\@\xspace}
\DeclareRobustCommand{\cucbvi}{\texttt{OptCMDP-bonus}\@\xspace}
\DeclareRobustCommand{\dualcmdp}{\texttt{OptDual-CMDP}\@\xspace}
\DeclareRobustCommand{\primaldualcmdp}{\texttt{OptPrimalDual-CMDP}\@\xspace}

\DeclareRobustCommand{\ucrl}{\texttt{UCRL2}\@\xspace}

\newcommand*\dkl[2]{d_{KL}(#1||#2)}
\newcommand*\bregman[2]{B_\omega\left(#1,#2\right)}

\title{Exploration-Exploitation in Constrained MDPs}
\author{Yonathan Efroni$^1$ \and Shie Mannor$^1$ \and Matteo Pirotta$^2$}
\date{%
    $^1$Technion, Israel, $^2$Facebook AI Research\\[2ex]%
    \today
}
\begin{document}

\maketitle

\begin{abstract}
       In many sequential decision-making problems, the goal is to optimize a utility function while satisfying a set of constraints on different utilities. This learning problem is formalized through Constrained Markov Decision Processes (CMDPs).
        In this paper, we investigate the exploration-exploitation dilemma in CMDPs.
        While learning in an unknown CMDP, an agent should trade-off exploration to discover new information about the MDP, and exploitation of the current knowledge to maximize the reward while satisfying the constraints. 
        While the agent will eventually learn a good or optimal policy, we do not want the agent to violate the constraints too often during the learning process.
        In this work, we analyze two approaches for learning in CMDPs. The first approach leverages the linear formulation of CMDP to perform optimistic planning at each episode. The second approach leverages the dual formulation (or saddle-point formulation) of CMDP to perform incremental, optimistic updates of the primal and dual variables. We show that both achieves sublinear regret w.r.t.\ the main utility while having a sublinear regret on the constraint violations. That being said, we highlight a crucial difference between the two approaches; the linear programming approach results in stronger guarantees than in the dual formulation based approach. 
\end{abstract}

{\small
\tableofcontents
}

\section{Introduction}
Markov Decision Processes (MDPs) have been successfully used to model several applications, including video games, robotics, recommender systems and many more.
However, MDPs do not take into account additional constrains that can affect the optimal policy and the learning process.
For example, while driving, we want to reach our destination but we want to avoid to go off-road, overcome the speed limits, collide with other cars~\citep{garcia2015comprehensive}.
Constrained MDPs~\citep{altman1999constrained} extend MDPs to handle constraints on the long term performance of the policy.
A learning agent in a CMDP has to maximize the cumulative reward while satisfying all the constraints.
Clearly, the optimal solution of a CMDP is different than the one of an MDP when at least one constraint is active. Then, the optimal policy, among the set of policies which satisfies the constraint, is stochastic.

In this paper, we focus on the online learning problem of CMDPs.
While interacting with an \emph{unknown} MDP, the agent has to trade-off exploration to gather information about the system and exploration to maximize the cumulative reward. Performing such exploration in a CMDP may be \emph{unsafe} since may lead to numerous violations of the constraints. Since the constraints depend on the long term performance of the agent and the CMDP is unknown, the agent cannot exactly evaluate the constraints. It can only exploit the current information to build an estimate of the constraints.
The objective is thus to design an algorithm with a small number of violations of the constraints.

\paragraph{Objective and Contributions.} The objective of this technical report is to provide an \textbf{extensive analysis} of exploration strategies for tabular constrained MDPs with finite-horizon cost.
Similar to~\citep{agrawal2019bandits}, we allow the agent to violate the constraints over the learning process but we require the cumulative cost of constraint violations to be small (\ie sublinear).
Opposite to~\citep{Zheng2020constrained}, we consider the CMDP to be unknown, \ie the agent does not know the transition kernel, the reward function and the constraints.

The performance of the learning agent is measured through the regret, that accounts for the difference in executing the optimal policy and the learning agent.
We define two regrets: \emph{i)} the regret \wrt to the main objective (as in standard MDP), \emph{ii)} the regret \wrt the constraint violations.
These terms account for both convergence to the optimal policy and cumulative cost for violations of the constraints.
We introduce and  analyze the following exploration strategies:
\begin{description}
        \item[\cucrl] leverages the ideas of \ucrl~\citep{jaksch2010near}. At each episodes, it builds a set of plausible CMDPs compatible with the observed samples, and plays the optimal policy of the CMDP with the lowest cost (\ie optimistic CMDP). To solve this planning problem, we introduce an extended linear programming (LP) problem in the space of occupancy measures. The important property is that there always exists a feasible solution of this extended LP. 
        \item[\cucbvi] merges the uncertainties about costs and transitions used by \cucrl into an exploration bonus. As a consequence, \cucbvi solves a single (optimistic) CMDP rather than planning in the space of plausible CMDPs. This leads to a more computationally efficient algorithm. In fact, this planning problem can be solved through an LP with $O(SAH)$ constraints and decision variables, a factor $O(S)$ smaller than the LP solved by \cucrl.
        \item[\dualcmdp] leverages the saddle-point formulation of constrained MDP~\citep[\eg][]{altman1999constrained}. It solves this problem using an optimistic version of the dual projected sub-gradient algorithm~(e.g., \citealt{beck2017first}).
                At each episode, \dualcmdp solves an optimistic MDP defined using the estimated Lagrangian multiplier.
                Then, it uses the computed solution to update the Lagrange multipliers via projected sub-gradient.
                The main advantage of this algorithm needs to solve a simple optimistic planning problem for MDPs (rather than for CMDPs).
        \item[\primaldualcmdp] exploits a primal-dual algorithm to solve the saddle-point problem associated to a CMDP. It performs incremental updates both on the primal and dual variables. It uses mirror descent to update the Q-function (thus the policy) and projected subgradient descent to update the Lagrange multipliers. Similarly to \cucbvi, this algorithm exploits an exploration bonus for both cost and constraint costs. This allows to use a simple dynamic programming approach to compute the Q-functions (no need to solve a constrained optimization problem).
\end{description}

For all the proposed algorithms, we provide an upper-bound to the regret and the cumulative constraint violations (see Tab.~\ref{tab:tabbounds}).
While the incremental algorithms (\dualcmdp and \primaldualcmdp) may be more amenable for practical applications, they present limitations from a theoretical perspective.
In fact, we were able to prove weaker guarantees for the Lagrangian approaches compared to UCRL-like algorithms (\ie \cucrl and \cucbvi).
While for UCRL-like algorithms we can bound the sum of positive errors, for Lagrangian algorithms we were able to bound only the cumulative (signed) error. This weaker term allows for ``cancellation of errors'' (see discussion in Sec.~\ref{sec:learningobj}). Whether it is possible to provide stronger guarantees is left as an open question.
Despite this, we think that the analysis of Lagrangian approaches is important since it is at the core of many practical algorithms.
For example, the Lagrangian formulation of CMDPs has been used in~\citep{Tessler2019rewardcpo,Paternain2019zeroduality}, but never analyzed from a regret perspective.


\begin{table*}
\begin{center}
        \small
\begin{tabular}{|c|| c | c |  c | }
        \hline
   Algorithm &Optimality Regret & Constraint Regret  \\ \hline \hline
   \cucrl& $\Regret_+\leq \Olog\br*{\sqrt{S\mathcal{N}H^4K}}$ & $\Regret_+\leq\Olog\br*{\sqrt{S\mathcal{N}H^4K}}$ \\ 
   \hline
   \cucbvi & $\Regret_+\leq\Olog\br*{\sqrt{S\mathcal{N}H^4K}}$ & $\Regret_+\leq\Olog\br*{\sqrt{S\mathcal{N}H^4K}}$\\  
   \hline
   \dualcmdp & $ \Regret\leq\Olog\br*{\sqrt{(S\mathcal{N}H^2+\rho^2I)H^2K}}$ & $\Regret\leq\Olog\br*{(1+\frac{1}{\rho})\sqrt{IS\mathcal{N}H^4K}}$ \\ 
   \hline
   \primaldualcmdp & $\Regret\leq\Olog\br*{\sqrt{(S\mathcal{N}H^2+\rho^2I^2H^2)H^2K}}$ & $\Regret\leq\Olog\br*{(1+\frac{1}{\rho})\sqrt{IS\mathcal{N}H^4K} +I\sqrt{H^4K}}$\\
\hline
\end{tabular}
\end{center}
\caption{Summary of the regret bounds obtained in this work. Algorithms \cucrl, \cucbvi, \dualcmdp, \primaldualcmdp are formulated and analyzed in sections~\ref{sec: paper curcl},~\ref{sec: paper curcl bonus},~\ref{sec: paper dual cmdp},~\ref{sec: paper primal dual cmdp}, respectively. The constant term, which is omitted from the table, of \cucbvi is significantly worse than the one of \cucrl. Notice that different types of regrets are bounded (see Section~\ref{sec:setup} for definitions).}
\label{tab:tabbounds}
\end{table*}



\subsection{Related Work}
The problem of online learning under constraints (with guarantees) have been analyzed both in bandits and in RL.
Conservative exploration focuses on the problem of learning an optimal policy while satisfying a constrained \wrt to a predefined baseline policy.
This problem can be seen as a specific instance of CMDPs where the constraint is that the policy should perform (in the long run) better than a predefined baseline policy. 
Conservative exploration has been analyzed both in bandits~\citep{Wu2016conservative,Kazerouni2017conservative,garcelon2020conservativebandits} and in RL~\citep{garcelon2020conservativerl}.
All these algorithms are able to guarantee that the performance of the learning agent is at least as good as the one of the baseline policy with high probability at any time.\footnote{To guarantee this the allow the performance of the learning agent to be $\alpha$-away from the baseline performance.}
While they enjoy strong theoretical guarantees, they performs poorly in practice since are too conservative.
In fact, the idea of these algorithms is to build budget (\eg by playing the baseline policy) in order to be able to take standard exploratory actions.
Concurrently to this paper,~\citep{Zheng2020constrained} has extended conservative exploration to CMDP with average reward objective. They assume that the transition functions are known, but the rewards and costs (\ie the constraints) are unknown. The goal is thus to guarantee that, \emph{at any time}, the policy executed by the agent satisfies the constraints with high probability.
These requirement poses several limitations. Similarly to~\citep{garcelon2020conservativerl}, they need to assume that the MDP is ergodic and that the initial policy is safe (\ie satisfies the constraints). Furthermore, despite the theoretical guarantees, this approach is not practical due to these strong requirements/assumptions.
\citet{agrawal2019bandits} studied the exploration problem for bandits under constraints as well as bandits with knapsack constraints~\citep{badanidiyuru2013bandits}. Algorithms \cucrl and \cucbvi can be understood as generalizing their bandit setting to an CMDP setting. That being said, in the following we derive regret bounds on a stronger type of regret relatively to~\citet{agrawal2019bandits} (see Remark~\ref{remark: two types of regret}).

There are several approaches in the literature that have focused on (approximately) solving CMDPs.
These methods are mainly based on Lagrangian-formulation~\citep{bhatnagar2012online,Chow2017risk,Tessler2019rewardcpo,Paternain2019zeroduality} or constrained optimization~\citep{Achiam2017cpo}.
Lagrangian-based methods formulate the CMDP optimization problem as a saddle-point problem and optimize it using primal-dual algorithms.
While these algorithms may eventually converge to the true policy, they have no guarantees on the policies recovered during the learning process.
Constrained Policy Optimization (CPO)~\citep{Achiam2017cpo} leverages the intuition behind conservative approaches~\citep[\eg][]{kakade2002approximately} to force the policy to improve overtime. This is a practical implementation of conservative exploration where the baseline policy is updated at each iteration.

Another way to solve CMDPs and guarantee safety during learning is through Lyapunov functions~\citep{Chow2018lyapunov,chow2019lyapunovcont}.
Despite the fact that some of these algorithms are approximately safe over the learning process, analysing the convergence is challenging and the regret analysis is lacking.
Other approaches use Gaussian processes to model the dynamics and/or the value function~\citep{Berkenkamp2017safe,Wachi2018safe,Koller2018learning,Cheng2019end} in order to be able to estimate the constraints and (approximately) guarantee safety over learning.

A related approach is the literature about budget learning in bandits~\citep[\eg][]{Ding2013mabsbudget,Combes2015budget}.
In this setting, the agent is provided with a budget (known and fix in advance) and the learning process is stopped as soon as the budget is consumed.
The goal is to learn how to efficiently handle the budget in order to maximize the cumulative reward.
A widely studied case of budget bandit is bandit with knapsack~\citep[\eg][]{Agrawal2014knapsacks,Badanidiyuru2018knapsacks}.
In our setting, we do not have a ``real'' concept of budget and the length of the learning process does not depend on the total cost of constraint violations.
This paper is also related to learning with fairness constraints~\citep[\eg][]{Joseph2016banditfair}. Similarly to conservative exploration, fairness constraints can be sometimes formulated as a specific instance of CMDPs.

\section{Preliminaries}\label{sec:setup}
We start introducing finite-horizon Markov Decision Processes (MDPs) and their constrained version.
We define $[N] \eqdef \brc*{1,\ldots,N},\;$ for all $N\in \mathbb{N}$.

\subsection{Finite-Horizon Constrained MDPs}

\paragraph{Finite Horizon MDPs.}
We consider finite-horizon MDPs with time-dependent dynamics~\citep{puterman1994markov}. A finite-horizon constraint MDP is defined by the tuple $\mathcal{M} = \br*{\mathcal{S},\mathcal{A}, c, p, s_1,H}$, where $\mathcal{S}$ and $\mathcal{A}$ are the state and action spaces with cardinalities $S$ and $A$, respectively. The non-stationary immediate cost for taking an action $a$ at state $s$ is a random variable $C_h(s,a)\in\brs*{0,1}$ with expectation $\E C_h(s,a)=c_h(s,a)$. The transition probability is $p_h(s'\mid s,a)$, the probability of transitioning to state $s'$ upon taking action $a$ at state $s$ at time-step $h$. The initial state in each episode is chosen to be the same state $s_1$ and $H\in \mathbb{N}$ is the {\em horizon}. Furthermore, $\N\eqdef \max_{s,a,h}\left| \brc*{s':p_h(s'\mid s,a)>0} \right|$ is the maximum number of non-zero transition probabilities across the entire state-action pairs.  

A Markov non-stationary randomized policy $\pi = (\pi_1, \pi_2, \ldots, \pi_H) \in \Pi^{\text{MR}}$ where $\pi_i: \mathcal{S} \rightarrow \Delta_A$ maps states to probabilities $\Delta_A$ on the action set $\mathcal{A}$. We denote by $a_h \sim \pi(s_h,h) := \pi_h(s_h)$, the action taken at time $h$ at state $s_h$ according to a policy $\pi$. 
For any $h \in [H]$ and $(s,a) \in \sset \times \aset$, the state-action value function of a non-stationary policy $\pi = (\pi_1, \ldots, \pi_H)$ is defined as 
\[
        Q^{\pi}_{h}(s,a) = c_{h}(s,a) + \mathbb{E} \left[ \sum_{l = h+1}^{H} c_{l}(s_{l}, a_l)\mid s_h=s,a_h=a, \pi, p \right]
\]
where the expectation is over the environment and policy randomness.
The value function is $V^{\pi}_h(s) = \sum_{a} \pi_h(a|s) Q^{\pi}_h(s, a)$.
Since the horizon is finite, under some regularity conditions, \citep{shreve1978alternative}, there always exists an optimal Markov non-stationary deterministic policy $\pi^\star$ whose value and action-value functions are defined as $V^\star_h(s) := V^{\pi^\star}_h(s) = \sup_{\pi} V^{\pi}_h(s)$ and $Q^\star_h(s,a) := Q^{\pi^\star}_h(s,a) = \sup_{\pi} Q^{\pi}_h(s,a)$. 
The Bellman principle of optimality (or Bellman optimality equation) allows to efficiently compute the optimal solution of an MDP using backward induction:
\begin{equation}\label{eq:optimality}
        V^\star_h(s) = \min_{a \in \mathcal{A}} \left\{ c_h(s,a) + \mathbb{E}_{s' \sim p_h(\cdot|s,a)} [V^\star_{h+1}(s')] \right\}, \quad
        Q^\star_h(s,a) = c_h(s,a) + \mathbb{E}_{s' \sim p_h(\cdot|s,a)} [V^\star_{h+1}(s')]
\end{equation}
where $V^\star_{H+1}(s) := 0$ for any $s \in \mathcal{S}$ and $V^\star_h(s) = \min_a Q_h^\star(s,a)$, for all $s \in \sset$.
The optimal policy $\pi^\star_h$ is thus greedy \wrt $V^\star_h$~\citep[\eg][]{puterman1994markov}.
Notice that by boundedness of the cost, for any $h$ and $(s,a)$, all functions $Q^\pi_h$, $V^\pi_h$, $Q^\star_h$, $V^\star_h$ are bounded in $[0, H-h+1]$. 

We can reformulate the optimization problem by using the \emph{occupancy measure}~\citep[\eg][]{puterman1994markov,altman1999constrained}.
The occupancy measure $q^\pi$ of a policy $\pi$ is defined as the set of distributions generated by executing the policy $\pi$ in the finite-horizon MDP $\mathcal{M}$~\citep[\eg][]{Zimin2013online}:
\begin{align*}
    q^\pi_h(s,a;p) \eqdef \E\brs*{\ind\brc*{s_h=s,a_h=a}\mid s_1= s_1,p,\pi} = \Pr\brc*{s_h=s,a_h=a\mid s_1=s_1,p, \pi}.
\end{align*}
For ease of notation, we define the matrix notation $q^\pi(p)\in \mathbb{R}^{HSA}$ where its $(s,a,h)$ element is given by $q^\pi_h(s,a;p)$.
This implies the following relation between the occupancy measure and the value of a policy:
\begin{align}
    V_1^\pi(s_1;p,c) = \sum_{h,s,a}q^{\pi}_h(s,a;p)c_h(s,a) \eqdef c^T q^\pi(p). \label{eq:prelim_v_to_sa_dist}
\end{align}
where $c \in  \mathbb{R}^{HSA}$ such that element  $(s,a,h)$ element is given by $c_h(s,a)$.
\begin{proof}
The value function $ V_1^\pi(s_1;p,c)$ is given by the following equivalent relations.
\begin{align*}
     &\mathbb{E}\left[ \sum_{h =1}^{H} c_{h}(s_{h}, a_h)\mid s_1=s_1, \pi, p \right] =\sum_{h =1}^{H}\mathbb{E}\left[  c_{h}(s_{h}, a_h)\mid s_1=s_1, \pi, p \right]\\
     &=\sum_{h =1}^{H} \sum_{s,a} c_{h}(s, a)\Pr\brc*{s_h=s,a_h=a\mid s_1=s_1,p,\pi}\\
     &\sum_{h =1}^{H} \sum_{s,a} c_{h}(s, a)q^\pi_h(s,a;p) = c^Tq^\pi(p),
\end{align*}
where the first relation holds by linearity of expectation.
\end{proof}

\paragraph{Finite Horizon Constraint  MDPs.} A constraint MDP~\citep{altman1999constrained} is an MDP supplied with a set of $I$ constraints  $\brc*{d_i, \alpha_i}_{i=1}^I$, where $d_i\in \mathbb{R}^{SAH}$ and $\alpha_i\in [0,H]$. The immediate $i^{th}$ constraint when taking an action $a$ from state $s$ at time-step $h$ is random variable $D_i(s,a)\in[0,1]$ with expectation $\E[D_{i,h}(s,a)] = d_{i,h}(s,a)$. The expected cost of the $i^{th}$ constraint violation from state $s$ at time-step $h$ is defined as
\begin{align*}
        V_{h}^\pi(s;p,d_i) \eqdef \E\brs*{\sum_{h'=h}^H d_{i,h'}\br*{s_{h'}, a_{h'}}\mid s_h=s, p, \pi}.
\end{align*}
Similarly to~\eqref{eq:prelim_v_to_sa_dist}, we can rewrite the constraint in terms of occupancy measure:
    $V_{h}^\pi(s;p,d_i) = d_i^T q^\pi(p)$.
    Notice that by boundedness of the constraint cost, for any $h,i$ and $(s,a)$, all functions $Q^\pi_h(s,a;d_i,p)$, $V^\pi_h(s;d_i,p)$, $Q^\star_h(s,a;d_i,p)$, $V^\star_h(s,;d_i,p)$ are bounded in $[0, H-h+1]$. 
The objective of a CMDP is to find a policy minimizing the cost while satisfying all the constraints. Formally,
\begin{equation} \label{eq:cop}
\begin{aligned}
        \pi^\star \in &~\argmin_{\pi\in \Pi^{\text{MR}}} c^T q^\pi(p)\\
                      &\text{s.t.}\; D q^\pi(p) \leq \alpha, 
\end{aligned}
\end{equation}
where $D\in \mathbb{R}^{I\times SAH}$ and $\alpha\in \mathbb{R}^I$ such that
\begin{align*}
    D=\begin{bmatrix}
    d_1^T\\
    \vdots\\
    d_I^T
   \end{bmatrix},\quad \alpha = \begin{bmatrix}
    \alpha_1\\
    \vdots\\
    \alpha_I
   \end{bmatrix},
\end{align*}
The optimal value is the value of $\pi^\star$ from the initial state, i.e., $V^\star_1(s_1) \eqdef V^{\pi^\star}_1(s_1; p,c)$.
\begin{assumption}[Feasibility]\label{assum: feasability}
        The unknown CMDP is feasible, \ie there exists an unknown policy $\pi \in \Pi^{\text{MR}}$ which satisfies the constraints. Thus, an optimal policy exists as well.
\end{assumption}

It is important to stress that the optimal policy of a CMDP may be stochastic~\citep[\eg][]{altman1999constrained}, \ie may not exist an optimal deterministic policy.
In fact, due to the constraints, the Bellman optimality principle, see Eq.~\ref{eq:optimality}, may not hold anymore. This means that we cannot leverage backward induction and the greedy operator. \citet{altman1999constrained} showed that it is possible to compute the optimal policy of a constrained problem by using linear programming. We will review this approach in Sec.~\ref{sec:linearprogramming}.


\subsection{The Learning Problem.} \label{sec:learningobj}
We consider an agent which repeatedly interacts with a CMDP in a sequence of $K$ episodes of fixed length $H$ by playing a non-stationary policy $\pi_k = (\pi_{1k}, \ldots, \pi_{Hk})$ where $\pi_{hk} : \mathcal{S} \rightarrow \Delta_A$. Each episode $k$ starts from the fixed initial state $s_1^k=s_1$. 
The learning agent does not know the transition or reward functions, and it relies on the samples (\ie trajectories) observed over episodes to improve its performance over time.

The performance of the agent is measured using multiple objectives: \emph{i)} the regret  relatively to the value of the best policy, and \emph{ii)} the amount of constraint violations. In sections~\ref{sec: paper curcl} and~\ref{sec: paper curcl bonus} we analyze algorithms with guarantees on the following type of regrets
\begin{align}
    &\Regret_+(K;c)= \sum_{k=1}^{K} \brs*{V_1^{\pi_k}(s_1;p,c) - V_1^\star(s_1)}_+ \label{eq: regret def c +}\\
    &\Regret_+(K;d)= \max_{i\in [I]} \sum_{k=1}^{K} \brs*{V_1^{\pi_k}(s_1;p,d_i) - \alpha_i}_+,  \label{eq: regret def d +}
\end{align}
where $[x]_+ := \max\{0, x\}$. The term $\Regret_+(K;d)$ represents the maximum cumulative cost for \emph{violations} of the constraints.

We later continue and analyze algorithms with reduced computational complexity in sections~\ref{sec: paper dual cmdp} and~\ref{sec: paper primal dual cmdp}. For these algorithms, we supply regret guarantees for all $K'\in[K]$ with respect to a weaker measure of regrets defined as follows.  
\begin{align}
    &\Regret(K;c)= \sum_{k=1}^{K} V_1^{\pi_k}(s_1;p,c) - V_1^\star(s_1) \label{eq: regret def c}\\
    &\Regret(K;d)= \max_{i\in [I]} \brs*{\sum_{k=1}^{K} V_1^{\pi_k}(s_1;p,d_i) - \alpha_i}. \label{eq: regret def d max}
\end{align}
\begin{remark}\label{remark: two types of regret}
Note that in our setting, the immediate regret $V_1^{\pi_k}(s_1;p,c) - V_1^\star(s_1)$ might be negative since policy $\pi_k$ might violate the constraints.
For this reason, bounding the regret as~$\Regret_+(K;c)$ is stronger than bounding~$\Regret_+(K;c)$ in the sense that the a bound on the first implies a bound on the latter; but not vice-versa. 

Similar relation  holds between the two definitions of the constraint violations types of regret; a bound on~$\Regret_+(K;d)$ implies a bound on $\Regret(K;d)$, but the opposite does not holds. In words, a bound on the first implies a bound on the absolute sum of constraint violations where the latter bounds the cumulative constraint violations, and, thus, allows for ``error cancellations''. 

\end{remark}

\subsection{Linear Programming for CMDPs} \label{sec:linearprogramming}
In Sec.~\ref{sec:setup}, we have seen that the cost criteria can be expressed as the expectation of the immediate cost \wrt to the occupancy measure. The convexity and compactness of this space is essential for the analysis of constrained MDPs. We refer the reader to~\citep[][Chap. 3 and 4]{altman1999constrained} for an analysis in infinite horizon problems.


We start stating two basic properties of an occupancy measure $q$. In this section, we remove the dependence on the model $p$ to ease the notation.
It is easy to see that the occupancy measure of any policy $\pi$ satisfies~\citep[\eg][]{Zimin2013online,bhattacharya2017linear}:
\begin{equation}\label{eq:occupancy_space}
        \begin{aligned}
                \sum_{a} q^\pi_h(s,a) &= \sum_{s', a'} p_{h-1}(s|s',a') q^\pi_{h-1}(s',a') && \forall s \in \mathcal{S} \\
                q_h^\pi(s,a) &\geq 0 && \forall s,a
        \end{aligned}
\end{equation}
for all $h \in [H] \setminus \{1\}$.
For $h=1$ and an initial state distribution $\mu$, we have that
\[
        q^\pi_1(s, a) = \pi_1(a|s) \cdot \mu(s) \qquad \forall s,a
\]
Notice that $\sum_{s,a} q^\pi_1(s,a) = 1$. As a consequence, by summing the first constraint in~\eqref{eq:occupancy_space} over $s$ we have that $\sum_{s,a} q_h^\pi(s,a) = 1$, for all $h \in [H]$. Thus the $q^\pi$ satisfying the constraints are probability measures. 
We denote by $\Delta^\mu(\mathcal{M})$ the space of occupancy measures.

Since the set $\Delta^\mu(\mathcal{M})$ can be described by a set of affine constraints, we can state the following property.
Please refer to~\citep[\eg][]{puterman1994markov,altman1999constrained,Mannor2005empirical} for more details.
\begin{proposition} \label{prop: convexity of state action frequencey}
        The set $\Delta^\mu(\mathcal{M})$ of occupancy measure is convex.
\end{proposition}

An important consequence of the linearity of the cost criteria and of the structure of $\Delta(\mathcal{M})$ is that the original control problem can be reduced to a Linear Program (LP) where the optimization variables are measures.
Furthermore, optimal solutions of the LP define the optimal Markov policy through the occupancy measure.
In fact, a policy $\pi^q$ generates an occupancy measure $q \in \Delta(\mathcal{M})$ if 
\[
        \pi^q_h(a|s) = \frac{q_h(s,a)}{\sum_b q_h(s,b)}, \qquad  \forall (s,a, h) \in \mathcal{S} \times \mathcal{A} \times [H].
\]

The constrained problem~\eqref{eq:cop} is equivalent to the LP:
\begin{align*}
        \min_{q}      &\sum_{s,a,h} q_h(s,a) c_h(s,a)\\
        \text{s.t.} 
                      & \sum_{s,a,h} q_h(s,a) d_{i,h}(s,a) \leq \alpha_i && \forall i \in [I]\\ 
                      &\sum_a q_h(s,a) = \sum_{s',a'} p_{h-1}(s|s',a') q_{h-1}(s',a') && \forall h \in [H] \setminus \{1\}\\
                      &\sum_{a} q_1(s,a) = \mu(s) && \forall s \in \mathcal{S}\\
                      & q_h(s,a) \geq 0 && \forall (s,a, h) \in \mathcal{S} \times \mathcal{A} \times [H]
\end{align*}
The constraint $\sum_{s,a} q_h(s,a) = 1$ is redundant.

\subsection{Notations and Definitions.} 
Throughout the paper, we use $t\in\brs*{H}$ and $k\in\brs*{K}$ to denote time-step inside an episode and the index of an episode, respectively. The filtration $\mathcal{F}_k$ includes all events (states, actions, and costs) until the end of the $k$-th episode, including the initial state of the $k+1$ episode. 
We denote by $n_h^k(s,a)$, the number of times that the agent has visited state-action pair $(s,a)$ at the $h$-th step, and by $\wb{X}_k$, the empirical average of a random variable $X$. Both quantities are based on experience gathered until the end of the $k^{th}$ episode and are $\F_k$ measurable. Since $\pi_k$ is $\filt$ measurable, so is $q_{h}^{\pi_k}(s,a;p)$. Furthermore, from this definition we have that for any $X$ which is $\mathcal{F}_{k-1}$ measureable
\begin{align*}
    \E[X(s^k_h,a^k_h)\mid \mathcal{F}_{k-1}] = \sum_{s,a} q^{\pi_k}_h(s,a;p) X(s,a).
\end{align*}

We use $\wt O(X)$ to refer to a quantity that depends on $X$ up to a poly-log expression of a quantity at most polynomial in $S,A,K,H$ and $\delta^{-1}$.
Similarly, $\lesssim$ represents $\leq$ up to numerical constans or poly-log factors. We define $X\vee Y\triangleq \max\brc*{X,Y}$.

\section{Upper Confidence Bounds for CMDPs}\label{sec: paper curcl}

\begin{algorithm}[t]
        \caption{\cucrl}\label{alg: optimistic model CMDP }
\begin{algorithmic}
\STATE {\bf Require:} $\delta\in(0,1)$
\STATE {\bf Initialize:} $n_h^0(s,a)=0$, $\wb{p}_h^0(s'\mid s,a) =1/S$ and $\wb c_h^0(s,a) =0$
\FOR{$k=1,...,K$}
    \STATE Define $\wt c^k$ and $\wt d^k$ as in~\eqref{eq:lower.bound.cd}
    \STATE Compute the solution of~\eqref{eq:ucrl_optcmd_policy} through the extended LP

    \STATE Execute $\pi_k$ and collect a trajectory $(s_h^k, a_h^k, c_h^k, \{d_{i,h}^k\}_i)$ for $h \in [H]$
    \STATE  Update counters and empirical model (\ie $n^k,\wb c^k, \wb d^k, \wb p^k$) as in~\eqref{eq:empirica_model}
\ENDFOR
\end{algorithmic}
\end{algorithm}

We start by considering a natural adaptation of \ucrl~\citep{jaksch2010near} to the setting of CMDPs which we call \cucrl~(see Algorithm~\ref{alg: optimistic model CMDP }).

Let $n_h^{k-1}(s,a) = \sum_{k'=1}^{k-1} \ind \br*{s_h^{k'} = s,a_h^{k'} = a}$ denote the number of times a pair $(s,a)$ was observed before episode $k$.
At each episode, \cucrl estimates the transition model, cost function and constraint cost function by their empirical average:
\begin{equation}
        \label{eq:empirica_model}
\begin{aligned}
    \wb p_h^{k-1}(s' \mid s,a)
    & =
    \frac{\sum_{k'=1}^{k-1} \ind \br*{s_h^{k'} = s,a_h^{k'} = a,s_{h+1}^{k'} = s'}}{n^{k-1}_h(s,a)\vee 1}
    \\
    \wb c_{h}^{k-1}(s,a)
    & =
    \frac{\sum_{k'=1}^{k-1} c_h^{k'} \cdot \ind \br*{s_h^{k'} = s,a_h^{k'} = a}}{n^{k-1}_h(s,a)\vee 1},\\
    \forall i\in [I], \qquad \wb d_{i,h}^{k-1}(s,a)
    & =
    \frac{\sum_{k'=1}^{k-1} d_{i,h}^{k'} \cdot \ind \br*{s_h^{k'} = s,a_h^{k'} = a}}{n^{k-1}_h(s,a) \vee 1}.
\end{aligned}
\end{equation}

Following the approach of \emph{optimism-in-the-face-of-uncertainty} we would like to act with an optimistic policy. To this end, we generalize the notion of optimism from the bandit setup presented in~\citep{agrawal2019bandits} to the RL setting. Specifically, we would like for our algorithm to satisfy the following demands:
\begin{enumerate}[label={(\alph*)}]
    \item \emph{Feasibility of $\pi^*$ for all episodes.} The optimal policy $\pi^*$ should be contained in the feasible set in every episode. 
    \item \emph{Value optimism.} The value of every policy should be optimistic relatively to its true value, $ V^\pi_1(s_1; \wt c_k,\wt p_k)\leq V^\pi_1(s_1;c,p)$ where $\wt c_k,\wt p_k$ are the optimistic cost and model by which the algorithm calculates the value of a policy.
\end{enumerate}
Indeed, optimizing over a set which satisfy \emph{(a)}  while satisfying \emph{(b)} results in an optimistic estimate of~$V^\star_1(s_1)$. 

Similar to \ucrl, at the beginning of each episode $k$, \cucrl constructs confidence intervals for the costs and the dynamics of the CMDP.
Formally, for any $(s,a) \in \mathcal{S} \times \mathcal{A}$ we define
\begin{align}
        \label{eq:ci_transitions}
        B^p_{h,k}(s,a) &= \Big\{ \wt p(\cdot|s,a) \in \Delta_S : \forall s' \in \mathcal{S},~ |\wt p(\cdot|s,a) - \wb p_h^{k-1}(\cdot|s,a)| \leq \beta^p_{h,k}(s,a,s') \Big\},\\
        \notag
        B^c_{h,k}(s,a) &= \Big[\wb c_h^{k-1}(s,a) - \beta^c_{h,k}(s,a), \wb c_h^{k-1}(s,a) + \beta^c_{h,k}(s,a)\Big],\\
        \notag
        B^d_{i,h,k}(s,a) &= \Big[\wb d_{i,h}^{k-1}(s,a) - \beta^d_{i,h,k}(s,a), \wb d_{i,h}^{k-1}(s,a) + \beta^d_{i,h,k}(s,a)\Big],
\end{align}
where the size of the confidence intervals is built using empirical Bernstein inequality~\citep[\eg][]{Audibert2007tuning,maurer2009empirical} for the transitions and Hoeffding inequality for the costs:
\begin{equation}
        \label{eq:betas}
\begin{aligned}
        \beta^p_{h,k}(s,a,s') &\lesssim \sqrt{\frac{\text{Var}\big(\wb{p}_h^{k-1}(s'|s,a)\big)}{n_h^{k-1}(s,a)\vee 1}} + \frac{1}{n_h^{k-1}(s,a)\vee 1}\\
        \beta^c_{h,k} = \beta^d_{i,h,k} &\lesssim \sqrt{\frac{1}{n_h^{k-1}(s,a)\vee 1}}
\end{aligned}
\end{equation}
where $\text{Var}\big(\wb{p}_h^{k-1}(s'|s,a)\big) = \wb{p}_h^{k-1}(s'|s,a) \cdot (1-\wb{p}_h^{k-1}(s'|s,a))$~\citep[\eg][]{Dann2015samplecomplexity}.
The set of plausible CMDPs associated with the confidence intervals is then
$\mathcal{M}_k = \{ M=(\mathcal{S},\mathcal{A}, \wt{c},\wt{d}, \wt{p})~:~ \wt c_h(s,a) \in B^c_{h,k}(s,a), \wt d_{i,h} \in B^d_{i,h,k}(s,a), \wt p_h(\cdot|s,a) \in B^p_{h,k}(s,a)\}$.
Once $\mathcal{M}_k$ been computed, \cucrl finds a solution to the optimization problem
\begin{equation}\label{eq:ucrl_cmdp_nomind}
\begin{aligned}
        (M_k, \pi_k) = \argmin_{(\wt{c}, \wt{d}_i, \wt{p}) \in \mathcal{M}_k,\; \pi \in \Pi^{\text{MR}} } &\quad 
\sum_{h,s,a}\wt{c}_{h}^k(s,a) q^\pi_h(s,a;\wt p)\\
        \text{s.t.}\qquad&\quad  
        \sum_{h,s,a}  \wt{d}_{i,h}(s,a) q^\pi_h(s,a;\wt p) \leq  \alpha_i,\qquad \forall i \in [H] 
\end{aligned}
\end{equation}
While this problem is well-defined and feasible, we can simplify it and avoid to optimize over the sets $B_k^c$ and $B_k^d$.
We define 
\begin{equation}\label{eq:lower.bound.cd}
        \wt c^k_h(s,a) = \wb c_h^{k-1}(s,a) - \beta^c_{h,k}(s,a) \quad \text{ and } \quad \wt d_{i,h}^k(s,a) = \wb{d}_{i,h}^{k-1}(s,a) - \beta^d_{i,h,k}(s,a)
\end{equation}
to be the lower confidence bounds on the costs. Then, we can solve the following optimization problem
\begin{equation}\label{eq:ucrl_optcmd_policy}
\begin{aligned}
        \min_{\wt{p} \in B^p_k,\; \pi \in \Pi^{\text{MR}} } &\quad \sum_{h,s,a}\wt{c}_{h}^k(s,a) q^\pi_h(s,a;\wt p)\\
        \text{s.t.}\qquad&\quad  
        \sum_{h,s,a} \wt{d}_{i,h}^k(s,a) q^\pi_h(s,a;\wt p) \leq  \alpha_i,\qquad \forall i \in [H] 
\end{aligned}
\end{equation}
Consider a feasible solution $M' = (\sset, \aset, c', d', p')$ and $\pi'$ of problem~\eqref{eq:ucrl_cmdp_nomind}.
We can replace $c'$ with $c_k$ and $d'$ with $d_k$ as in~\eqref{eq:lower.bound.cd} and still have a feasible solution.
This holds since $c' \geq c_k$ and $d' \geq d_k$ componentwise.

We can now state some property of~\eqref{eq:ucrl_optcmd_policy}.
\begin{proposition}\label{prop:ucrl_cmdp_opt}
        The optimization problem~\eqref{eq:ucrl_optcmd_policy} is \emph{feasible}.
        Denote by $\pi_k$ the policy recovered solving~\eqref{eq:ucrl_optcmd_policy} and by $\wt{M}_k = (\sset,\aset, \wt c_k, \wt d_k, \wt p_k)$ the associated CMDP. Then, policy $\pi_k$ is \emph{optimismtic}, \ie 
        \[
                V^{\pi_k}_1(s_1; \wt c_k, \wt p_k) := \wt{c}^\top_k q^{\pi_k}(\wt p_k) \leq c^\top q^{\pi^\star}(p) := V^{\star}_1(s_1; c, p)
        \]
\end{proposition}
\begin{proof}
        The proof of optimism is reported in Lem.~\ref{lemma:ucrl_lp_optimism} and the feasibility is proven in Lem.~\ref{lemma:ucrl_optcmdp_pi_star_is_feasible}.
\end{proof}

\paragraph{The extended LP problem.}
Problem~\eqref{eq:ucrl_optcmd_policy} is similar to~\eqref{eq:cop}, the crucial difference is that the true costs and dynamics are unknown.
Since we cannot directly optimize this problem, 
we propose to rewrite~\eqref{eq:ucrl_optcmd_policy} as an extended LP problem by considering the state-action-state occupancy measure $z^\pi(s,a,s';p)$ defined as $z_h^\pi(s,a,s';p) = p_{h}(s'|s,a) q^\pi_h(s,a;p)$.
We leverage the Bernstein structure of $B^p_{h,k}$ (see Eq.~\ref{eq:ci_transitions}) to formulate the extended LP over variable $z$: 
\begin{align*}
        \min_{z}      &\sum_{h,s,a,s'} z_h(s,a,s') c_h(s,a)\\
        \text{s.t.} 
                      & \sum_{h, s,a,s'} z_h(s,a,s') d_{i,h}(s,a) \leq \alpha_i && \forall i \in [I]\\ 
                      &\sum_{a,s'} z_h(s,a,s') = \sum_{s',a'} z_{h-1}(s',a',s) && \forall h \in [H] \setminus \{1\}\\
                      &\sum_{a,s'} z_1(s,a,s') = \mu(s) && \forall s \in \mathcal{S}\\
                      & z_h(s,a,s') \geq 0 && \forall (s,a,s', h) \in \mathcal{S} \times \mathcal{A} \times \sset \times [H]\\
                      &z_h(s,a,s') - \Big( \wb p_h^{k-1}(s'|s,a) + \beta^p_{h,k}(s,a,s') \Big) \sum_{y}z_h(s,a,y) \leq 0  &&\forall (s,a,s', h) \in \mathcal{S} \times \mathcal{A} \times \sset \times [H]\\
                      &-z_h(s,a,s') + \Big( \wb p_h^{k-1}(s'|s,a) - \beta^p_{h,k}(s,a,s') \Big) \sum_{y}z_h(s,a,y) \leq 0  &&\forall (s,a,s', h) \in \mathcal{S} \times \mathcal{A} \times \sset \times [H]
\end{align*}
This LP has $O(S^2HA)$ constraints and $O(S^2HA)$ decision variables.
Such an approach was also used in~\cite{jin2019learning} in a different context. Notice that $B^p_k$ can be chosen by using different concentration inequalities, \eg $L_1$ concentration inequality for probability distributions. \citet{rosenberg2019online} showed that even in that case we can formulate an extended LP.

Once we have computed $z$, we can recover the policy and the transitions as
\begin{align*}
        \wt{p}_{h}^k(s'|s,a) = \frac{z(s,a,s')}{\sum_{y} z(s,a,y)} \quad \text{ and } \quad \pi_k(a|s) = \frac{\sum_{s'} z(s,a,s')}{\sum_{b,s'} z(s,b,s')} 
\end{align*}

sProposition~\ref{prop:ucrl_cmdp_opt} shows that \emph{(a)} and \emph{(b)} are satisfied and the solution is optimistic. This allows us to provide the following guarantees.
\begin{tcolorbox}[boxrule=0pt, arc=0pt,boxsep=0pt, left=5pt,right=5pt,top=5pt,bottom=5pt]
\begin{restatable}[Regret Bounds for \cucrl]{theorem}{TheoremUCRLOptLP}\label{theorem: optimistic model constraint RL}
Fix $\delta\in (0,1)$. With probability at least $1-\delta$ for any $K'\in[K]$ the following regret bounds hold
\begin{align*}
    &\Regret_+(K';c)  \leq    \Olog \Big(\sqrt{S\mathcal{N}H^4K} + (\sqrt{\mathcal{N}}+H)H^2SA \Big),\\
    &\Regret_+(K';d)  \leq     \Olog \Big( \sqrt{S\mathcal{N}H^4K} + (\sqrt{\mathcal{N}}+H)H^2SA \Big).
\end{align*}
\end{restatable}
\end{tcolorbox}

\section{Exploration Bonus for CMDPs}\label{sec: paper curcl bonus}
\cucrl is an efficient algorithm for exploration in constrained MDPs. An obvious shortcoming of \cucrl is its high computational complexity due to the solution of the extended LP with $O(S^2HA)$ constraints and decision variables.
In this section, we present a bonus-based algorithm for exploration in CMDPs that we call \cucbvi. This algorithm can be seen as a generalization of UCBVI~\citep{azar2017minimax} to constrained MDPs. The main advantage of \cucbvi is that it requires to solve a single CMDP. To this extent, it has to solve an LP problem with $O(SAH)$ constraints and decision variables.

At each episode $k$, \cucbvi builds an optimistic CMDP  $M_k := (\sset, \aset, \wt{c}^k, \wt{d}^k, \wb{p}^{k-1})$ where
\begin{equation}\label{eq:bonusalg_cd}
        \wt{c}_h^k(s,a) = \wb{c}_h^k(s,a) - b_h^k(s,a) \quad \text{ and } \quad  \wt{d}_{i,h}^k(s,a) = \wb{d}_{i,h}^k(s,a) - b_h^k(s,a),
\end{equation}
while $\wb c^k$, $\wb d^k$ and $\wb p^k$ are the empirical estimates defined in~\eqref{eq:empirica_model}.
The term $b_h^k$ integrates the uncertainties about costs and transitions into a single \emph{exploration bonus}. Formally,
\begin{equation}\label{eq:bonusalg_bonus}
        b_{h}^k(s,a) \simeq 
        \beta^r_{h,k}(s,a) + H \sum_{s'} \beta_{h,k}^p(s,a,s')
\end{equation}
where $\beta^r$ and $\beta^p$ are defined as in~\eqref{eq:betas}. 
Then, \cucbvi solves the following optimization problem
\begin{equation}\label{eq:bonusalg_optprob}
\begin{aligned}
        \min_{\pi \in \Pi^{\text{MR}} } &\quad \sum_{h,s,a}\wt{c}_{h}^k(s,a) q^\pi_h(s,a;\wb p^{k-1})\\
        \text{s.t.}\qquad&\quad  
        \sum_{h,s,a} \wt{d}_{i,h}^k(s,a) q^\pi_h(s,a;\wb{p}^{k-1}) \leq  \alpha_i,\qquad \forall i \in [H] 
\end{aligned}
\end{equation}
This problem can be solved using the LP described in Sec.~\ref{sec:linearprogramming}.
In App.~\ref{supp optimism cucrl}, we show that $\pi_k$ is an optimistic policy, \ie $V^{\pi_k}_1(s_1; \wt{c}^k, \wb{p}^k) \leq V_1^\star(s_1)$.


\begin{algorithm}[t]
        \caption{\cucbvi}
        \label{alg: bonus based optimism CMDP }
\begin{algorithmic}
\STATE {\bf Require:} $\delta\in(0,1)$
\STATE {\bf Initialize:} $n_h^0(s,a)=0$, $\wb{p}_h^0(s'\mid s,a) =1/S$ and $\wb c_h^0(s,a) =0$
\FOR{$k=1,...,K$}
    \STATE Compute exploration bonus $b_h^k$ as in~\eqref{eq:bonusalg_bonus}
    \STATE Define $\wt c^k$ and $\wt d^k$ as in~\eqref{eq:bonusalg_cd}
    \STATE Compute the solution of~\eqref{eq:bonusalg_optprob} through LP
    \STATE Execute $\pi_k$ and collect a trajectory $(s_h^k, a_h^k, c_h^k, \{d_{i,h}^k\}_i)$ for $h \in [H]$
    \STATE  Update counters and empirical model (\ie $n^k,\wb c^k, \wb d^k, \wb p^k$) as in~\eqref{eq:empirica_model}
\ENDFOR
\end{algorithmic}
\end{algorithm}


\begin{tcolorbox}[boxrule=0pt, arc=0pt,boxsep=0pt, left=5pt,right=5pt,top=5pt,bottom=5pt]
\begin{restatable}[Regret Bounds for \cucbvi]{theorem}{TheoremBonusOptLP}\label{theorem: bonus based opt constraint RL}
Fix $\delta\in (0,1)$. With probability at least $1-\delta$ for any $K'\in[K]$ the following regret bounds hold
\begin{align*}
    &\Regret_+(K';c)  \leq    \Olog\br*{\sqrt{S\mathcal{N}H^4K} + S^2 H^{4}A(\mathcal{N}H+S)},\\
    &\Regret_+(K';d)  \leq     \Olog\br*{\sqrt{S\mathcal{N}H^4K} + S^2 H^{4}A(\mathcal{N}H+S)}.
\end{align*}
\end{restatable}
\end{tcolorbox}
The regret bounds of \cucbvi include the same $\Olog\br*{\sqrt{S\mathcal{N}H^4K}}$ term as of \cucrl. However, the constant term in the regret bounds of \cucbvi has worst dependence w.r.t.\ $S,H,\mathcal{N}$. This suggests that in the limit of large state space the bonus-based approach for CMDPs have worse performance relatively to the optimistic model approach.
\begin{remark}\label{remark: worst performance of bonus curcl}
The origin of the worst regret bound comes from the larger bonus term~\eqref{eq:bonusalg_bonus} we need to add to compensate on the lack of knowledge of the transition model. This bonus term, allows us to replace the optimistic planning w.r.t.\ a set of transition models (as in \cucrl) by using the empirical transition model. However, it leads to a value function which is not bounded within $[0,H]$ but within $[-\sqrt{S}H^2,H]$.
To circumvent this problem, a truncated Bellman operator has been used~\citep[\eg][]{azar2017minimax,dann2017unifying}.
The value of a policy $\pi$ is thus defined as:
\begin{align*}
    &Q_h^{\pi}(s,a;\tilde c_k, \bar p_{k-1})=\max\brc*{0,\tilde c_h^k(s,a) + \bar p_h^{k-1}(\cdot\mid s,a)V^{\pi}_{h+1}(\cdot; \tilde c_k, \bar p_{k-1})}\\
    &V_h^{\pi}(s;\tilde c_k, \bar p_{k-1}) = \inner{Q_h^{\pi}(s, \cdot ;\tilde c_k, \bar p_{k-1}),\pi_h(\cdot\mid s)}.
\end{align*}
However, plugging this idea into the CMDP problem (Sec.~\ref{sec:linearprogramming}) is not simple. In particular, it is not clear how to enforce truncation in the space of occupancy measures. Thus, reduction to LP seems problematic to obtain.
At the same time, using dynamic programming to solve CMDP is problematic due to the presence of constraints (and the lack of Bellman optimality principle).
We leave it for future work to devise a polynomial algorithm to solve this problem, or establishing it is a ``hard-problem'' to solve. If solved, it would result in an algorithm with similar performance to that of \cucrl (up to polylog and constant factors).
\end{remark}

\section{Optimistic Dual and Primal-Dual Approaches for CMDPs}

In previous sections, we analyzed algorithms which require access to a solver of an LP with at least $\Omega(SHA)$ decision variables and constraints. In the limit of large state space, solving such linear program is expected to be prohibitively expensive in terms of computational cost. Furthermore, most of the practically used RL algorithms~\citep[\eg][]{Achiam2017cpo,Tessler2019rewardcpo} are motivated by the Lagrangian formulation of CMDPs.

Motivated by the need to reduce the computational cost, we follow the Lagrangian approach to CMDPs in which the dual problem to CMDP~\eqref{eq:cop} is being solved. Introducing Lagrange multipliers $\lambda\in \mathbb{R}^I_+$,  the dual problem to \eqref{eq:cop} is given by
\begin{align}
    L^* =  \max_{\lambda\in \mathbb{R}_+^I} \min_{\pi\in \Delta_A^S}\brc*{ c^T q^{\pi}(p) +\lambda^T(Dq^{\pi}(p)-\alpha)}\label{eq: cmdp lagrangina formulation}
\end{align}

With this in mind, a natural way to solve a CMDP is to use a dual sub-gradient algorithm~\citep[see \eg][]{beck2017first} or a primal-dual gradient algorithm. Viewing the problem in this manner, a CMDP can be solved by playing a game between two-player; the agent $\pi$ and the Lagrange multiplier~$\lambda$. This process is expected to converge to the Nash equilibrium with value $L^*$. Furthermore, strong duality is known to hold for CMDP~\cite[\eg][]{altman1999constrained} and thus the expected value of this game is expected to converge to $L^* = V^*_1(s_1)$. This general approach is also followed in the line of works on online learning with long-term constraints~\citep[\eg][]{mahdavi2012trading,yu2017online}. There, the problem does not have a decision horizon $H$ nor state space as in our case. 

As the environment is unknown, and the agents gathers its experience based on samples, the algorithm should use an exploration mechanism with care.  To handle the exploration, we use the optimism approach. In the following sections, we formulate and establish regret bounds for optimistic dual and primal-dual approaches to solve a CMDP. These algorithms are computationally easier than the algorithms of previous sections.  Unfortunately, the regret bounds obtained in this section are weaker. We establish bounds on $\Regret(K;c)$ (resp. $\Regret(K;d)$) instead of $\Regret_+(K;c)$ (resp. $\Regret_+(K;d)$) as in previous section (see Sec.~\ref{sec:learningobj} for details).

\subsection{Optimistic Dual Algorithm for CMDPs}\label{sec: paper dual cmdp}

We start by describing the optimistic dual approach for CMDPs.~\dualcmdp is based upon the dual projected sub-gradient algorithm (e.g.,~\cite{beck2017first}). It can also be interpreted through the lens of online learning. In this sense, we can interpret~\dualcmdp as solving a two-player game in a decentralized manner where the first player (the agent, $\pi$) applies ``be-the-leader'' algorithm, and the second player (the Lagrange multiplier, $\lambda$) uses projected gradient-descent.

Algorithm \dualcmdp (see Alg.~\ref{alg: dual optimistic model cmdp}) acts by performing two stages in each iteration. At the first stage it solves the following optimistic problem:
\[
        \pi_k, \wt p_k \in \argmin_{\pi\in \Pi^{\text{MR}},\; p'\in B^p_{k} }  (\wt c_{k} +\wt D_{k}^T\lambda_{k} )^\top q^{\pi}(p') -\lambda_{k}^T\alpha
\]
where $\wt c_k,\wt d_{k,i}$ and $B_k^p$ are the same as in Sec.~\ref{sec: paper curcl} (refer to~\eqref{eq:ci_transitions} and~\eqref{eq:lower.bound.cd}). 
This problem corresponds to finding the optimal policy (denoted $\pi_k$) of the following extended MDP $\mathcal{M}_k = \{ M= (\sset, \aset, r^+, p^+) \,:\, r^+_h(s,a) = \wt c^k_h(s,a) + \sum_{i} (d^k_{i,h}(s,a) - \alpha_i) \lambda^k_i, p^+_h(\cdot|s,a) \in B^p_{h,k}(s,a)\}$.
Since this is an extended MDP and not a CMDP, we can use standard dynamic programming techniques.
One possibility is to use the extended LP similar to the one introduced in Sec.~\ref{sec: paper curcl}.
Otherwise, we can use backward induction to compute $Q_k$
\[
        Q_h^k(s,a) = r^+_h(s,a) + \min_{p' \in B_{h,k}^p(s,a)} \sum_{s'} p'(s'|s,a) \min_{a'} Q^k_{h+1}(s',a')
\]
with $Q^k_{H+1}(s,a) = 0$ for all $s,a$. Then, $\pi^k_h(s) \in \argmin_a Q^k_h(s,a)$. To compute $q^{\pi_k}_h(s,a)$ we can use Alg. 3 in~\citep{jin2019learning}.

At the second stage, \dualcmdp updates the Lagrange multipliers proportionally to the violation of the ``optimistic'' constraints: $\lambda_{k+1} = \brs*{\lambda_{k} + \frac{1}{t_\lambda} (\wt D_{k} q^{\pi_k}(\wt p_k)-\alpha)}_+$.

The following assumption is standard for the analysis of dual projected sub-gradient method which we make as well. This assumption is quite mild and demands a policy which satisfy the constraint with equality exists. For example, a policy with zero constraint-cost (from state $s_1$) exists this assumption hold. 
\begin{assumption}[Slater Point]\label{assum: slater point} We assume there exists an unknown policy $\wb \pi$ for which $d_i^T q^{\wb \pi}(p)<\alpha_i$ for all the constraints $i\in [I]$. Set
\begin{align*}
    \rho = \frac{c^T q^{\wb \pi}(p) - c^T q^{\pi^*}(p)}{\min_{i=1,..,I} \br*{\alpha_i-d_i^T q^{\wb \pi}(p)}}.
\end{align*}
\end{assumption}

\begin{algorithm}[t]
\caption{ 
\dualcmdp}\label{alg: dual optimistic model cmdp}
\begin{algorithmic}
\REQUIRE $t_\lambda = \sqrt{\frac{H^2IK}{\rho^2}}$, $\lambda_1 \in \mathbb{R^I}, \lambda_1= 0$, Counters, empirical averages
\FOR{$k=1,...,K$}
    \STATE {\color{gray} \# Update Policy}
    \STATE $
        \pi_k, \wt p_k \in \argmin_{\pi\in \Pi^{\text{MR}},\; p'\in B^p_{k} }  (\wt c_{k} +\wt D_{k}^T\lambda_{k} )^\top q^{\pi}(p') -\lambda_{k}^T\alpha
        $
    \STATE {\color{gray} \# Update Dual Parameters}
    \STATE $\lambda_{k+1} = \brs*{\lambda_{k} + \frac{1}{t_\lambda} (\wt D_{k-1} q^{\pi_k}(\wt p_k)-\alpha)}_+$
    \STATE Execute $\pi_k$ and collect a trajectory $(s_h^k, a_h^k, c_h^k, \{d_{i,h}^k\}_i)$ for $h \in [H]$
    \STATE  Update counters and empirical model (\ie $n^k,\wb c^k, \wb d^k, \wb p^k$) as in~\eqref{eq:empirica_model}
\ENDFOR
\end{algorithmic}
\end{algorithm}

The following theorem establishes guarantees for both the performance and the total constraint violation (see App.~\ref{supp: dual optimistic model full proof} for the proof).
\begin{tcolorbox}[boxrule=0pt, arc=0pt,boxsep=0pt, left=5pt,right=5pt,top=5pt,bottom=5pt]
\begin{restatable}[Regret Bounds for \dualcmdp]{theorem}{theoremDualOptimisticModel}\label{theorem: dual optimistic CMDP}
For any $K'\in[K]$ the regrets the following bounds hold 
\begin{align*}
    &\Regret(K';c) \leq    \Olog\br*{\sqrt{S\mathcal{N}H^4K} +\rho\sqrt{H^2IK}  + (\sqrt{\mathcal{N}}+H)H^2SA }\\
    &\Regret(K';d) \leq \Olog\br*{((1+\frac{1}{\rho})\br*{\sqrt{IS\mathcal{N}H^4K} +  (\sqrt{\mathcal{N}}+H)\sqrt{I}H^2SA}}.
\end{align*}
\end{restatable}
\end{tcolorbox}

See that the regret bounded in Theorem~\ref{theorem: dual optimistic CMDP} is $\Regret$ and not $\Regret_+$ as in Sec.~\ref{sec: paper curcl} and \ref{sec: paper curcl bonus}. This difference in types of regret, as we believe, is not an artifact of the analysis. It can be directly attributed to bounds from convex analysis~\citep{beck2017first}. Meaning, establishing a guarantee on $\Regret_+$, instead on $\Regret$,  for \dualcmdp requires to improve convergence guarantees of dual projected gradient-descent. 

Finally, we think that it may be possible to use exploration bonus instead of solving the extended problem. However, we leave this point for future work.

\subsection{Optimistic Primal Dual approach for CMDPs}\label{sec: paper primal dual cmdp} 
In this section, we formulate and analyze \primaldualcmdp (Algorithm~\ref{alg: primal dual optimistic model cmdp}). This algorithm performs incremental, optimistic updates of both primal and dual variables. Optimism is achieved by using \emph{exploration bonuses} (refer to Sec.~\ref{sec: paper curcl bonus}).

Instead of solving an extended MDP as \dualcmdp, \primaldualcmdp evaluates the $Q$-functions of both the cost and constraint cost w.r.t.\ the current policy $\pi_k$ by using the optimistic costs $\wt c_k,\wt d_{k,i}$ and the empirical transition model $\bar p_{k}$. Note that the optimistic cost and constraint costs are obtained using the exploration bonus $b_h^k(s,a)$ defined in Eq.~\ref{eq:bonusalg_cd} (see also Eq.~\ref{eq:ucrl_optcmd_policy}).
Then, it applies a Mirror Descent (MD)~\citep{beck2003mirror} update on the weighted $Q$-function
\begin{align*}
    Q^{k}_h(s,a) = Q^{\pi_k}_h(s,a;\wt c_k, \wb{p}_{k-1}) + \sum_{i=1}^I\lambda_{k,i}Q^{\pi_k}_h(s,a;\wt d_{k,i}, \wb{p}_{k-1}),
\end{align*}
and updates the dual variables, \ie the Lagrange multipliers $\lambda$, by a projected gradient step. Since we optimize over the simplex and choose the Bregman distance to be the KL-divergence, the update rule of MD has a close solution (see the policy update step in Alg.~\ref{alg: primal dual optimistic model cmdp}).

Importantly, in the policy evaluation stage \primaldualcmdp uses a \emph{truncated policy evaluation}, which prevents the value function to be negative (see Algorithm~\ref{alg: truncated policy evaluation}). This allows us to avoid the problems experienced in~\cucbvi when such truncation is not being performed.

Furthermore, differently then in~\dualcmdp, in \primaldualcmdp we project the dual parameter to be within the set $\Lambda_\rho\eqdef \brc*{\lambda: 0\leq \lambda \rho {\bf 1}}$. Such projection can be done efficiently. We remark that such an approach was also applied in~\citep{nedic2009subgradient} for  convex-concave saddle-points problems. The reason for restricting the set of Lagrange multipliers to $\Lambda_\rho$ for our needs is to keep $Q^k$ bounded (if a component of $\lambda_k$ diverges then $Q^k$ might diverge). On the other hand, we wish to keep the set sufficiently big- otherwise, we cannot supply  guarantees on the constraint violations. The set $\Lambda_\rho$ is sufficient to meet both these needs. We remark that projecting on $\Lambda_{\rho'}$ with $\rho'\geq \rho$ would also lead to convergence guarantees by applying similar proof techniques.

The computational complexity of \primaldualcmdp amounts to estimate the state-action value functions $Q^{\pi_k}_h(s,a;\wt c_k, \wb{p}_{k-1})$,  $Q^{\pi_k}_h(s,a;\wt d_{k,i}, \wb{p}_{k-1})$ instead of solving an extended MDP as in~\dualcmdp. However, as the following theorem establishes, the reduced computational cost comes with a worse regrets guarantees. As for~\dualcmdp we assume a slater point exists (see Assumption~\ref{assum: slater point}).

\begin{algorithm}[tb]
\caption{\primaldualcmdp}\label{alg: primal dual optimistic model cmdp}
\begin{algorithmic}
\REQUIRE $t_\lambda = \sqrt{\frac{H^2IK}{\rho^2}},t_K=\sqrt{\frac{2\log A }{(H^2(1+I\rho)^2 K)}}$, $\lambda_1 \in \mathbb{R^I}, \lambda_1= 0$, Counters, empirical averages
\FOR{$k=1,...,K$}
    \STATE Compute exploration bonus $b_h^k$ as in~\eqref{eq:bonusalg_bonus}
    \STATE Define $\wt c^k$ and $\wt d^k$ as in~\eqref{eq:bonusalg_cd}
    \STATE {\color{gray} \# Policy Evaluation}
    \STATE $\brc*{Q^{\pi_k}_h(s,a;\wt c_k, \wb{p}_{k-1})}_{s,a,h} =\text{Trun. Policy Evaluation}(\wt c_k, \wb p_{k-1},\pi_k)$
    \STATE $\forall i\in[I],\ \brc*{Q^{\pi_k}_h(s,a; \wt d_{i,k}, \wb{p}_{k-1})}_{s,a,h} =\text{Trun. Policy Evaluation}(\wt d_{i}^{k}, \wb p_{k-1},\pi_k)$
    \STATE {\color{gray} \# Policy Update}
        \FOR{$\forall h,s,a \in  [H] \times \sset\times \aset$}
        \STATE $ Q^{k}_h(s,a) = Q^{\pi_k}_h(s,a;\wt c_k, \wb{p}_{k-1}) + \sum_{i=1}^I\lambda_{k,i}Q^{\pi_k}_h(s,a;\wt d_{k,i}, \wb{p}_{k-1})$
        \STATE $\pi_h^{k+1}(a|s) \! = \frac{\pi_h^k(a\mid s) \exp\br*{- t_K Q^{k}_h(s,a)}} {\sum_{a'} \pi_h^k(a'\mid s) \exp\br*{- t_K Q^{k}_h(s,a')}} $
    \ENDFOR
    \STATE {\color{gray} \# Update Dual Parameters}
    \STATE $\lambda_{k+1} = \max\brc*{ \lambda_{k} + \frac{1}{t_\lambda} (\wt D_{k-1} q^{\pi_k}(\wt p_k)-\alpha),0}$
    \STATE $\lambda_{k+1} = \min\brc*{\lambda_{k+1}, \rho {\bf 1}}$
    \STATE Execute $\pi_k$ and collect a trajectory $(s_h^k, a_h^k, c_h^k, \{d_{i,h}^k\}_i)$ for $h \in [H]$
    \STATE  Update counters and empirical model (\ie $n^k,\wb c^k, \wb d^k, \wb p^k$) as in~\eqref{eq:empirica_model}
\ENDFOR
\end{algorithmic}
\end{algorithm}

\begin{algorithm}[tb]
\caption{Truncated Policy Evaluation}\label{alg: truncated policy evaluation}
\begin{algorithmic}
\REQUIRE $\forall s,a,s',h,\ \wh l_h(s,a), \wh p_h(s'\mid s,a), \pi_h(a\mid s)$
    \STATE $\forall s \in \sset,\  V^{\pi}_{H+1}(s) = 0$
    \FOR{ $\forall h = H,..,1$}
        \FOR{$\forall s,a \in  \sset\times \aset $}
             \STATE $\wh Q^{\pi}_h(s,a; \wh l,\wh p) = \max\brc*{\wh l_h(s,a) +\wh  p_h(\cdot |s,a)\wh V^{\pi}_{h+1}(\cdot ; \wh l, \wh{p}),0}$
        \ENDFOR
        \FOR{$\forall s \in  \sset$}
            \STATE $\wh V^{\pi}_{h}(s;\wh l, \wh{p}) = \inner{\wh Q^{\pi}_h(s,\cdot; \wh l,\wh p),\pi_h(\cdot \mid s)}$
        \ENDFOR
    \ENDFOR       
    \RETURN $\brc*{\wh Q_h^\pi(s,a)}_{h,s,a}$
\end{algorithmic}
\end{algorithm}

The following theorem establishes guarantees for both the performance and the total constraint violation (see App.~\ref{supp: primal dual bonus optimistic model full proof} for the proof).
\begin{tcolorbox}[boxrule=0pt, arc=0pt,boxsep=0pt, left=5pt,right=5pt,top=5pt,bottom=5pt]
\begin{restatable}[Regret Bounds for \primaldualcmdp]{theorem}{theoremPrimalDualOptimisticModel}\label{theorem: primal dual optimistic CMDP}
For any $K'\in[K]$ the regrets the following bounds hold 
\begin{align*}
    &\Regret(K';c) \leq    \Olog\br*{\sqrt{S\mathcal{N}H^4K} + \sqrt{H^4(1+I\rho)^2 K} + (\sqrt{\mathcal{N}}+H)H^2SA }\\
    &\Regret(K';d) \leq \Olog\br*{(1+\frac{1}{\rho})\br*{\sqrt{IS\mathcal{N}H^4K} +  (\sqrt{\mathcal{N}}+H)\sqrt{I}H^2SA}+ I\sqrt{H^4 K}}.
\end{align*}
\end{restatable}
\end{tcolorbox}

Observe that Theorem~\ref{theorem: primal dual optimistic CMDP} has worst performance relatively to Theorem~\ref{theorem: dual optimistic CMDP} w.r.t. the terms multiplying the $\sqrt{K}$ term. However, its constant term has similar performance to the constant term in Theorem~\ref{theorem: dual optimistic CMDP}.

\section{Conclusions and Summary}

In this work, we formulated and analyzed different algorithms by which safety constraints can be combined in the framework of RL by combining learning in CMDPs. We investigated both UCRL-like approaches (Sec.~\ref{sec: paper curcl} and \ref{sec: paper curcl bonus}) motivated by \ucrl~\citep{jaksch2010near}, as well as, optimistic dual and primal-dual approaches, motivated by practical successes of closely related algorithms~\citep[\eg][]{Achiam2017cpo,Tessler2019rewardcpo}. For all these algorithms, we established regret guarantees for both the performance and constraint violations.

Interestingly, although the dual and primal-dual approaches are nowadays more practically acceptable, we uncovered an important deficiency of these methods; these have `weaker' performance guarantees ($\Regret$) relatively to UCRL-like algorithms ($\Regret_+$). This fact highlights an important practical message if an algorithm designer is interested in good performance w.r.t. $\Regret_+$. Furthermore, the primal-dual algorithm (section~\ref{sec: paper primal dual cmdp}), which is computationally easier, has worse performance relatively to the optimistic dual algorithm (section~\ref{sec: paper dual cmdp}). In light of these observations, we believe an important future venue is to further study the computational-performance tradeoff in safe RL. This would allow algorithm designers better understanding into the types of guarantees that can be obtained when using different types of safe RL algorithms.



\bibliographystyle{plainnat}
\bibliography{citations}

\appendix

\section{Optimistic Algorithm based on Bounded Parameter CMDPs} \label{supp: cucrl}

In this section, we establish regret guarantees for \cucrl (Alg.~\ref{alg: optimistic model CMDP }).
%
As a first step, we recall the algorithm and we formally states the confidence intervals.
The empirical transition model, cost function and constraint cost functions are defined as in~\eqref{eq:empirica_model}.
We recall that \cucrl constructs confidence intervals for the costs and the dynamics of the CMDP.
Formally, for any $(s,a) \in \mathcal{S} \times \mathcal{A}$ we define
\begin{align}
        \label{eq:app_ci_transitions}
        B^p_{h,k}(s,a) &= \Big\{ \wt p(\cdot|s,a) \in \Delta_S : \forall s' \in \mathcal{S},~ |\wt p(\cdot|s,a) - \wb p_h^{k-1}(\cdot|s,a)| \leq \beta^p_{h,k}(s,a,s') \Big\},\\
        \notag
        B^c_{h,k}(s,a) &= \Big[\wb c_h^{k-1}(s,a) - \beta^c_{h,k}(s,a), \wb c_h^{k-1}(s,a) + \beta^c_{h,k}(s,a)\Big],\\
        \notag
        B^d_{i,h,k}(s,a) &= \Big[\wb d_{i,h}^{k-1}(s,a) - \beta^d_{i,h,k}(s,a), \wb d_{i,h}^{k-1}(s,a) + \beta^d_{i,h,k}(s,a)\Big],
\end{align}
where
\begin{equation}
        \label{eq:supp_betas}
\begin{aligned}
        \beta^p_{h,k}(s,a,s') &:= 2\sqrt{\frac{\text{Var}\big(\wb{p}_h^{k-1}(s'|s,a)\big) L_\delta^p}{n_h^{k-1}(s,a)\vee 1}} + \frac{14/3 L^p_{\delta}}{n_h^{k-1}(s,a)\vee 1}\\
        \beta^c_{h,k} = \beta^d_{i,h,k} &:= \sqrt{\frac{L_\delta}{n_h^{k-1}(s,a)\vee 1}}
\end{aligned}
\end{equation}
with $L_{\delta}^{p} = \ln\br*{\frac{6SAHK}{\delta}}$, $L^c_\delta = 2\ln\br*{\frac{6SAH(I+1)K}{\delta}}$ and $\text{Var}\big(\wb{p}_h^{k-1}(s'|s,a)\big) = \wb{p}_h^{k-1}(s'|s,a) \cdot (1-\wb{p}_h^{k-1}(s'|s,a))$.
The set of plausible CMDPs associated with the confidence intervals is then
$\mathcal{M}_k = \{ M=(\mathcal{S},\mathcal{A}, \wt{c},\wt{d}, \wt{p})~:~ \wt c_h(s,a) \in B^c_{h,k}(s,a), \wt d_{i,h} \in B^d_{i,h,k}(s,a), \wt p_h(\cdot|s,a) \in B^p_{h,k}(s,a)\}$.
In the next section, we define the \emph{good event} under which $M^\star \in \mathcal{M}_k$ w.h.p.

\subsection{Failure Events} \label{supp: ucrl model optimistic good events}

Define the following failure events.

\begin{align*}
    &F_k^p=\brc*{\exists s,a,s',h:\ \abs{ p_h(s' \mid s,a) - \wb p_h^{k-1}(s'\mid s,a)}\geq \beta^p_{h,k}(s,a,s')
    }\\
    &F^N_k = \brc*{\exists s,a,h: n_h^{k-1}(s,a) \le \frac{1}{2} \sum_{j<k} q_h^{\pi_k}(s,a\mid p)-H\ln\frac{SAH}{\delta'}}\\
    &F^{c}_{k} = \brc*{\exists s,a,h: |\wb{c}^k_h(s,a)- c_h(s,a) |\geq \beta_{h,k}^c(s,a)}\\
    &F^{d}_{k} = \brc*{\exists s,a,h,i\in [I]: |\wb{d}^k_{i,h}(s,a)- d_{i,h}(s,a) |\geq \beta_{i,h,k}^d(s,a) }
\end{align*}

Furthermore, the following relations hold by standard arguments.

\begin{itemize}
    \item Let $F^{cd}=\bigcup_{k=1}^K F^c_k \cup F^d_k.$ Then $\Pr\brc*{F^{cd}}\leq \delta'$, by Hoeffding's inequality, and using a union bound argument on all $s,a$, all possible values of $n_{k}(s,a)$, all $i\in[I]$ and $k\in [K]$. Furthermore, for $n(s,a)=0$ the bound holds trivially since $C, D_i\in[0,1]$. 
    \item Let $F^P=\bigcup_{k=1}^K F^{p}_k$. Using Thm. 4 in~\citep{maurer2009empirical}, for every fixed $s,a,h,k$ and value of $n_h^k(s,a)$, we have that
    \begin{align*}
        \Pr\brc*{\abs{ p_h(s' \mid s,a) - \wb p_h^{k-1}(s'\mid s,a)}\geq  \epsilon_1 }\leq \delta'',
    \end{align*}
    where 
    \begin{align*}
            \epsilon_1 = \sqrt{\frac{2\text{Var}\Big(\wb p_h^{k-1}(s'\mid s,a)\Big)\ln\br*{\frac{2}{\delta''}}}{n_h^{k-1}(s,a)\vee 1}} + \frac{7 \ln\br*{\frac{2}{\delta''}}}{3(n_h^{k-1}(s,a)-1)\vee 1}.
    \end{align*}
    See that for any $n_h^k(s,a)\geq 2$, we use Theorem 4 in~\citep{maurer2009empirical}, and for $n_h^k(s,a)\in \brc*{0,1}$ the bound holds trivially. This also implies that
    \begin{align*}
        \Pr\brc*{\abs{ p_h(s' \mid s,a) - \wb p_h^{k-1}(s'\mid s,a)}\geq  \epsilon_2 }\leq \delta'',
    \end{align*}
    where 
    \begin{align*}
            \epsilon_2 = \sqrt{\frac{2\text{Var}\Big(\wb p_h^{k-1}(s'\mid s,a)\Big)\ln\br*{\frac{2}{\delta''}}}{n_h^{k-1}(s,a)\vee 1}} + \frac{7 \ln\br*{\frac{2}{\delta''}}}{3(n_h^{k-1}(s,a)-1\vee 1)},
    \end{align*}    
    since $\epsilon_1\leq \epsilon_2$. Applying union bound on all $s,a,h$, and all possible values of $n_k(s,a)$ and $k\in [K]$ and set $\delta''=\frac{\delta'}{(SAHK)^2}$ we get that $\Pr\brc*{F^P}\leq \delta'$.  This analysis was also used in~\citep{jin2019learning}.
    \item Let $F^N=\bigcup_{k=1}^K F^N_k.$ Then, $\Pr\brc*{F^N}\leq \delta'$. The proof is given in~\citep[Cor. E.4]{dann2017unifying}.
\end{itemize}

\begin{remark}{Boundness of of immediate cost and constraints cost.} Notice that we assumed that the random variables $C_h(s,a)\in[0,1]$ and $D_{i,h}(s,a)\in[0,1]$ for any $s,a,h$.
\end{remark}

\begin{lemma}[Good event of \cucrl]\label{lemma: ucrl failure events}
Setting $\delta'=\frac{\delta}{3}$ then $\Pr\brc{\wb G}\leq \delta$ where $${\wb G = F^c \bigcup F^d \bigcup F^p\bigcup F^N = F^{cd} \bigcup F^p\bigcup F^N}.$$ When the failure events does not hold we say the algorithm is outside the failure event, or inside the good event $G$ which is the complement of $\wb G$.
\end{lemma}


The fact $F^p$ holds conditioning on the good event implies the following result~\citep[\eg][Lem. 8]{jin2019learning}.
\begin{lemma}\label{lemma: probability ball around p}
Conditioned on the basic good event, for all $k,h,s,a,s'$ there exists constants $C_1,C_2>0$ for which we have that 
$$ \abs*{\wb p_h^{k-1}(s' \mid s,a) - p_h(s' \mid s,a)} =  C_1\sqrt{\frac{ p_h(s' \mid s,a)L_{\delta,p}}{ n_h^{k}(s,a) \vee 1}} + \frac{C_2 L_{\delta,p}}{ n_h^{k}(s,a) \vee 1}, $$
where $ L_{\delta,p} = \ln\br*{\frac{6SAHK}{\delta}}.$
\end{lemma}

\subsection{Optimism}\label{sup: optimsm ucrl lp}
Recall that $\wt{D}\in \mathbb{R}^{I\times SAH}$ and $\alpha\in \mathbb{R}^I$ such that $\wt{D} = \left[\wt d_1^k, \ldots, \wt d_I^k \right]^\top$ and 
$\alpha = [\alpha_1, \ldots, \alpha_I]^\top$, 
with $\wt d^k$ and $\wt{c}^k$ defined in~\eqref{eq:lower.bound.cd}.

\begin{lemma}[Optimism] \label{lemma:ucrl_lp_optimism}
        Conditioning on the good event, for any $\pi$ there exists a transition model $p'\in B^p_{k}$ for which (i) $\wt D_k q^\pi(p') \leq D q^\pi(p)$, and , (ii) $\wt c^T_k q^\pi(p') \leq c^T q^\pi(p)$.
\end{lemma}
\begin{proof}
Conditioning on the good event, the true model $p$ is contained in $B^p_{k}$. Furthermore, conditioned on the good event $\wt D_k \leq D$ and $\wt c_k \leq c$ component-wise. Thus, setting $p'=p\in B^p_k$ we get
\begin{align*}
    &\wt D_k q^\pi(p') = \wt D_k q^\pi(p)  \leq D q^\pi(p)\\
    &\wt c_k^T q^\pi(p') = \wt c_k^T q^\pi(p)  \leq c^T q^\pi(p),
\end{align*}
where we used the fact that $q^\pi(p) \geq 0$ component-wise.
\end{proof}

\begin{lemma}[$\pi^*$ is Feasible Policy.]\label{lemma:ucrl_optcmdp_pi_star_is_feasible} Conditioning on the good event, $\pi^*$ is a feasible policy for any $k\in [K]$, i.e., 
\begin{align*}
    \pi^*\in \brc*{\pi\in \Delta_A^S: \wt D_k q^\pi(p')\leq \alpha, p'\in B^p_k}.
\end{align*}
\end{lemma}
\begin{proof}
Denote $\Pi_D = \brc{\pi: D q^\pi(p)\leq \alpha}$ as the set of policies which does not violate the constraint on the true model. Furthermore, let $$\Pi^k_D = \brc{\pi: \wt D_k q^\pi(p')\leq \alpha, p'\in B^p_k}$$ be the set of policies which do not violate the constraint w.r.t.\ all possible models at episode $k$. Observe that $\Pi^k_D$ is the set of feasible policies at episode $k$ for \cucrl. 

Conditioning on the good event, by Lemma~\ref{lemma:ucrl_lp_optimism} $D q^\pi(p)\leq \alpha$ implies that exists $p'\in B^p_k$ such that $\wt D_k q^\pi(p')\leq \alpha$. Thus,
\begin{align}
    \Pi_D \subseteq \Pi^k_D. 
\end{align}

Since $\pi^\star \in \Pi_D$ it implies that $\pi^\star \in \Pi^k_D.$

\end{proof}

From the two lemmas we arrive to the following important corollary
\begin{corollary}\label{corollary: ucrl constraint problem optimism}
Conditioning on the good event (i) $ V^{\pi_k}_1(s_1;\wt c_k,\wt p_k)\leq V_1^\star(s_1)$, and, (ii) $V^{\pi_k}_1(s_1;\wt c_k,\wt p_k)\leq V_1^{\pi_k}(s_1;c,p)$.
\end{corollary}
\begin{proof}

The following relations hold.
\begin{align*}
    V^*(s_1) = &\min_{\pi\in \Delta_A^S} \brc*{c^T q^\pi(p) \mid  \pi\in \Pi_D}\\
    \geq &\min_{\pi\in \Delta_A^S,p'\in B^p_k} \brc*{c^T q^\pi(p) \mid  \pi\in \Pi^k_D}\\
    =& \min_{\pi\in \Delta_A^S,p'\in B^p_k} \brc*{c^T q \mid  \wt D_k  q^\pi(p') \leq \alpha}\\
    \geq & \min_{\pi\in \Delta_A^S,p'\in B^p_k} \brc*{\wt c^T_k  q^{\pi}(p') \mid  \wt D\wb q^{\pi}(p')\leq \alpha} =  V_1^{\pi_k}(s_1;\wt c_k,\wt p_k).    
\end{align*}
The second relation holds by Lemma~\ref{lemma:ucrl_optcmdp_pi_star_is_feasible} and the forth relation holds by Lemma~\ref{lemma:ucrl_lp_optimism}.

\end{proof}

\subsection{Proof of Theorem~\ref{theorem: optimistic model constraint RL}} \label{supp: ucrl opt cmdp regret guarantees}

In this section, we establish the following regret bounds for \cucrl (see Alg.~\ref{alg: optimistic model CMDP }). 

\TheoremUCRLOptLP*




\begin{proof}
We start by conditioning on the good event. By Lem.~\ref{lemma: ucrl failure events} it holds with probability at least~$1-\delta.$

We now analyze the regret relatively to the cost $c$. The following relations hold for any $K'\in[K]$.
\begin{align*}
    Regret^+(K';c) &= \sum_k \brs*{V_1^{\pi_k}(s_1;c,p) - V_1^*(s_1;c,p)}_+ \leq \sum_k \brs*{ V_1^{\pi_k}(s_1;c,p) - V_1^{\pi_k}(s_1;\wt c_k, \wt p_k)}_+\\
    &= \sum_k  V_1^{\pi_k}(s_1;c,p) - V_1^{\pi_k}(s_1;\wt c_k, \wt p_k)\\
    &\leq \Olog(\sqrt{S\mathcal{N}H^4K} + (\sqrt{\mathcal{N}}+H)H^2SA).
\end{align*}

The second and third relations hold by optimism, \ie Cor.~\ref{corollary: ucrl constraint problem optimism}.
The forth relation holds by Lem.~\ref{lemma: on policy errors optimistic model}. See that assumptions 1,2,3 of Lem.~\ref{lemma: on policy errors optimistic model} are satisfied conditioning on the good event.

We now turn to prove the regret bound on the constraint violation. For any $i\in[I]$ and $K'\in[K]$ the following relations hold.
\begin{align*}
        \sum_{k=1}^{K'}  \brs*{V^{\pi_k}_1(s_1;d_i,p) - \alpha_i}_+ &= \sum_{k=1}^{K'}  \brs*{\underbrace{V^{\pi_k}_1(s_1;d_i,p) - V^{\pi_k}_1(s_1; \wt{d}^k_i, \wt p^k)}_{\geq 0}+\underbrace{V^{\pi_k}_1(s_1; \wt{d}^k_i, \wt p^k) - \alpha_i }_{\leq 0}}_+ \\
                                                            &\leq \sum_{k=1}^{K'}  V^{\pi_k}_1(s_1;d_i) - V^{\pi_k}_1(s_1; \wt{d}^k_i, \wt p^k) \\
    &\leq \Olog(\sqrt{S\mathcal{N}H^4K} + (\sqrt{\mathcal{N}}+H)H^2SA).
\end{align*}

The first relation holds since $V^{\pi_k}_1(s_1; \wt{d}^k_i, \wt p^k)\leq \alpha$ as the optimization problem solved in every episode is feasible (see Lem.~\ref{lemma:ucrl_optcmdp_pi_star_is_feasible}).
Furthermore, by optimism $V_1^{\pi_k}(s_1;\wt d_{i,k},\wt p_k)\leq V_1^{\pi_k}(s_1;d_i,p)$ (see the first relation of Lem.~\ref{lemma:ucrl_lp_optimism}). The third relation holds by applying Lem.~\ref{lemma: on policy errors optimistic model}. See that assumptions \emph{(a)}, \emph{(b)} and \emph{(c)} of Lem.~\ref{lemma: on policy errors optimistic model} are satisfied conditioning on the good event (see also Lem.~\ref{lemma: probability ball around p}).
\end{proof}

\section{Optimistic Algorithm based on Exploration Bonus}\label{app:bonusalg}
In this section, we establish regret guarantees for \cucbvi (see Alg.~\ref{alg: bonus based optimism CMDP }).
The main advantage of this algorithm \wrt \cucrl is the computational complexity. While \cucrl requires to solve an extended CMDP through an LP with $O(S^2AH)$ constraints and decision variable, \cucbvi requires to find the solution of a single CMDP by solving an LP with $O(SAH)$ constraints and variables.

At each episode $k$, \cucbvi builds an optimistic CMDP  $M_k := (\sset, \aset, \wt{c}^k, \wt{d}^k, \wb{p}^k)$ where
\begin{equation*}
        \wt{c}_h^k(s,a) = \wb{c}_h^k(s,a) - b_h^k(s,a) \quad \text{ and } \quad  \wt{d}_{i,h}^k(s,a) = \wb{d}_{i,h}^k(s,a) - b_h^k(s,a),
\end{equation*}
while $\wb c^k$, $\wb d^k$ and $\wb p^k$ are the empirical estimates defined in~\eqref{eq:empirica_model}.
The exploration bonus $b_h^k$ is defined as 
\begin{equation}\label{eq:app_bonusalg_bonus}
        b_{h}^k(s,a) :=
        \underbrace{\beta^c_{h,k}(s,a)}_{:=b^c_{h,k}(s,a)} + \underbrace{ H \sum_{s'} \beta_{h,k}^p(s,a,s')}_{:= b^p_{h,k}(s,a)}
\end{equation}
where $\beta^c$ and $\beta^p$ are defined as in~\eqref{eq:supp_betas}. 

The policy by which we act at episode $k$ is given by solving the following optimization problem
\begin{align*}
    \pi_k, \wt p_k =& \arg\min_{\pi\in \Delta_A^S} \wt c_{k}^T q^\pi(\wb p_{k-1})\\
    &s.t.\quad  \wt D_{k} q^\pi(\wb p_{k-1}) \leq  \alpha
\end{align*}
where $\wt{D} =[\wt d^k_1, \ldots, \wt d_I^k]^\top$ and $\wt d_i^k$ is defined as in~\eqref{eq:bonusalg_cd}.
Solving this problem can be done by solving an LP, much similar to the LP by which a CMDP is solved (Section~\ref{sec:linearprogramming}).

Before supplying the proof of Theorem~\ref{theorem: bonus based opt constraint RL} we formally defining the set of good events which we show holds with high probability. Conditioning on the good, we establish the optimism of \cucbvi and then regret bounds for \cucbvi.

\subsection{Failure Events} \label{supp: bonus based optimistic good events}

We define the same set of good events as for \cucrl (App.~\ref{supp: ucrl model optimistic good events}). We restate this set here for convenience.
\begin{align*}
    &F_k^p=\brc*{\exists s,a,s',h:\ \abs{ p_h(s' \mid s,a) - \wb p_h^{k-1}(s'\mid s,a)}\geq \beta^p_{h,k}(s,a,s') }\\
    &F^N_k = \brc*{\exists s,a,h: n_h^{k-1}(s,a) \le \frac{1}{2} \sum_{j<k} q_h^{\pi_k}(s,a\mid p)-H\ln\frac{SAH}{\delta'}}\\
    &F^{c}_{k} = \brc*{\exists s,a,h: |\wb{c}^k_h(s,a)- c_h(s,a) |\geq \beta^c_{h,k}(s,a) }\\
    &F^{d}_{k} = \brc*{\exists s,a,h,i\in [I]: |\wb{d}^k_{i,h}(s,a)- d_{i,h}(s,a) |\geq \beta^d_{i,h,k}(s,a)}
\end{align*}

As in App.~\ref{supp: ucrl model optimistic good events} the union of these events hold with probability greater than $1-\delta$. 

\begin{lemma}[Good event of \cucbvi]\label{lemma: bonus optimsitc cmdp failure events}
Setting $\delta'=\frac{\delta}{3}$ then $\Pr\brc{\wb G}\leq \delta$ where $${\wb G = F^c \bigcup F^d \bigcup F^p\bigcup F^N}.$$ When the failure events does not hold we say the algorithm is outside the failure event, or inside the good event $G$ which is the complement of $\wb G$.
\end{lemma}

\begin{lemma}\label{lemma: bonus based probability ball around p}
Conditioned on the basic good event, for all $k,h,s,a,s'$ there exists constants $C_1,C_2>0$ for which we have that 
$$ \abs*{\wb p_h^{k-1}(s' \mid s,a) - p_h(s' \mid s,a)} =  C_1\sqrt{\frac{ p_h(s' \mid s,a) L_{\delta,p}}{ n_h^{k}(s,a) \vee 1}} + \frac{C_2 L_{\delta,p}}{ n_h^{k}(s,a) \vee 1}, $$
where $ L_{\delta,p} = \ln\br*{\frac{6SAHK}{\delta}}.$
\end{lemma}

\subsection{Optimism}\label{supp optimism cucrl}

\begin{lemma}[Per-State Optimism.]\label{lemma: per state optims bonus} Conditioning on the good event, for any $\pi, s,a,h,k, i\in[I]$ it holds that
\begin{align*}
        \wt c_h(s,a) - c_h(s,a) - \sum_{s'} (p_h-\wb p_{h}^{k-1})(s'\mid s,a)V^\pi_{h+1}(s' ;c,p)\leq 0,
\end{align*}
and
\begin{align*}
        \wt d_h(s_h,a_h) - d_h(s_h,a_h) - \sum_{s'} (p_h-\wb p_{h}^{k-1})(s'\mid s_h,a_h)V^\pi_{h+1}(s' ;d_i,p)\leq 0.
\end{align*}
\end{lemma}
\begin{proof}
For any $s,a,h,k$, conditioning on the good event,
\begin{align}
        \wb c_h(s,a) - c_h(s,a) - b_{h,k}^c(s,a) &\leq  \underbrace{|\wb c_h(s,a) - c_h(s,a)|}_{\leq \beta_{h,k}^c(s,a)} -b_{h,k}^c(s,a) \leq 0
        \label{eq: per staet optimism relation 1}
\end{align}
by the choice of the bonus $b_{h,k}^c$.

Furthermore, for any $s,a,h,k$
\begin{align}
    &(p_h-\wb p_{h}^{k-1})(\cdot\mid s,a)V^\pi_{h+1}(c) -b^p_{h,k}(s,a) \nonumber \\
    &\leq \sum_{s'}\abs*{(p_h-\wb p_{h}^{k-1})(s'\mid s,a)}\abs*{V^\pi_{h+1}(s';d_i)} -b^p_{h,k}(s,a) \nonumber\\
    &\leq H\sum_{s'}\abs*{(p_h-\wb p_{h}^{k-1})(s'\mid s,a)} -b^p_{h,k}(s,a) \nonumber\\
    &\leq 2H\sum_{s'}\sqrt{\frac{\wb p_h^{k-1}(s' \mid s,a) L_{p,\delta}}{n_h^{k}(s,a) \vee 1}} +H \frac{14L_{p,\delta}}{3 \br*{(n_h^{k}(s,a) - 1) \vee 1}}  -b^p_{h,k}(s,a) \nonumber\\
    & = b^p_{h,k}(s,a) -b^p_{h,k}(s,a) = 0, \label{eq: per staet optimism relation 2}
\end{align}
where the forth relation holds conditioning on the good event, and the fifth relation by the choice of the bonus $b^p_{h,k}(s,a)$.

Combining \eqref{eq: per staet optimism relation 1} and \eqref{eq: per staet optimism relation 2} we get that
\begin{align*}
        \wt c_h(s,a) - c_h(s,a) - (p_h-\wb p_{h}^{k-1})(\cdot\mid s,a)V^\pi_{h+1}(\cdot ;c,p)\leq 0.
\end{align*}

Repeating this analysis while replacing $c,\wt c_k$ with $d_i, \wt d_{i,k}$ we conclude the proof of the lemma. 
\end{proof}

\begin{lemma}[Optimism] \label{lemma: bonus lp optimism}
Conditioning on the good event, for any $\pi,s,h,k,i$ it holds that  (i) $V^{\pi}_h(s;\wt c_k,\wb p_k) \leq V^{\pi}_h(s;c,p)$, and, (ii) $V^{\pi}_h(s;\wt d_{i}^k,\wb p_k) \leq V^{\pi}_h(s;d_i,p)$.
\end{lemma}
\begin{proof}
For any $k\in [K]$ we have that
\begin{align*}
   & V^{\pi}(s_1;\wt c_k,\wb p_k) - V^{\pi}(s_1;c,p)\\
   & = \E\brs*{ \sum_{h=1}^H \wt c_h(s_h,a_h) - c_h(s_h,a_h) - (p_h-\wb p_{h}^{k-1})(\cdot\mid s_h,a_h)V^\pi_{h+1}(\cdot ;c,p) \Big| s_1,\pi,\wb p_{k-1}}
\end{align*}
where we used the value difference lemma (see Lem.~\ref{lemma: value difference}).
Applying the first statement of Lem.~\ref{lemma: per state optims bonus} which hold for any $s,a,h,k$ (conditioning on the good event) we conclude the proof of the first claim.



The second claim follows by the same analysis on the difference $V^{\pi}_h(s;\wt d_{i}^k,\wb p_{k-1}) - V^{\pi}_h(s;d_i,p)$, \ie using the value difference lemma and the second claim in Lem.~\ref{lemma: per state optims bonus}. 
\end{proof}

The following lemma shows that the problem solved by \cucbvi is always feasible. This lemma follows the same idea used to prove the feasibility for \cucrl (see Lem.~\ref{lemma:ucrl_optcmdp_pi_star_is_feasible}).
\begin{lemma}[$\pi^\star$ is Feasible Policy.]\label{lemma: bonus based optimism pi star is feasible} Conditioning on the good event, $\pi^\star$ is a feasible policy for any $k\in [K]$, i.e., 
\begin{align*}
        \pi^*\in \brc*{\pi\in \Delta_A^S: \wt D_k q^\pi(\wb p_{k-1})\leq \alpha}.
\end{align*}
\end{lemma}
\begin{proof}
        Denote $\Pi_D = \brc{\pi: D q^\pi(p)\leq \alpha}$ as the set of policies which does not violate the constraint on the true model. Furthermore, let $$\Pi^k_D = \brc{\pi: \wt D_k q^\pi(\wb p_{k-1})\leq \alpha}$$ be the set of policies which do not violate the constraint w.r.t. all possible models at the $k^{th}$ episode. 

Conditioning on the good event, by Lem.~\ref{lemma: bonus lp optimism} $D q^\pi(p)\leq \alpha$ implies that $\wt D_k q^\pi(\wb p_{k-1})\leq \alpha$. Thus,
\begin{align}
\Pi_D \subseteq \Pi^k_D. 
\end{align}

Since $\pi^*\in \Pi_D$ it implies that $\pi^*\in \Pi^k_D.$

\end{proof}

From the two lemmas we arrive to the following  corollary as
\begin{corollary}\label{corollary: bonus optimsm constraint problem optimism}
        Conditioning on the good event (i) $ V^{\pi_k}_1(s_1;\wt c_k,\wb p_{k-1})\leq V_1^\star(s_1)$, and, (ii) $ V^{\pi_k}_1(s_1;\wt c_k,\wb p_{k-1})\leq V_1^{\pi_k}(s_1;c,p)$.
\end{corollary}
\begin{proof}

The following relations hold.
\begin{align*}
    V^*(s_1) = &\min_{\pi\in \Delta_A^S} \brc*{c^T q^\pi(p) \mid  \pi\in \Pi_D}\\
    \geq &\min_{\pi\in \Delta_A^S} \brc*{c^T q^\pi(p) \mid  \pi\in \Pi^k_D}\\
        =& \min_{\pi\in \Delta_A^S} \brc*{c^T q^\pi(p) \mid  \wt D_k  q^\pi(\wb p_{k-1}) \leq \alpha}\\
        \geq & \min_{\pi\in \Delta_A^S} \brc*{\wt c^T_k  q^{\pi}(\wb p_{k-1}) \mid \wt D_k  q^\pi(\wb p_{k-1}) \leq \alpha} = V_1^{\pi_k}(s_1;\wt c_k,\wb p_{k-1}).    
\end{align*}
The second relation holds by Lem.~\ref{lemma: bonus based optimism pi star is feasible} and the forth relation holds by Lem.~\ref{lemma: bonus lp optimism}.

\end{proof}

\subsection{Proof of Theorem~\ref{theorem: bonus based opt constraint RL}}

In this section, we establish the following regret bounds for \cucbvi algorithm.

\TheoremBonusOptLP*

Unlike the proof of the \cucbvi algorithm (Thm.~\ref{theorem: optimistic model constraint RL}), the value function is not constraint to be within $[0,H]$ . However, since the bonus is bounded, the estimated value function is bounded in the range of $[-\sqrt{S}H^2,H]$. Although this discrepency, in the following we are able to reach similar dependence in $\sqrt{K}$. The fact the estimated value is bounded in \cucbvi differently then in \cucrl results in worse constant term as Thm.~\ref{theorem: bonus based opt constraint RL} exhibits (see Remark~\ref{remark: worst performance of bonus curcl}).

\begin{proof}
We start by conditioning on the good event. By Lem.~\ref{lemma: ucrl failure events}, it holds with probability at least~$1-\delta.$
We now analyze the regret relatively to the cost $c$. The following relations hold for any $K'\in[K]$:
\begin{align*}
        \Regret_+(K';c) &= \sum_k \brs*{V_1^{\pi_k}(s_1;c,p) - V_1^\star(s_1;c,p)}_+ \leq \sum_k \brs*{ V_1^{\pi_k}(s_1;c,p) - V_1^{\pi_k}(s_1;\wt c_k, \wb p_{k-1})}_+\\
                        &= \sum_k  V_1^{\pi_k}(s_1;c,p) - V_1^{\pi_k}(s_1;\wt c_k, \wb p_{k-1})\\
    &\leq \Olog\br*{\sqrt{S\mathcal{N}H^4K} + S^2 H^{4}A(\mathcal{N}H+S)}.
\end{align*}

The second and third relations hold by optimism, see Cor.~\ref{corollary: bonus optimsm constraint problem optimism}. The forth relation holds by Lem.~\ref{lemma: on policy errors bonus based optimistic}. See that assumptions 1,2,3 of Lem.~\ref{lemma: on policy errors bonus based optimistic} are satisfied conditioning on the good event. Assumption 4 of Lem.~\ref{lemma: on policy errors bonus based optimistic} holds by the optimism of the value estimate (see Lem.~\ref{lemma: bonus lp optimism}). Assumption 5 of Lem.~\ref{lemma: on policy errors bonus based optimistic} holds by Lem.~\ref{lemma: per state optims bonus}.

We now turn to prove the regret bound on the constraint violation. For any $i\in[I]$ and $K'\in[K]$ the following relations hold.
\begin{align*}
    \sum_{k=1}^{K'}  \brs*{V^{\pi_k}_1(s_1;d_i) - \alpha}_+ &= \sum_{k=1}^K  \brs*{\underbrace{V^{\pi_k}_1(s_1;d_i,p) - \wb V^{\pi_k}_1(s_1;d_i)}_{\geq 0}+\underbrace{\wb V^{\pi_k}_1(s_1;d_i) - \alpha }_{\leq 0}}_+ \\
                                                            &\leq \sum_{k=1}^K  V^{\pi_k}_1(s_1;d_i,p) - V^{\pi_k}_1(s_1; \wt d_i^k, \wb p_{k-1})\\ 
    &\leq \Olog\br*{\sqrt{S\mathcal{N}H^4K} + S^2 H^{4}A(\mathcal{N}H+S)}.
\end{align*}

The first relation holds since $V^{\pi_k}_1(s_1; \wt d_i^k, \wb p_{k-1}) \leq \alpha$ as the optimization problem solved in every episode is feasible, see Lem.~\ref{lemma: bonus based optimism pi star is feasible}.
Furthermore, by optimism $V_1^{\pi_k}(s_1;\wt d_{i}^{k},\wt p_k)\leq V_1^{\pi_k}(s_1;d_i,p)$ (see the first relation of Lem.~\ref{lemma: bonus lp optimism}). The third relation holds by applying Lem.~\ref{lemma: on policy errors bonus based optimistic}. See that assumptions 1,2,3 of Lem.~\ref{lemma: on policy errors bonus based optimistic} are satisfied conditioning on the good event (see also Lem.~\ref{lemma: bonus based probability ball around p}).

\end{proof}

\section{Constraint MDPs Dual Approach}\label{supp: dual optimistic model full proof}
In this section, we establish regret guarantees for \dualcmdp by proving Theorem~\ref{theorem: dual optimistic CMDP}. Unlike both previous sections,  \dualcmdp does not require an LP solver, but repeatedly solves MDPs with uncertainty in their transition model.

Before supplying the proof of Theorem~\ref{theorem: dual optimistic CMDP} we formally define the set of good events which we show holds with high probability. Conditioning on the good, we establish the optimism of \dualcmdp and then regret bounds for \dualcmdp. The regret bound of \dualcmdp relies on results from constraint convex optimization with some minor adaptations which we establish in Appendix~\ref{supp: convex optimization review}.

\subsection{Definitions}

We introduce a notation that will be used across the proves of this section. Following this notation allows us to apply generic results from convex optimization to the problem.

\begin{itemize}
    \item The optimistic and true constraints valuation are denoted by
\begin{align*}
    &\wt g_k = ( \wt D_k q^{\pi_k}(\wt p_k) -\alpha)\\
    &g_k = (D q^{\pi_k}(p) -\alpha).
\end{align*}
\item The optimistic value, true value, and optimal value are denoted by
\begin{align*}
    &\wt f_k = \wt c_k^T q^{\pi_k}(\wt p_k)\\
    & f_k = c^T {q}^{\pi_k}\\
    & f_{opt} = V_1^*(s_1) = c^T {q}^{*}.
\end{align*}

\end{itemize}

\subsection{Failure Events} \label{supp: dual optimistic model good events}

We define the same set of good events as for \dualcmdp (Appendix~\ref{supp: ucrl model optimistic good events}). We restate this set here for convenience.
%
\begin{align*}
    &F_k^p=\brc*{\exists s,a,s',h:\ \abs{ p_h(s' \mid s,a) - \wb p_h^{k-1}(s'\mid s,a)}\geq \beta^p_{h,k}(s,a,s')
    }\\
    &F^N_k = \brc*{\exists s,a,h: n_h^{k-1}(s,a) \le \frac{1}{2} \sum_{j<k} q_h^{\pi_k}(s,a\mid p)-H\ln\frac{SAH}{\delta'}}\\
    &F^{c}_{k} = \brc*{\exists s,a,h: |\wb{c}^k_h(s,a)- c_h(s,a) |\geq \beta_{h,k}^c(s,a)}\\
    &F^{d}_{k} = \brc*{\exists s,a,h,i\in [I]: |\wb{d}^k_{i,h}(s,a)- d_{i,h}(s,a) |\geq \beta_{i,h,k}^d(s,a) }
\end{align*}

As in Appendix~\ref{supp: ucrl model optimistic good events} the union of these events hold with probability greater than $1-\delta$. 

\begin{lemma}[Good event of \dualcmdp]\label{lemma: dual cmdp failure events}
Setting $\delta'=\frac{\delta}{3}$ then $\Pr\brc{\wb G}\leq \delta$ where $${\wb G = F^c \bigcup F^d \bigcup F^p\bigcup F^N}.$$ When the failure events does not hold we say the algorithm is outside the failure event, or inside the good event $G$ which is the complement of $\wb G$.
\end{lemma}

\begin{lemma}\label{lemma: dual probability ball around p}
Conditioned on the basic good event, for all $k,h,s,a,s'$ there exists constants $C_1,C_2>0$ for which we have that 
$$ \abs*{\wb p_h^{k-1}(s' \mid s,a) - p_h(s' \mid s,a)} =  C_1\sqrt{\frac{ p_h(s' \mid s,a) L_{\delta,p}}{ n_h^{k}(s,a) \vee 1}} + \frac{C_2 L_{\delta,p}}{n_h^{k}(s,a) \vee 1}, $$
where $ L_{\delta,p} = \ln\br*{\frac{6SAHK}{\delta}}.$
\end{lemma}


\subsection{Proof of Theorem~\ref{theorem: dual optimistic CMDP}}
 
In this section, we establish the following regret bound for \dualcmdp.
 
\theoremDualOptimisticModel*

We start by proving several useful lemmas on which the proof is based upon.
\begin{lemma}[Dual Optimism] \label{lemma: incremental constraint to optimality}
Conditioning on the good event, for any $k\in[K]$
\begin{align*}
    \wt f_k - f_{opt} \leq  -\lambda_k^T \wt g_k
\end{align*}
\end{lemma}
\begin{proof}
We have that
\begin{align*}
    f_{opt} = c^T q^{\pi^*}(p) &\geq c^T q^{\pi^*}(p) + \lambda_k^T (Dq^{\pi^*}(p)-\alpha) \\
    &\geq \min_{\pi\in \Delta_A^S, p'\in \mathcal{P}_{k} }  \wt c_{k}^T q^{\pi}(p') + \lambda_k^T (\wt D_{k}q^{\pi}(p')-\alpha)\\
    & = \wt c_{k}^T q^{\pi_k}(\wt p_k) + \lambda_k^T (\wt D_{k}q^{\pi_k}(\wt p_k)-\alpha)\\
    &= \wt f_k +  \lambda_k^T \wt g_k.
\end{align*}
The first relation holds since $\pi^*$ satisfies the constraint (Assumption~\ref{assum: feasability}) which implies that $(Dq^{\pi^*}(p)-\alpha)\leq 0$, and that $\lambda_k\geq0$ by the update rule. The second relation holds since conditioning on the good event the true model is contained in $B^p_{k}$ as well as $\wt c_{k} \leq c$. 
\end{proof}

\begin{restatable}
[Update Rule Recursion Bound]{lemma}{LemmaDualUpdateRuleRecursion}\label{lemma: incremental alg recusrion relation}
For any $\lambda\in \mathbb{R}_+^I$ and $K'\in[K]$
\begin{align*}
    \sum_{k=1}^{K'}\br*{-\wt g_k^T\lambda_{k}}  + \sum_{k=1}^{N}\wt g_k^T \lambda \leq  \frac{t_\lambda}{2}\norm{\lambda_{1} - \lambda}^2_2 + \frac{1}{2t_\lambda}\sum_{k=1}^{K'}\norm{\wt g_k}^2
\end{align*}
\end{restatable}
\begin{proof}

For any $\lambda \in \mathbb{R}_+^I$ by the update rule we have that
\begin{align*}
    \norm{\lambda_{k+1} - \lambda}^2_2 &= \norm{[\lambda_{k} + \frac{1}{t_\lambda}\wt g_k]_+ - [\lambda]_+}^2_2\\
    &\leq \norm{\lambda_{k} + \frac{1}{t_\lambda}\wt g_k - \lambda}^2_2\\
    &=\norm{\lambda_{k} - \lambda}_2^2 + \frac{2}{t_\lambda}\wt g_k^T(\lambda_{k} - \lambda) + \frac{1}{t_\lambda^2}\norm{\wt g_k}^2.
\end{align*}
Summing this relation for $k\in [K']$ and multiplying both sides by $t_\lambda/2$ we get
\begin{align}
     - \frac{t_\lambda}{2}\norm{\lambda_{1} - \lambda}^2_2 &\leq \frac{t_\lambda}{2}\norm{\lambda_{K'+1} - \lambda}^2_2- \frac{t_\lambda}{2}\norm{\lambda_{1} - \lambda}^2_2 \nonumber \\
     &\leq \sum_{k=1}^{K'}\wt g_k^T(\lambda_{k} - \lambda) + \frac{1}{2t_\lambda}\sum_{k=1}^{K'}\norm{\wt g_k}^2.\nonumber
\end{align}

Rearranging we get,
\begin{align*}
    \sum_{k=1}^{N}\br*{-\wt g_k^T\lambda_{k}}  + \sum_{k=1}^{N}\wt g_k^T \lambda \leq  \frac{t_\lambda}{2}\norm{\lambda_{1} - \lambda}^2_2 + \frac{1}{2t_\lambda}\sum_{k=1}^{K'} \norm{\wt g_k}^2 
\end{align*}
for any $\lambda\in \mathbb{R}_+^I$.

\end{proof}

We are now ready to establish Theorem~\ref{theorem: dual optimistic CMDP}.
\begin{proof}
Plugging Lemma~\ref{lemma: incremental constraint to optimality} into Lemma~\ref{lemma: incremental alg recusrion relation}  we get
\begin{align*}
    &\sum_{k=1}^{K'}\br*{\wt f_k - f_{opt}}  + \sum_{k=1}^{K'}\wt g_k^T \lambda \leq \sum_{k=1}^{K'}\br*{-\wt g_k^T\lambda_{k}}  + \sum_{k=1}^{K'}\wt g_k^T \lambda \leq  \frac{t_\lambda}{2}\norm{\lambda_1-\lambda}_2^2 + \frac{1}{2t_\lambda}\sum_{k=1}^{K'}\norm{\wt g_k}^2.
\end{align*}

Adding, subtracting $\sum_{k=1}^{K'} g_k^T \lambda, \sum_{k=1}^{K'} f_k$ and rearranging we get
\begin{align}
    &\sum_{k=1}^{K'}\br*{f_k - f_{opt}}  + \sum_{k=1}^{K'} g_k^T \lambda \nonumber\\
    &\leq \frac{t_\lambda}{2}\norm{\lambda_1-\lambda}_2^2 +  \frac{1}{2t_\lambda}\sum_{k=1}^{K'}\norm{\wt g_k}^2 + \sum_{k=1}^{K'} (g_k-\wt g_k)^T \lambda + \sum_{k=1}^{K'} (f_k-\wt f_k) \nonumber \\
    &\leq \frac{t_\lambda}{2}\norm{\lambda_1-\lambda}_2^2 +  \frac{1}{2t_\lambda}\sum_{k=1}^{K'}\norm{\wt g_k}^2 + \sqrt{\sum_{i=1}^I\br*{\sum_{k=1}^{K'} (g_{k,i}-\wt g_{k,i})}^2} \norm{\lambda}_2 + \sum_{k=1}^{K'} (f_k-\wt f_k) \label{eq: central bound constraint and optimality}
\end{align}
for any $\lambda\in \mathbb{R}^I_+$, where the last relation holds by Cauchy Schwartz inequality.

We now bound each term in~\eqref{eq: central bound constraint and optimality}. Notice that $\wt{g}_{k,i} = V^{\pi_k}(s_1;\wt d_{k,i},\wt p_k) - \alpha_i\in[-L^c_\delta H,H]$ (where $L_\delta = 2\ln\br*{\frac{6SAH(I+1)K}{\delta}}$); it is a value function defined on an MDP with immediate cost in $[-L^c_\delta H,H]$ and $\alpha\in[0,H]$. Thus, we have that
\begin{align*}
    \frac{1}{2t_\lambda}\sum_{k=1}^{K'}\norm{\wt g_k}^2\lesssim \frac{H^2IK}{2t_\lambda}.
\end{align*}

Applying Lemma~\ref{lemma: on policy errors optimistic model} (see that assumptions \emph{(a)}, \emph{(b)} and \emph{(c)} hold conditioning on the good event), we get that
\begin{align*}
    &\abs*{\sum_{k=1}^{K'} (f_k-\wt f_k)} = \abs*{\sum_{k=1}^{K'} (V^{\pi_k}(s_1;c,p)-\wh V^{\pi_k}(s_1;\wt c_k,\wb p_k)}  \leq \Olog\br*{\sqrt{S\mathcal{N}H^4K} + (\sqrt{\mathcal{N}}+H)H^2SA }\\
    &\abs*{\sum_{k=1}^{K'} (g_{k,i}-\wt g_{k,i})} = \abs*{\sum_{k=1}^{K'} (V^{\pi_k}(s_1;d_i,p)-\wh V^{\pi_k}(s_1;\wt d_{k,i},\wb p_k)} \leq \Olog\br*{\sqrt{S\mathcal{N}H^4K} + (\sqrt{\mathcal{N}}+H)H^2SA },
\end{align*}
which implies that
\begin{align*}
    \sqrt{\sum_{i=1}^I\br*{\sum_{k=1}^{K'} (g_{k,i}-\wt g_{k,i})}^2}\leq \Olog\br*{\sqrt{IS\mathcal{N}H^4K} + (\sqrt{\mathcal{N}}+H)\sqrt{I}H^2SA}.
\end{align*}
Plugging these bounds back into~\eqref{eq: central bound constraint and optimality} and setting $t_\lambda = \sqrt{\frac{H^2IK}{\rho^2}}$ we get
\begin{align}
    &\sum_{k=1}^{K'}\br*{f_k - f_{opt}}  + \sum_{k=1}^{K'} g_k^T \lambda \nonumber\\
    &\quad\lesssim (\rho+\frac{\norm{\lambda}_2^2}{\rho})\sqrt{H^2IK} +\br*{\sqrt{IS\mathcal{N}H^4K} + (\sqrt{\mathcal{N}}+H)\sqrt{I}H^2SA}\norm{\lambda}_2 \nonumber\\
    &\quad\quad + \br*{\sqrt{S\mathcal{N}H^4K} + (\sqrt{\mathcal{N}}+H)H^2SA } \label{eq: dual central second bound constraint and optimality},
\end{align}
for any $\lambda\in \mathbb{R}_+^I$.

\paragraph{First claim of Theorem~\ref{theorem: dual optimistic CMDP}.} Setting $\wb \lambda=0$ (see that $\wb \lambda\in \mathbb{R}^I_+$)  in~\eqref{eq: dual central second bound constraint and optimality} we get


\begin{align*}
    \sum_{k=1}^{K'} V^{\pi_k}(s_1;c,p) - V^*(s_1) = \sum_{k=1}^{K'}f_k - f_{opt} \lesssim \Olog\br*{\sqrt{S\mathcal{N}H^4K} +\rho\sqrt{H^2IK}  + (\sqrt{\mathcal{N}}+H)H^2SA }.
\end{align*}

\paragraph{Second claim of Theorem~\ref{theorem: dual optimistic CMDP}.} Fix $i\in [I]$ and let 
\begin{align*}
    \wb \lambda_i = \begin{cases}
            \rho e_i & \brs{\sum_{k=1}^{K'} g_{i,k}}_+ \neq 0\\
            0 & \text{otherwise},
    \end{cases}
\end{align*}
where $e_i(i)= 1$ and $e_i(j)=0$ for $j\neq i$, and $\rho$ is given in Assumption~\ref{assum: slater point}. See that $\wb \lambda_i\in \mathbb{R}_+^I$ and that, by the definition,
\begin{align}
    \norm{ \wb \lambda_i}_2^2\leq \rho^2 \label{eq: bound on lambda i norm}
\end{align}
Setting $\lambda=\wb \lambda_i$ in~\eqref{eq: dual central second bound constraint and optimality} we get
\begin{align*}
    &\sum_{k=1}^{K'}\br*{f_k - f_{opt}}  + \rho \brs*{\sum_{k=1}^{K'} g_{i,k}}_+ \\
    &\leq \Olog\br*{(1+\rho)\br*{\sqrt{IS\mathcal{N}H^4K} + \sqrt{H^2IK}+ (\sqrt{\mathcal{N}}+H)\sqrt{I}H^2SA}} \eqdef \epsilon(K).
\end{align*}
Since the bound holds for any $i\in [I]$ we get that
\begin{align*}
    &\max_{i\in [I]} \sum_{k=1}^{K'}\br*{f_k - f_{opt}}  + \rho \brs*{\sum_{k=1}^{K'} g_{i,k}}_+=   \sum_{k=1}^{K'}\br*{f_k - f_{opt}}  + \rho \max_{i\in [I]}\brs*{\sum_{k=1}^{K'} g_{i,k}}_+ \\
    & =  \sum_{k=1}^{K'}\br*{f_k - f_{opt}}  + \rho \max_{i\in [I]}\abs*{\brs*{\sum_{k=1}^{K'} g_{i,k}}_+}\\
    & =  \sum_{k=1}^{K'}\br*{f_k - f_{opt}}  + \rho \norm*{\brs*{\sum_{k=1}^{K'} g_{k}}_+}_\infty \leq \epsilon(K).
\end{align*}
Now, by the convexity of the state-action frequency (see Proposition~\ref{prop: convexity of state action frequencey}) function there exists a policy $\pi_{K'}$ which satisfies $q^{\pi_{K'}}(p) =\frac{1}{K'}\sum_{k=1}^{K'}q^{\pi_k}(p)$ for any $K'$. Since both $f$ and $g$ are linear in $\frac{1}{K'}\sum_{k=1}^{K'}q^{\pi_k}(p)$ we have that
\begin{align*}
    \frac{1}{K'}\br*{\sum_{k=1}^{K'}\br*{f_k - f_{opt}}  + \rho \norm*{\brs*{\sum_{k=1}^{K'} g_k}_+}_2} = f_{\pi_{K'}} - f_{opt} + \rho \norm*{\brs*{g_{\pi_{K'}}}_+}_2 \leq \frac{1}{K'}\epsilon(K).
\end{align*}

Applying Corollary~\ref{corollary: bound on a dual variable} and Theorem~\ref{theorem: thoeorem constraint bound beck} we conclude that
\begin{align*}
    \max_{i\in[I]} \brs*{\sum_{k=1}^{K'} g_k}\leq \max_{i\in[I]} \brs*{\brs*{\sum_{k=1}^{K'} g_k}_+} = \norm*{\brs*{\sum_{k=1}^{K'} g_k}_+}_\infty \leq \frac{\epsilon(K)}{\rho},
\end{align*}
for any $K'\in [K]$.
\end{proof}

\begin{remark}[Convexity of the RL Objective Function]
Although it is common to refer to the objective function in RL as non-convex, in the state action visitation polytope the objective is linear and, hence, convex (however, the problem is constraint to the state action visitation polytope). Thus, we can use Theorem~\ref{theorem: thoeorem constraint bound beck} and Cor.~\ref{corollary: bound on a dual variable} which are valid for constraint convex problems.
\end{remark}

\section{Constraint MDPs Primal Dual Approach}\label{supp: primal dual bonus optimistic model full proof}

In this section we establish regret guarantees for \primaldualcmdp by proving Theorem~\ref{theorem: primal dual optimistic CMDP}. Unlike for \dualcmdp, \primaldualcmdp requires an access to a (truncated) policy estimation algorithm which returns $\wh Q^\pi_h(s,a;\wt c_k,\wb p_k),Q^\pi_h(s,a;\wt d_{k,i},\wb p_k)$, i.e., the $Q$-function w.r.t. to the empirical transition model and optimistic cost and constraint cost. This reduces the computational complexity of \primaldualcmdp. However, it results in worse performance guarantees relatively to \dualcmdp.

Before supplying the proof of Theorem~\ref{theorem: primal dual optimistic CMDP} we formally define the set of good events which we show holds with high probability. Conditioning on the good, we establish the optimism of \primaldualcmdp and then regret bounds for \primaldualcmdp. The regret bounds of \primaldualcmdp relies on results from constraint convex optimization with some minor adaptations which we establish in Appendix~\ref{supp: convex optimization review}.

\subsection{Failure Events} \label{supp: primal dual optimistic model good events}

We define the same set of good events as for UCRL-OptCMDP (Appendix~\ref{supp: ucrl model optimistic good events}). We restate this set here for convenience.
\begin{align*}
    &F_k^p=\brc*{\exists s,a,s',h:\ \abs{ p_h(s' \mid s,a) - \wb p_h^{k-1}(s'\mid s,a)}\geq \beta^p_{h,k}(s,a,s')
    }\\
    &F^N_k = \brc*{\exists s,a,h: n_h^{k-1}(s,a) \le \frac{1}{2} \sum_{j<k} q_h^{\pi_k}(s,a\mid p)-H\ln\frac{SAH}{\delta'}}\\
    &F^{c}_{k} = \brc*{\exists s,a,h: |\wb{c}^k_h(s,a)- c_h(s,a) |\geq \beta_{h,k}^c(s,a)}\\
    &F^{d}_{k} = \brc*{\exists s,a,h,i\in [I]: |\wb{d}^k_{i,h}(s,a)- d_{i,h}(s,a) |\geq \beta_{i,h,k}^d(s,a) }
\end{align*}

As in Appendix~\ref{supp: ucrl model optimistic good events} the union of these events hold with probability greater than $1-\delta$. 

\begin{lemma}[Good event of \primaldualcmdp]\label{lemma: primal dual cmdp failure events}
Setting $\delta'=\frac{\delta}{3}$ then $\Pr\brc{\wb G}\leq \delta$ where $${\wb G = F^c \bigcup F^d \bigcup F^p\bigcup F^N}.$$ When the failure events does not hold we say the algorithm is outside the failure event, or inside the good event $G$ which is the complement of $\wb G$.
\end{lemma}

\begin{lemma}\label{lemma: primal dual probability ball around p}
Conditioned on the basic good event, for all $k,h,s,a,s'$ there exists constants $C_1,C_2>0$ for which we have that 
$$ \abs*{\wb p_h^{k-1}(s' \mid s,a) - p_h(s' \mid s,a)} =  C_1\sqrt{\frac{ p_h(s' \mid s,a) L_{\delta,p}}{ n_h^{k}(s,a) \vee 1}} + \frac{C_2 L_{\delta,p}}{n_h^{k}(s,a) \vee 1}, $$
where $ L_{\delta,p} = \ln\br*{\frac{6SAHK}{\delta}}.$
\end{lemma}

\subsection{Optimality and Optimism}

\begin{lemma}[On Policy Optimality.]\label{lemma: primal dual on policy optimality}
Conditioning on the good event, for any $k\in [K']$
\begin{align*}
    \sum_{k=1}^{K'}\wt f_k +  \lambda_k^T \wt g_k -f_{\pi^*} - \lambda_k^T g_{\pi^*}\leq \Olog( \sqrt{H^4(1+I\rho)^2 K})
\end{align*}
\end{lemma}
\begin{proof}
By definition,
\begin{align*}
    &f_{\pi^*} + \lambda_k^T g_{\pi^*} = V^{\pi^*}_1(s_1;c,p) +\sum_{i=1}^I\lambda_{k,i} V^{\pi^*}_1(s_1;d_{i},p) - \sum_{i=1}^I\lambda_{k,i}\alpha_i\\
    &\wt f_k +  \lambda_k^T \wt g_k = \wh V^{\pi_k}_1(s_1; \wt c_k,\wb p_k) + \sum_{i=1}^I\lambda_{k,i} \wh V^{\pi_k}_1(s_1; \wt d_{k,i},\wb p_k)- \sum_{i=1}^I\lambda_{k,i}\alpha_i.
\end{align*}
Let
\begin{align*}
    &Q^{k}_h(s,a) \eqdef Q^{\pi_k}_h(s,a;\wt c_k, \wb{p}_{k-1}) + \sum_{i=1}^I\lambda_{k,i}Q^{\pi_k}_h(s,a;\wt d_{k,i}, \wb{p}_{k-1})\\
    &V^{k}_h(s_1) \eqdef \inner{Q^{k}_h(s,\cdot), \pi_h^k}.
\end{align*}
Applying the extended value difference lemma~\ref{lemma: extended value difference} we get that
\begin{align*}
    &\sum_{k=1}^{K'}\wt f_k +  \lambda_k^T \wt g_k -f_{\pi^*} - \lambda_k^T g_{\pi^*}\\
    &=\sum_{k=1}^{K'} V^{k}_1(s_1) - V^{\pi^*}_1(s_1; c+\lambda_k \wt d, p)\\
    & =\underbrace{\sum_{k=1}^K \sum_{h=1}^H \E \brs*{ \inner*{Q_h^k(s_h,\cdot), \pi_h^k (\cdot \mid s_h )-\pi^*_h(\cdot \mid s_h )} \mid s_1 = s_1, \pi^*,p}}_{(\romannumeral 1)} \nonumber\\
    & + \sum_{k=1}^K \sum_{h=1}^H \E   \brs*{\underbrace{Q_h^k(s_h,a_h) - c_h(s_h,a_h)- \sum_{i=1}^I\lambda_k d_{h,i}(s_h,a_h) - p_h(\cdot \mid s_h,a_h) V_{h+1}^{k}}_{(\romannumeral 2)} \mid s_1 = s_1,\pi^*, p}.
\end{align*}
To bound $(i)$, we apply Lemma~\ref{lemma: term 2 omd term} while setting $\pi=\pi^*$.
\begin{align}
    (i)= \sum_{k=1}^{K'}\sum_{h=1}^H \E \brs*{  \inner*{Q_h^k(s_h,\cdot), \pi_h^k (\cdot \mid s_h )-\pi^*_h(\cdot \mid s_h )} \mid s_1 = s_1, \pi^*,p}\lesssim \sqrt{H^4(1+I\rho)^2 K}, \label{eq: sum optimiality bound term 1}
\end{align}
To bound $(ii)$, observe that by Lemma~\ref{lemma: primal dual optimism} for all $s,a,h,k$ it holds that
\begin{align*}
    Q_h^k(s,a) - c_h(s,a)- \sum_{i=1}^I\lambda_k d_{h,i}(s,a) - p_h(\cdot \mid s,a) V_{h+1}^{k}\leq 0.
\end{align*}
This implies that 
\begin{align}
    (ii)\leq 0 \label{eq: sum optimiality bound term 2}
\end{align}
since $(ii)$ is an expectation over negative terms. Combining \eqref{eq: sum optimiality bound term 1} and \eqref{eq: sum optimiality bound term 2} we conclude that
\begin{align*}
    \sum_{k=1}^{K'}\wt f_k +  \lambda_k^T \wt g_k -f_{\pi^*} - \lambda_k^T g_{\pi^*} =\sum_{k=1}^{K'} V^{k}_1(s_1) - V^{\pi^*}_1(s_1; c+\lambda_k \wt d, p) \lesssim \sqrt{H^4(1+I\rho)^2 K}.
\end{align*}
\end{proof}

\begin{lemma}[Policy Estimation Optimism] \label{lemma: primal dual optimism}
Conditioning on the good event, for any $s,a,h,k$ the following bound holds
\begin{align*}
    Q_h^k(s,a) - c_h(s,a) -\sum_{i=1}^I\lambda_k d_{h,i}(s,a) - p_h(\cdot \mid s,a) V_{h+1}^{k} \leq 0,
\end{align*}
where 
\begin{align}
    &Q_h^k(s,a) = Q^{\pi_k}_h(s,a;\wt c_k, \wb{p}_{k-1}) + \sum_{i=1}^I\lambda_{k,i}Q^{\pi_k}_h(s,a;\wt d_{k,i}, \wb{p}_{k-1}), \label{eq: Qk primal dual}\\
    &V^k_h(s) = \inner{Q_h^k(s,\cdot),\pi_h^k(\cdot\mid s)}.\label{eq: Vk primal dual}
\end{align}
\end{lemma}
See that $ Q^{\pi_k}_h(s,a;\wt c_k, \wb{p}_{k-1}), Q^{\pi_k}_h(s,a;\wt d_{k,i}, \wb{p}_{k-1})$ are defined in the update rule of \primaldualcmdp (Algorithm~\ref{alg: primal dual optimistic model cmdp}).
\begin{proof}

For all $s,a,h,k$ the following relations hold.
\begin{align}
    &Q_h^k(s,a) - c_h(s,a) -\sum_{i=1}^I\lambda_k d_{h,i}(s,a) - p_h(\cdot \mid s,a) V_{h+1}^{k} \nonumber \\
    =& Q^{\pi_k}_h(s,a;\wt c_k, \wb{p}_{k-1}) + \sum_{i=1}^I\lambda_{k,i}Q^{\pi_k}_h(s,a;\wt d_{k,i}, \wb{p}_{k-1}) \nonumber\\
    &-c_h(s,a) - \sum_{i=1}^I\lambda_{k,i} d_{h,i}(s,a) - p_h(\cdot \mid s,a) \br*{V^{\pi_k}_{h+1}(\cdot;\wt c_k, \wb{p}_{k-1}) + \sum_{i=1}^I\lambda_{k,i}V^{\pi_k}_{h+1}(\cdot;\wt d_{k,i}, \wb{p}_{k-1})}, \label{eq: primal dual optimism bounds}
\end{align}
where $V^{\pi_k}_h(\cdot;\wt c_{k}, \wb{p}_{k-1}) \eqdef \inner{Q^{\pi_k}_h(s,\cdot;\wt c_{k}, \wb{p}_{k-1}),\pi_h^k(\cdot,s)},\ V^{\pi_k}_h(\cdot;\wt d_{k,i}, \wb{p}_{k-1}) \eqdef \inner{Q^{\pi_k}_h(s,\cdot;\wt d_{k,i}, \wb{p}_{k-1}),\pi_h^k(\cdot,s)}.$
Furthermore, see that
\begin{align}
    Q^{\pi_k}_h(s,a;\wt c_k, \wb{p}_{k-1}) =&\max\brc*{0,\wt c_h^{k}(s,a) +\wb p_h^{k-1}(\cdot |s,a)V^{\pi_k}_{h+1}(\cdot ;\wt c_k,\wb p_k)} \nonumber\\
    =& \max\brc*{0,\wb c_h^{k-1}(s,a) - b_{h,k-1}(s,a) - b_{h,k-1}^p(s,a)  +\wb p_h^{k-1}(\cdot |s,a)V^{\pi_k}_{h+1}(\cdot ;\wt c_k,\wb p_k)}\nonumber\\
    \leq& \max\brc*{0,\wb c_h^{k-1}(s,a) - b_{h,k-1}(s,a)} \nonumber\\
    &+ \max\brc*{0,- b_{h,k-1}^p(s,a)  +\wb p_h^{k-1}(\cdot |s,a)V^{\pi_k}_{h+1}(\cdot ;\wt c_k,\wb p_k)}\label{eq: primal dual q c decouples},
\end{align}
since $\max\brc*{0,a+b}\leq \max\brc*{0,a} + \max\brc*{0,b}$. Similarly, for any $i\in[I]$,
\begin{align}
     Q^{\pi_k}_h(s,a;\wt d_{i,k}, \wb{p}_{k-1})\leq &\max\brc*{0,\wb d^{k-1}_{i,h}(s,a) - b_{h,k-1}(s,a)} \nonumber \\
     &+ \max\brc*{0,- b_{h,k-1}^p(s,a)+\wb p_h^{k-1}(\cdot |s,a)V^{\pi_k}_{h+1}(\cdot ;\wt d_{i,k},\wb p_k)}\label{eq: primal dual q d decouples}.
\end{align}

Plugging~\eqref{eq: primal dual q c decouples} and~\eqref{eq: primal dual q d decouples} into~\eqref{eq: primal dual optimism bounds} we get
\begin{align}
    &Q_h^k(s,a) - c_h(s,a) - p_h(\cdot \mid s,a) V_{h+1}^{k} \nonumber\\
    \leq & \max\brc*{0,\wb c_h^{k-1}(s,a) - b_{h,k-1}(s,a)} - c_h(s,a)\label{eq: primal dual optimsm 21}\\
    &+  \max\brc*{0,- b_{h,k-1}^p(s,a)  +\wb p_h^{k-1}(\cdot |s,a)V^{\pi_k}_{h+1}(\cdot ;\wt c_k,\wb p_k)} - p_h(\cdot \mid s,a) V^{\pi_k}_h(\cdot;\wt c_k, \wb{p}_{k-1})\label{eq: primal dual optimsm 22}\\
    &+ \sum_{i=1}^I\lambda_{k,i}\br*{\max\brc*{0,\wb d^{k-1}_{i,h}(s,a) - b_{h,k-1}(s,a)} - d_{h,i}(s,a) } \label{eq: primal dual optimsm 23}\\
    &+ \sum_{i=1}^I\lambda_{k,i}\br*{\max\brc*{0,- b_{h,k-1}^p(s,a)  +\wb p_h^{k-1}(\cdot |s,a)V^{\pi_k}_{h+1}(\cdot ;\wt d_{i,k},\wb p_k)} - p_h(\cdot\mid s,a)V^{\pi_k}_h(\cdot;\wt d_{k,i}, \wb{p}_{k-1})}.\label{eq: primal dual optimsm 24}
\end{align}

We now show each of these terms is negative conditioning on the good event. 
\begin{align*}
    \eqref{eq: primal dual optimsm 21}=&\max\brc*{0,\wb c_h^{k-1}(s,a) - b_{h,k-1}(s,a)} - c_h(s,a) \\
    =& \max\brc*{- c_h(s,a),\wb c_h^{k-1}(s,a)- c_h(s,a) - b_{h,k-1}(s,a)}\\
    \leq& \max\brc*{- c_h(s,a),\sqrt{\frac{L_\delta}{n^{k-1}_h(s,a)}} - b_{h,k-1}(s,a)}\\
    =&\max\brc*{- c_h(s,a),0}\leq 0.
\end{align*}
Furthermore, observe that
\begin{align}
    &- b_{h,k-1}^p(s,a)  +\wb p_h^{k-1}(\cdot |s,a)V^{\pi_k}_{h+1}(\cdot ;\wt c_k,\wb p_k) - p_h(\cdot \mid s,a) V^{\pi_k}_h(\cdot;\wt c_k, \wb{p}_{k-1}) \nonumber \\
    & \leq - b_{h,k-1}^p(s,a)  +\sum_{s'}|(\wb p_h^{k-1} - p_h)(s' |s,a)||V^{\pi_k}_{h+1}(s' ;\wt c_k,\wb p_k)|  \nonumber\\
    & \leq - b_{h,k-1}^p(s,a)  +H\sum_{s'}|(\wb p_h^{k-1} - p_h)(s' |s,a)|  \nonumber\\
    &\leq - b_{h,k-1}^p(s,a)  +2H\sqrt{\frac{\wb p_h^k(s'\mid s,a)\ln\br*{\frac{2SAHK}{\delta'}}}{n_h^{k-1}(s,a)\vee 1}} + \frac{14H \ln\br*{\frac{2SAHK}{\delta'}}}{3(n_h^{k-1}(s,a)-1\vee 1)}\nonumber \\
    &= - b_{h,k-1}^p(s,a) + b_{h,k-1}^p(s,a)=0. \label{eq: primal dual optimsm 3}
\end{align}
The second relation holds since $V^{\pi_k}_{h+1}(s' ;\wt c_k,\wb p_k) \eqdef \inner{Q^{\pi_k}_{h+1}(s',\cdot;\wt c_{k}, \wb{p}_{k-1}),\pi_h^k(\cdot,s)}\in[0,H]$ by the update rule (\primaldualcmdp uses truncated policy evaluation, see Algorithm~\ref{alg: truncated policy evaluation}). The third relation holds conditioning on the good event. The forth relation holds by the choice of $b_{h,k-1}^p$. Applying~\eqref{eq: primal dual optimsm 3} we get that
\begin{align*}
    \eqref{eq: primal dual optimsm 22} =& \max\brc*{0,- b_{h,k-1}^p(s,a)  +\wb p_h^{k-1}(\cdot |s,a)V^{\pi_k}_{h+1}(\cdot ;\wt c_k,\wb p_k)} - p_h(\cdot \mid s,a) V^{\pi_k}_h(\cdot;\wt c_k, \wb{p}_{k-1})\\
    &\leq\max\brc*{- p_h(\cdot \mid s,a) V^{\pi_k}_h(\cdot;\wt c_k, \wb{p}_{k-1}),- b_{h,k-1}^p(s,a)  +(\wb p_h^{k-1}-p_h)(\cdot |s,a)V^{\pi_k}_{h+1}(\cdot ;\wt c_k,\wb p_k)} \leq  0.
\end{align*}

Similarly, we get that each term in the sums at~\eqref{eq: primal dual optimsm 23},\eqref{eq: primal dual optimsm 24} is non-positive. Since $\lambda_k \geq0$ we conclude that both $\eqref{eq: primal dual optimsm 23}\leq 0$ and $\eqref{eq: primal dual optimsm 24}\leq 0$. Thus, we establish that
\begin{align*}
    Q_h^k(s,a) - c_h(s,a) - p_h(\cdot \mid s,a) V_{h+1}^{k}\leq 0.
\end{align*}
\end{proof}

\begin{lemma}[OMD Term Bound]\label{lemma: term 2 omd term}
Conditioned on the good event, we have that for any $\pi$
\begin{align*}
     \sum_{k=1}^K \sum_{h=1}^H \E \brs*{ \inner*{Q_h^k(s_h,\cdot), \pi_h^k (\cdot \mid s_h )-\pi_h(\cdot \mid s_h )} \mid s_1 = s, \pi,p} \leq \sqrt{2 H^4(1+I\rho)^2 K \log A}.
\end{align*}
\end{lemma}
\begin{proof}
This term accounts for the optimization error, bounded by the OMD analysis. 

By standard analysis of OMD~\citep{orabona2019modern} with the KL divergence used as the Bregman distance (see Lemma~\ref{lemma: fundamental inequality of OMD}) we have that for any $s,h$ and for policy any $\pi$,
\begin{align}
    \sum_{k=1}^K \inner*{ Q_h^k( \cdot \mid  s), \pi_h^k(\cdot \mid s) - \pi_h(\cdot \mid s) } \leq \frac{\log A}{t_K} + \frac{t_K}{2} \sum_{k=1}^K \sum_a \pi_h^k(a \mid s) (Q_h^k(s,a))^2 \label{eq: omds term analysis 1 relation}
\end{align}
where $t_K$ is a fixed step size.

By the form of $Q^k$~\eqref{eq: Qk primal dual} we get that $Q^k\geq 0$ since it is a sum of positive terms (policy evaluation is done with truncated policy evaluation, see Algorithm~\ref{alg: primal dual optimistic model cmdp}). Furthermore, we upper bound $Q^k$ for any $s,a,h,k$ as follows, 
\begin{align*}
    Q_h^k(s,a) &\eqdef Q^{\pi_k}_h(s,a;\wt c_k, \wb{p}_{k-1}) + \sum_{i=1}^I\lambda_{k,i}Q^{\pi_k}_h(s,a;\wt d_{k,i}, \wb{p}_{k-1})\\
    &\leq H+ H\sum_{i=1}^I\lambda_{k,i} \leq H+HI\rho.
\end{align*}
The second relation holds by the fact that $Q^{\pi_k}_h(s,a;\wt c_k, \wb{p}_{k-1}),Q^{\pi_k}_h(s,a;\wt d_{k,i}, \wb{p}_{k-1})\leq H$ by the update rule (both $\wt c_k,\wt d_{i,k}\leq 1$, thus, an expectation over an $H$ such terms is smaller than $H$) and the fact $\lambda_k\geq 0$ (by the update rule).

Plugging this bound into~\eqref{eq: omds term analysis 1 relation} we get that for any $s,a,h$
\begin{align}
\sum_{k=1}^{K'}  \inner*{ Q_h^{k}(s,\cdot),\pi_h^k(\cdot\mid s) - \pi_h(\cdot\mid s)}  \leq \frac{\log A}{t_K} +  \frac{t_K H^2(1+I\rho)^2 K }{2}. \label{eq: label term 2 MD}
\end{align}

Thus, the following relations hold.
\begin{align*}
    &\sum_{k=1}^K \sum_{h=1}^H \E \brs*{ \inner*{Q_h^k(s_h,\cdot), \pi_h^k (\cdot \mid s_h )-\pi_h(\cdot \mid s_h )} \mid s_1 = s, \pi,p}\\
    &=\sum_{h=1}^H  \E \brs*{\sum_{k=1}^K \inner*{Q_h^k(s_h,\cdot), \pi_h^k (\cdot \mid s_h )-\pi_h(\cdot \mid s_h )} \mid s_1 = s, \pi,p}\\
    &\leq \sum_{h=1}^H \E \brs*{\frac{\log A}{t_K} +  t_K H^2 K  \mid s_1 = s, \pi} = \frac{H\log A}{t_K} +  \frac{t_K H^3(1+I\rho)^2 K }{2}.
\end{align*}

See that the first relation holds as the expectation does not depend on $k$. Thus, by linearity of expectation, we can switch the order of summation and expectation. The second relation holds since~\eqref{eq: label term 2 MD} holds for any $s$.

Finally, by choosing $t_K=\sqrt{2\log A /(H^2(1+I\rho)^2 K)}$, we obtain

\begin{align}
    \sum_{k=1}^K \sum_{h=1}^H \E \brs*{ \inner*{Q_h^k(s_h,\cdot), \pi_h^k (\cdot \mid s_h )-\pi_h(\cdot \mid s_h )} \mid s_1 = s, \pi,p} \leq \sqrt{2 H^4(1+I\rho)^2 K \log A}.
\end{align}
\end{proof}

\subsection{Proof of Theorem~\ref{theorem: primal dual optimistic CMDP}}
 
In this section, we establish the following regret bound for \primaldualcmdp.
 
\theoremPrimalDualOptimisticModel*

We start by proving several useful lemmas on which the proof is based upon.
\begin{lemma}[Dual Optimism] \label{lemma: primal dual constraint to optimality}
Conditioning on the good event, for any $k\in[K']$
\begin{align*}
    \wt f_k - f_{opt} \leq  -\lambda_k^T \wt g_k + \br*{\wt f_k +  \lambda_k^T \wt g_k -f_{\pi^*} - \lambda_k^T g_{\pi^*}}
\end{align*}
\end{lemma}
\begin{proof}
We have that
\begin{align*}
    f_{opt} = c^T q^{\pi^*}(p) &\geq c^T q^{\pi^*}(p) + \lambda_k^T (Dq^{\pi^*}(p)-\alpha) \\
    &= f_{\pi^*} + \lambda_k^T g_{\pi^*}\\
    & = \wt f_k +  \lambda_k^T \wt g_k +f_{\pi^*} + \lambda_k^T g_{\pi^*} - \wt f_k -  \lambda_k^T \wt g_k.
\end{align*}
The first relation holds since $\pi^*$ satisfies the constraint (Assumption~\ref{assum: feasability}) which implies that ${(Dq^{\pi^*}(p)-\alpha)\leq 0}$, and that $\lambda_k\geq0$ by the update rule. 
\end{proof}

We now state a lemma which corresponds to Lemma~\ref{lemma: incremental alg recusrion relation} from previous section. 

\begin{restatable}
[Update Rule Recursion Bound Primal-Dual]{lemma}{LemmaPrimalDualUpdateRuleRecursion}\label{lemma: incremental primal dual alg recusrion relation}
For any $\lambda\in \brc*{\lambda\in\mathbb{R}^I: 0 \leq \lambda \leq \rho {\bf 1}}$ and $K'\in[K]$
\begin{align*}
    \sum_{k=1}^{K'}\br*{-\wt g_k^T\lambda_{k}}  + \sum_{k=1}^{N}\wt g_k^T \lambda \leq  \frac{t_\lambda}{2}\norm{\lambda_{1} - \lambda}^2_2 + \frac{1}{2t_\lambda}\sum_{k=1}^{K'}\norm{\wt g_k}^2
\end{align*}
\end{restatable}
\begin{proof}
Similar proof to Lemma~\ref{lemma: incremental alg recusrion relation} while using the fact that projection to the set $\brc*{\lambda\in\mathbb{R}^I: 0 \leq \lambda \leq \rho {\bf 1}}$ is non-expansive operator as the operator $\brs{x}_+$.
\end{proof}






We are now ready to establish Theorem~\ref{theorem: primal dual optimistic CMDP}.
\begin{proof}
Applying Lemma~\ref{lemma: primal dual constraint to optimality} into Lemma~\ref{lemma: incremental primal dual alg recusrion relation}  we get
\begin{align*}
    &\sum_{k=1}^{K'}\br*{\wt f_k - f_{opt}}  + \sum_{k=1}^{K'}\wt g_k^T \lambda \\
    &\leq \sum_{k=1}^{K'}\br*{-\wt g_k^T\lambda_{k}}  + \sum_{k=1}^{K'}\wt g_k^T \lambda  + \sum_{k=1}^{K'}\wt f_k +  \lambda_k^T \wt g_k -f_{\pi^*} - \lambda_k^T g_{\pi^*}\\
    &\leq  \frac{t_\lambda}{2}\norm{\lambda_1-\lambda}_2^2 + \frac{1}{2t_\lambda}\sum_{k=1}^{K'}\norm{\wt g_k}^2+ \sum_{k=1}^{K'}\wt f_k +  \lambda_k^T \wt g_k -f_{\pi^*} - \lambda_k^T g_{\pi^*}.
\end{align*}

Adding, subtracting $\sum_{k=1}^{K'} g_k^T \lambda, \sum_{k=1}^{K'} f_k$ and rearranging we get
\begin{align}
    &\sum_{k=1}^{K'}\br*{f_k - f_{opt}}  + \sum_{k=1}^{K'} g_k^T \lambda \nonumber\\
    &\leq \frac{t_\lambda}{2}\norm{\lambda}_2^2 +  \frac{1}{2t_\lambda}\sum_{k=1}^{K'}\norm{\wt g_k}^2 + \sum_{k=1}^{K'} (g_k-\wt g_k)^T \lambda + \sum_{k=1}^{K'} (f_k-\wt f_k) \nonumber \\
    &\quad +\sum_{k=1}^{K'}\wt f_k +  \lambda_k^T \wt g_k -f_{\pi^*} - \lambda_k^T g_{\pi^*}\nonumber \\
    &\leq \frac{t_\lambda}{2}\norm{\lambda}_2^2 +  \frac{1}{2t_\lambda}\sum_{k=1}^{K'}\norm{\wt g_k}^2 + \sqrt{\sum_{i=1}^I\br*{\sum_{k=1}^{K'} (g_{k,i}-\wt g_{k,i})}^2} \norm{\lambda}_2 + \sum_{k=1}^{K'} (f_k-\wt f_k) \nonumber\\
    &\quad+\sum_{k=1}^{K'}\wt f_k +  \lambda_k^T \wt g_k -f_{\pi^*} - \lambda_k^T g_{\pi^*}\label{eq: primal dual central bound constraint and optimality}
\end{align}
for any $\lambda\in \mathbbm{R}^I_+$, where the last relation holds by Cauchy Schwartz inequality.

We now bound each term in~\eqref{eq: primal dual central bound constraint and optimality}. Since $\wt{g_k}\in[-H,H]$
\begin{align*}
    \frac{1}{2t_\lambda}\sum_{k=1}^{K'}\norm{\wt g_k}^2\leq \frac{H^2IK}{2t_\lambda}.
\end{align*}

Applying Lemma~\ref{lemma: on policy errors truncated policy estimation} (see that assumptions (1),(2),(3) hold conditioning on the good event), we get that
\begin{align*}
    &\abs*{\sum_{k=1}^{K'} (f_k-\wt f_k)} = \abs*{\sum_{k=1}^{K'} (V^{\pi_k}(s_1;c,p)-\wh V^{\pi_k}(s_1;\wt c_k,\wb p_k)}=  \leq \Olog\br*{\sqrt{S\mathcal{N}H^4K} + (\sqrt{\mathcal{N}}+H)H^2SA }\\
    &\abs*{\sum_{k=1}^{K'} (g_{k,i}-\wt g_{k,i})} = \abs*{\sum_{k=1}^{K'} (V^{\pi_k}(s_1;d_i,p)-\wh V^{\pi_k}(s_1;\wt d_{k,i},\wb p_k)} \leq \Olog\br*{\sqrt{S\mathcal{N}H^4K} + (\sqrt{\mathcal{N}}+H)H^2SA },
\end{align*}
which implies that
\begin{align*}
    \sqrt{\sum_{i=1}^I\br*{\sum_{k=1}^{K'} (g_{k,i}-\wt g_{k,i})}^2}\leq \Olog\br*{\sqrt{IS\mathcal{N}H^4K} + (\sqrt{\mathcal{N}}+H)\sqrt{I}H^2SA}.
\end{align*}
Lastly, by Lemma~\ref{lemma: primal dual on policy optimality},
\begin{align*}
    \sum_{k=1}^{K'}\wt f_k +  \lambda_k^T \wt g_k -f_{\pi^*} - \lambda_k^T g_{\pi^*}\lesssim \sqrt{H^4(1+I\rho)^2 K}.
\end{align*}
Plugging these bounds back into~\eqref{eq: primal dual central bound constraint and optimality} and setting $t_\lambda = \sqrt{\frac{H^2IK}{\rho^2}}$ we get
\begin{align}
    &\sum_{k=1}^{K'}\br*{f_k - f_{opt}}  + \sum_{k=1}^{K'} g_k^T \lambda \nonumber\\
    &\lesssim (\rho+\frac{\norm{\lambda}_2^2}{\rho})\sqrt{H^2IK} +\br*{\sqrt{IS\mathcal{N}H^4K} + (\sqrt{\mathcal{N}}+H)\sqrt{I}H^2SA}\norm{\lambda}_2 \nonumber\\
    &\quad + \br*{\sqrt{S\mathcal{N}H^4K} + (\sqrt{\mathcal{N}}+H)H^2SA }+ \sqrt{H^4(1+I\rho)^2 K} \label{eq: primal dual central second bound constraint and optimality},
\end{align}
for any $0\leq \lambda\leq \rho {\bf 1}$.

\paragraph{First claim of Theorem~\ref{theorem: primal dual optimistic CMDP}}. Fix $\wb \lambda=0$ which satisfies $0\leq \wb \lambda\leq \rho {\bf 1}$    in~\eqref{eq: primal dual central second bound constraint and optimality} we get
\begin{align*}
    &\sum_{k=1}^{K'} V^{\pi_k}(s_1;c,p) - V^*(s_1) = \sum_{k=1}^{K'}f_k - f_{opt} \\
    &\leq \Olog\br*{\sqrt{S\mathcal{N}H^4K} + \sqrt{H^4(1+I\rho)^2 K} + (\sqrt{\mathcal{N}}+H)H^2SA }.
\end{align*}

\paragraph{Second claim of Theorem~\ref{theorem: primal dual optimistic CMDP}.} Fix $i\in [I]$ and let 
\begin{align*}
    \wb \lambda_i = \begin{cases}
            \rho e_i & \brs{\sum_{k=1}^{K'} g_{i,k}}_+ \neq 0\\
            0 & \text{otherwise}
    \end{cases}
\end{align*}
where $e_i(i)= 1$ and $e_i(j)=0$ for $j\neq i$, and $\rho$ is given in Assumption~\ref{assum: slater point} See that $0\leq \wb\lambda_i\leq \rho{\bf 1}$. Furthermore, it holds that
\begin{align}
    \norm{ \wb \lambda_i}_2^2\leq \rho^2 \label{eq: primal dual bound on lambda i norm}
\end{align}
Set $\lambda = \wb \lambda_i$ in \eqref{eq: primal dual central second bound constraint and optimality}  we get
\begin{align}
    &\sum_{k=1}^{K'}\br*{f_k - f_{opt}}  + \rho \brs*{\sum_{k=1}^{K'} g_{i,k}}_+ \nonumber \\
    &\lesssim (1+\rho)\br*{\sqrt{IS\mathcal{N}H^4K} + (\sqrt{\mathcal{N}}+H)\sqrt{I}H^2SA}+ \sqrt{H^4(1+I\rho)^2 K} \eqdef \epsilon(K) \label{eq: primal dual central bound incremental recursion}
\end{align}
where we applied~\eqref{eq: primal dual bound on lambda i norm} in the second relation. Since the bound~\eqref{eq: primal dual central bound incremental recursion} holds for any $i$ we get that
\begin{align*}
    &\max_{i\in [I]} \sum_{k=1}^{K'}\br*{f_k - f_{opt}}  + \rho \brs*{\sum_{k=1}^{K'} g_{i,k}}_+=   \sum_{k=1}^{K'}\br*{f_k - f_{opt}}  + \rho \max_{i\in [I]}\brs*{\sum_{k=1}^{K'} g_{i,k}}_+ \\
    & =  \sum_{k=1}^{K'}\br*{f_k - f_{opt}}  + \rho \max_{i\in [I]}\abs*{\brs*{\sum_{k=1}^{K'} g_{i,k}}_+}\\
    & =  \sum_{k=1}^{K'}\br*{f_k - f_{opt}}  + \rho \norm*{\brs*{\sum_{k=1}^{K'} g_{k}}_+}_\infty \leq \epsilon(K).
\end{align*}

Now, by the convexity of the state-action frequency function (Proposition~\ref{prop: convexity of state action frequencey}) there exists a policy $\pi_{K'}$ which satisfies $q^{\pi_{K'}}(p) =\frac{1}{K'}\sum_{k=1}^{K'}q^{\pi_k}(p)$ for any $K'$. Since both $f$ and $g$ are linear in $\frac{1}{K'}\sum_{k=1}^{K'}q^{\pi_k}(p)$ we have that
\begin{align*}
    \frac{1}{K'}\br*{\sum_{k=1}^{K'}\br*{f_k - f_{opt}}  + \rho \norm*{\brs*{\sum_{k=1}^{K'} g_k}_+}_2} = f_{\pi_{K'}} - f_{opt} + \rho \norm*{\brs*{g_{\pi_{K'}}}_+}_2 \leq \frac{1}{K'}\epsilon(K).
\end{align*}

Applying Corollary~\ref{corollary: bound on a dual variable} and Theorem~\ref{theorem: thoeorem constraint bound beck} we conclude that
\begin{align*}
    \max_{i\in[I]} \brs*{\sum_{k=1}^{K'} g_k}\leq \max_{i\in[I]} \brs*{\brs*{\sum_{k=1}^{K'} g_k}_+} = \norm*{\brs*{\sum_{k=1}^{K'} g_k}_+}_\infty \leq \frac{\epsilon(K)}{\rho},
\end{align*}
for any $K'\in [K]$.
\end{proof}

\section{Bounds of On-Policy Errors}

\begin{lemma}[On Policy Errors for Optimistic Model]\label{lemma: on policy errors optimistic model}

Let $l_h(s,a),  \wt l^k_h(s,a)$ be a a cost function, and its optimistic cost. Let $p$ be the true transition dynamics of the MDP and $\wt p_k$ be an estimated transition dynamics. Let $V^{\pi}_h(s;l,p),V^{\pi}_h(s;\wt l_k,\wt p_k)$ be the value of a policy $\pi$ according to the cost and transition model $l,p$ and $\wt l_k,\wt p_k$, respectively.  Assume the following holds for all $s,a,h,k\in [K]$:
\begin{enumerate}[label={(\alph*)}]
    \item $|\wt l^k_h(s,a) - l_h(s,a)| \lesssim \frac{1}{\sqrt{n_h^{k-1}(s,a)}}.$ 
\item $|\wt p_h^k(s'\mid s,a) - p_h(s'\mid s,a)| \lesssim \sqrt{\frac{ p_h(s' \mid s,a)}{ n_h^{k-1}(s,a) \vee 1}} + \frac{1}{ n_h^{k-1}(s,a) \vee 1}.$
\item $n_h^{k-1}(s,a) \le \frac{1}{2} \sum_{j<k} q_h^{\pi_k}(s,a\mid p)-H\ln\frac{SAH}{\delta'}.$
\end{enumerate}

Furthermore, let $\pi_k$ be the policy by which the agent acts at the $k^{th}$ episode. Then, for any $K'\in [K]$
\begin{align*}
    \sum_{k=1}^{K'} |V^{\pi_k}_1(s_1; l,p) - V^{\pi_k}_1(s_1; \wt l_k,\wt p_k)| \leq \Olog\br*{\sqrt{S\mathcal{N}H^4K} + (\sqrt{\mathcal{N}}+H)H^2SA }.
\end{align*}
\end{lemma}
\begin{proof}
The following relations hold.
\begin{align*}
    &\sum_{k=1}^{K'} |V^{\pi_k}_1(s_1; l,p) - V^{\pi_k}_1(s_1; \wt l_k,\wt p_k)| \\
    &= \sum_{k=1}^{K'} \abs*{\E[\sum_{h=1}^H (l_h(s_h,a_h) - \wt l^k_h(s_h,a_h) ) + (p_h - \wt p_h^k)(\cdot \mid s_h,a_h)\wt V^{\pi_k}_{h+1}\mid s_1, p,\pi_k]}\\
    &\leq \underbrace{\sum_{k=1}^{K'} \E[\sum_{h=1}^H |l_h(s_h,a_h) - \wt l^k_h(s_h,a_h)| \mid s_1, p,\pi_k]}_{(i)}\\
    &\quad +\underbrace{\sum_{k=1}^{K'} \E[\sum_{h=1}^H \sum_{s'}|(p_h - \wt p_h^k)(s' \mid s_h,a_h)||\wt V^{\pi_k}_{h+1}(s' ; \wt l_k,\wt p_k)|\mid s_1, p,\pi_k]}_{(ii)},
\end{align*}
where the first relation holds by the value difference Lem.~\ref{lemma: value difference}. We now bound the terms $(i)$ and $(ii)$.

\paragraph{Bound on $(i)$.} To bound $(i)$ we use the assumption (1) and get,
\begin{align*}
    (i) & \lesssim \sum_{k=1}^{K'}\sum_{h=1}^H \E[ \frac{1}{\sqrt{n_h^{k-1}(s_h,a_h)}} \mid s_1, p,\pi_k]\\
    & = \sum_{k=1}^{K'}\sum_{h=1}^H \E[ \frac{1}{\sqrt{n_h^{k-1}(s^k_h,a^k_h)}} \mid \filt] \leq \Olog\br*{\sqrt{SAH^2K} +SAH}.
\end{align*}

The first relation holds by assumption \emph{(a)}. The second relation holds since $\pi_k$ is the policy by which the agent acts at episode $k$ in the true MDP. The third relation holds by Lem.~\ref{lemma: bounding on trajectory visitation}.

\paragraph{Bound on $(ii)$.} To bound $(ii)$ use the fact that
\begin{align}
    | V^{\pi_k}_{h+1}(s;\wt l_k,\wt p_k)|\lesssim H \label{eq: optimistic model ucrl bound on wt V}
\end{align}
for every $s$ since the immediate cost is bounded in $\abs{\wt{l}_h^k(s,a)} \lesssim l_h(s,a) +\frac{1}{\sqrt{n_h^{k-1}(s,a)}}\lesssim l_h(s,a)$ component-wise up to constants, since the second term is bounded by $\Olog(1)$. Thus,

\begin{align*}
    (ii)&\lesssim H\sum_{k=1}^{K'}\sum_{h=1}^H \E[\sqrt{\frac{1}{ n_h^{k}(s_h,a_h) \vee 1}}\sum_{s'}\sqrt{p_h(s' \mid s_h,a_h)} + \frac{S}{n_h^{k}(s_h,a_h) \vee 1} \mid s_1, p,\pi_k]\\
    &\leq  H\sum_{k=1}^{K'}\sum_{h=1}^H \E[\sqrt{\frac{1}{ n_h^{k}(s_h,a_h) \vee 1}}\sqrt{\mathcal{N}}\sqrt{\sum_{s'} p_h(s' \mid s_h,a_h)} + \frac{S}{n_h^{k}(s_h,a_h) \vee 1} \mid s_1, p,\pi_k]\\
    &=  H\sum_{k=1}^{K'}\sum_{h=1}^H \E[\sqrt{\frac{1}{ n_h^{k}(s_h,a_h) \vee 1}}\sqrt{\mathcal{N}} + \frac{S}{n_h^{k}(s_h,a_h) \vee 1} \mid s_1, p,\pi_k]\\
    &=  H\sum_{k=1}^{K'}\sum_{h=1}^H \E[\sqrt{\frac{1}{ n_h^{k}(s^k_h,a^k_h) \vee 1}}\sqrt{\mathcal{N}} + \frac{S}{n_h^{k}(s^k_h,a^k_h) \vee 1} \mid \filt ]\\
    &\lesssim \sqrt{S\mathcal{N}H^4K} + \sqrt{\mathcal{N}}H^2SA + SH^3A  \leq \Olog\br*{\sqrt{S\mathcal{N}H^4K} + (\sqrt{\mathcal{N}}+H)H^2SA }.
\end{align*}
The first relation holds by plugging the bound~\eqref{eq: optimistic model ucrl bound on wt V} and assumption \emph{(b)} into $(ii)$. The second relation holds by Jensen's inequality. The third relation holds since $p$ is a probability distribution. The forth relation holds since $\pi_k$ is the policy with which the agent interacts with the true CMDP. The fifth relation holds by Lem.~\ref{lemma: bounding on trajectory visitation} (its assumption holds by assumption \emph{(c)}).

Combining the bounds on $(i)$ and $(ii)$ we conclude the proof.
\end{proof}

\begin{lemma}[On Policy Errors for Truncated Policy Estimation]\label{lemma: on policy errors truncated policy estimation}

Let $l_h(s,a),  \wt l^k_h(s,a)$ be a a cost function, and its optimistic cost. Let $p$ be the true transition dynamics of the MDP and $\wb p_k$ be an estimated transition dynamics. Let $V^{\pi}_h(s;l,p)$ be the value of a policy $\pi$ according to the cost and transition model $l,p$. Furthermore, let $\wh V^{\pi}_h(s;\wt l_k,\wb p_k)$ be a value function calculated by a truncated value estimation (see Algorithm~\ref{alg: truncated policy evaluation}) by the cost and transition model $\wt l_k,\wb p_k$.  Assume the following holds for all $s,a,h,k\in [K]$:
\begin{enumerate}
    \item $|\wt l^k_h(s,a) - l_h(s,a)| \lesssim \frac{1}{\sqrt{n_h^{k-1}(s,a)}}.$ 
\item $|\wb p_h^k(s'\mid s,a) - p_h(s'\mid s,a)| \lesssim \sqrt{\frac{ p_h(s' \mid s,a)}{ n_h^{k-1}(s,a) \vee 1}} + \frac{1}{ n_h^{k-1}(s,a) \vee 1}.$
\item $n_h^{k-1}(s,a) \le \frac{1}{2} \sum_{j<k} q_h^{\pi_k}(s,a\mid p)-H\ln\frac{SAH}{\delta'}.$
\end{enumerate}

Furthermore, let $\pi_k$ be the policy by which the agent acts at the $k^{th}$ episode. Then, for any $K'\in [K]$
\begin{align*}
    \sum_{k=1}^{K'} |V^{\pi_k}_1(s_1; l,p) - \wh V^{\pi_k}_1(s_1; \wt l_k,\wb p_k)| \leq \Olog\br*{\sqrt{S\mathcal{N}H^4K} + (\sqrt{\mathcal{N}}+H)H^2SA }.
\end{align*}
\end{lemma}
\begin{proof}
The following relations hold.
\begin{align}
    &\sum_{k=1}^{K'} |V^{\pi_k}_1(s_1; l,p) - \wh V^{\pi_k}_1(s_1; \wt l_k,\wt p_k)| \\
    &= \sum_{k=1}^{K'} \abs*{\E[\sum_{h=1}^H (l_h(s_h,a_h) - p_h(\cdot \mid s_h,a_h)\wt V^{\pi_k}_{h+1} - \wh Q^{\pi_k}(s_h,a_h; \wt l^k,\wb p_k)\mid s_1, p,\pi_k]} \label{eq: truncated value difference bound relation 1}
\end{align}
Observe that
\begin{align*}
    -Q^{\pi_k}(s_h,a_h; \wt l^k,\wb p_k) = \min\brc*{0, -\l_h^k(s_h,a_h) - \wb p_h^k(\cdot\mid s_h,a_h) \wh V^{\pi_k}},
\end{align*}
where the first relation holds by the extended value difference lemma~\ref{lemma: extended value difference}. Plugging back to~\eqref{eq: truncated value difference bound relation 1} we get
\begin{align*}
    &\eqref{eq: truncated value difference bound relation 1}\leq \underbrace{\sum_{k=1}^{K'} \E[\sum_{h=1}^H |l_h(s_h,a_h) - \wt l^k_h(s_h,a_h)| \mid s_1, p,\pi_k]}_{(i)}\\
    &\quad +\underbrace{\sum_{k=1}^{K'} \E[\sum_{h=1}^H \sum_{s}|(p_h - \wt p_h^k)(s' \mid s_h,a_h)||\wh V^{\pi_k}_{h+1}(s' ; \wt l_k,\wt p_k)|\mid s_1, p,\pi_k]}_{(ii)},
\end{align*}
We now bound the terms $(i)$ and $(ii)$.

\paragraph{Bound on $(i)$.} To bound $(i)$ we use the assumption (1) and get,
\begin{align*}
    (i) & \lesssim \sum_{k=1}^{K'}\sum_{h=1}^H \E[ \frac{1}{\sqrt{n_h^{k-1}(s_h,a_h)}} \mid s_1, p,\pi_k]\\
    & = \sum_{k=1}^{K'}\sum_{h=1}^H \E[ \frac{1}{\sqrt{n_h^{k-1}(s^k_h,a^k_h)}} \mid \filt] \leq \Olog\br*{\sqrt{SAH^2K} +SAH}.
\end{align*}

The first relation holds by assumption (1). The second relation holds since $\pi_k$ is the policy by which the agent acts at the $k^{th}$ episode at the true MDP. The third relation holds by Lemma~\ref{lemma: bounding on trajectory visitation}.

\paragraph{Bound on $(ii)$.} To bound $(ii)$ use the fact that
\begin{align}
    | \wh V^{\pi_k}_{h+1}(s;\wt l_k,\wt p_k)|\lesssim H \label{eq: truncated value primal dual bound on wt V}
\end{align}
for every $s$ since the immediate cost is bounded in $\abs{\wt{l}_h^k(s,a)} \lesssim l_h(s,a)\leq 1 +\frac{1}{\sqrt{n_h^{k-1}(s,a)}}\lesssim l_h(s,a)$ component-wise up to constants, since the second term is bounded by $\Olog(1)$. Thus,

\begin{align*}
    (ii)&\lesssim H\sum_{k=1}^{K'}\sum_{h=1}^H \E[\sqrt{\frac{1}{ n_h^{k}(s_h,a_h) \vee 1}}\sum_{s'}\sqrt{p_h(s' \mid s_h,a_h)} + \frac{S}{n_h^{k}(s_h,a_h) \vee 1} \mid s_1, p,\pi_k]\\
    &\leq  H\sum_{k=1}^{K'}\sum_{h=1}^H \E[\sqrt{\frac{1}{ n_h^{k}(s_h,a_h) \vee 1}}\sqrt{\mathcal{N}}\sqrt{\sum_{s'} p_h(s' \mid s_h,a_h)} + \frac{S}{n_h^{k}(s_h,a_h) \vee 1} \mid s_1, p,\pi_k]\\
    &=  H\sum_{k=1}^{K'}\sum_{h=1}^H \E[\sqrt{\frac{1}{ n_h^{k}(s_h,a_h) \vee 1}}\sqrt{\mathcal{N}} + \frac{S}{n_h^{k}(s_h,a_h) \vee 1} \mid s_1, p,\pi_k]\\
    &=  H\sum_{k=1}^{K'}\sum_{h=1}^H \E[\sqrt{\frac{1}{ n_h^{k}(s^k_h,a^k_h) \vee 1}}\sqrt{\mathcal{N}} + \frac{S}{n_h^{k}(s^k_h,a^k_h) \vee 1} \mid \filt ]\\
    &\lesssim \sqrt{S\mathcal{N}H^4K} + \sqrt{\mathcal{N}}H^2SA + SH^3A  \leq \Olog\br*{\sqrt{S\mathcal{N}H^4K} + (\sqrt{\mathcal{N}}+H)H^2SA }.
\end{align*}
The first relation holds by plugging the bound~\eqref{eq: truncated value primal dual bound on wt V} and assumption (2) into $(ii)$. The second relation holds by Jensen's inequality. The third relation holds since $p$ is a probability distribution. The third relation holds since $\pi_k$ is the policy with which the agent interacts with the true MDP~$p$. The fifth relation holds by Lemma~\ref{lemma: bounding on trajectory visitation} (its assumption holds by assumption (3)).

Combining the bounds on $(i)$ and $(ii)$ we conclude the proof.
\end{proof}

\begin{lemma}[On Policy Errors for Bonus Based Optimism]\label{lemma: on policy errors bonus based optimistic}
        Let $l_h(s,a),  \wt l^k_h(s,a)$ be a cost function, and its optimistic cost. Let $p$ be the true transition dynamics of the MDP and $\wb p_{k-1}$ be an estimated transition dynamics. Let $V^{\pi}_h(s;l,p),V^{\pi}_h(s;\wt l_k,\wb p_{k-1})$ be the value of a policy $\pi$ according to the cost and transition model $l,p$ and $\wt l_k,\wb p_{k-1}$, respectively.  Assume the following holds for all $s,a,s',h,k\in [K]$:
\begin{enumerate}
        \item $|\wt l^k_h(s,a) - l_h(s,a)| \lesssim \sqrt{\frac{1}{n_h^{k-1}(s,a)  \vee 1}} + \sum_{s'}H\sqrt{\frac{\wb p_h^{k-1}(s' \mid s,a)}{n_h^{k-1}(s,a) \vee 1}} + \frac{HS}{ \br*{(n_h^{k-1}(s,a) - 1) \vee 1}}.$ 
    \item $ \abs*{\wb p_h^{k-1}(s' \mid s,a) - p_h(s'\mid s,a)} \lesssim \sqrt{\frac{ \wb p_h(s' \mid s,a)}{(n_h^{k-1}(s,a)-1) \vee 1}} + \frac{1}{ (n_h^{k-1}(s,a) - 1) \vee 1}.$
\item $n_h^{k-1}(s,a) \le \frac{1}{2} \sum_{j<k} q_h^{\pi_k}(s,a; p)-H\ln\frac{SAH}{\delta'}.$
\item $V^{\pi_k}_h(s;\wt l_k,\wb p_{k-1}) \leq V^{\pi_k}_h(s; l, p).$ 
\item $l_h(s,a)-\wt l_h^k(s,a) + (p_h(\cdot\mid s,a)-\wb p_h^{k-1}(\cdot \mid s,a))V^{\pi}_{h+1}(\cdot|l,p) \geq 0.$
\end{enumerate}




Let $\pi_k$ be the policy by which the agent acts at episode $k$. Then, for any $K'\in [K]$
\begin{align*}
        \sum_{k=1}^{K'} V^{\pi_k}_1(s_1; l,p) - V^{\pi_k}_1(s_1; \wt l_k,\wb p_{k-1}) \leq\Olog\br*{\sqrt{S\mathcal{N}H^4K} + S^2 H^{4}A(\mathcal{N}H+S)}.
\end{align*}
\end{lemma}
\begin{proof}
        Denote for any $s,h$ $\wt V^{\pi_k}_{h}(s) = V^{\pi_k}_{h}(s;\wt l_k,\wb p_{k-1})$ and $V^{\pi_k}_{h}(s) = V^{\pi_k}_{h}(s;l,p)$. The following relations hold:
\begin{align}
    &\sum_{k=1}^{K'} V^{\pi_k}_1(s_1) - \wt V^{\pi_k}_1(s_1) \nonumber \\
    &= \sum_{k=1}^{K'}\sum_{h=1}^H \E\Big[\ (l_h(s_h,a_h) - \wt l^k_h(s_h,a_h) ) + (p_h - \wb p_h^{k-1})(\cdot \mid s_h,a_h)\wt V^{\pi_k}_{h+1} \Big| s_1, p,\pi_k \Big] \nonumber\\
    &\leq \underbrace{\sum_{k=1}^{K'}\sum_{h=1}^H \E\Big[ |l_h(s_h,a_h) - \wt l^k_h(s_h,a_h)| \; \Big|\; s_1, p,\pi_k \Big]}_{(i)} \nonumber\\
    &\quad +\underbrace{\sum_{k=1}^{K'}\sum_{h=1}^H \E\Big[ \sum_{s'}\abs*{(p_h - \wb p_h^{k-1})(s' \mid s_h,a_h)}\abs{V^{\pi_k}_{h+1}(\cdot;l,p)(s')}\; \Big|\; s_1, p,\pi_k \Big]}_{(ii)} \nonumber\\
    &\quad +\underbrace{\sum_{k=1}^{K'}\sum_{h=1}^H \E\Big[ \abs*{(p_h - \wb p_h^{k-1})(\cdot \mid s_h,a_h)( V^{\pi_k}_{h+1}(\cdot;\wt l_k,\wb p_{k-1}) - V^{\pi_k}_{h+1}(\cdot;l,p))} \; \Big|\; s_1, p,\pi_k \Big]}_{(iii)}, \label{eq: central bound lemma bonus based optimsim}
\end{align}
where the first relation holds by the value difference lemma (see Lem.~\ref{lemma: value difference}).

\paragraph{Bound on $(i)$ and $(ii)$.} Since $0\leq V^{\pi_k}_{h+1}(\cdot;l,p)(s)\leq H$ (the value of the true MDP is bounded in $[0,H]$), we can bound both $(i)$ and $(ii)$ by the same analysis as in Lem.~\ref{lemma: on policy errors optimistic model}. Thus,
\begin{align*}
    (i)+(ii) \leq \sqrt{S\mathcal{N}H^4K} + (\sqrt{\mathcal{N}}+H)H^2SA.
\end{align*}

\paragraph{Bound on $(iii)$.} Applying Lem.~\ref{lemma: lower order bonus lp} we obtain the following bound
\begin{align*}
    (iii)\lesssim S^2 H^{4}A(\mathcal{N}H+S) + \sqrt{\mathcal{N}}S H^{5/2}\sqrt{A}\sqrt{\sum_{k}(V^{\pi_k}_{1}(s_1) - \wt{V}^{\pi_k}_{1}(s_1))}.
\end{align*}

Plugging the bounds on terms $(i),\ (ii),$ and $(iii)$ into~\eqref{eq: central bound lemma bonus based optimsim} we get
\begin{align*}
    &\sum_{k=1}^{K'} V^{\pi_k}_1(s_1) - \wt V^{\pi_k}_1(s_1)\\
    &\lesssim \sqrt{S\mathcal{N}H^4K} + S^2 H^{4}A(\mathcal{N}H+S) + \sqrt{\mathcal{N}}S H^{5/2}\sqrt{A}\sqrt{\sum_{k}(V^{\pi_k}_{1}(s_1) - \wt{V}^{\pi_k}_{1}(s_1))}.
\end{align*}

Denoting $X = \sum_{k=1}^{K'} V^{\pi_k}_1(s_1) - \wt V^{\pi_k}_1(s_1) $ this bound has the form $0\leq X\leq a + b\sqrt{X}$, where
\begin{align*}
    &a = \sqrt{S\mathcal{N}H^4K} + S^2 H^{4}A(\mathcal{N}H+S)\\
    &b =  \sqrt{\mathcal{N}}S H^{5/2}\sqrt{A}.
\end{align*}
Applying Lem.~\ref{lemma: self bounding property}, by which $X\leq a + b^2$, we get
\begin{align*}
    \sum_{k=1}^{K'} V^{\pi_k}_1(s_1) - \wt V^{\pi_k}_1(s_1) \lesssim \sqrt{S\mathcal{N}H^4K} + S^2 H^{4}A(\mathcal{N}H+S).
\end{align*}

\end{proof}

\begin{lemma}\label{lemma: lower order bonus lp}
Let the assumptions of Lem.~\ref{lemma: on policy errors bonus based optimistic} hold. Then, for any $K'\in [K]$
\begin{align*}
    &\sum_{k=1}^{K'}\sum_{h=1}^H \E\brs*{ \abs*{(p_h - \wb p_h^{k-1})(\cdot \mid s^k_h,a^k_h)( V^{\pi_k}_{h+1}(\cdot;\wt l_k,\wb p_{k-1}) - V^{\pi_k}_{h+1}(\cdot;l,p))} \mid \filt}\\
    &\hspace{2cm}\lesssim   S^2 H^{4}A(\mathcal{N}H+S) + \sqrt{\mathcal{N}}S H^{5/2}\sqrt{A}\sqrt{\sum_{k}(V^{\pi_k}_{1}(s_1) - \wt{V}^{\pi_k}_{1}(s_1))}.
\end{align*}
\end{lemma}

\begin{proof}

        Denote for any $s,h$ $\wt V^{\pi_k}_{h}(s) = V^{\pi_k}_{h}(s;\wt l_k,\wb p_{k-1})$ and $V^{\pi_k}_{h}(s) = V^{\pi_k}_{h}(s;l,p)$. The following relations hold:
\begin{align}
    &\sum_k \E\brs*{\sum_{t=1}^H \abs*{(p_h- \wb p_h^{k-1})(\cdot \mid s_h,a_h) (\wt{V}^{\pi_k}_{h+1} - V^{\pi_k}_{h+1})}\mid s_1,\pi_k,p} \nonumber \\
    &=\sum_{k,h,s,a} q_h^{\pi_k}(s,a;p)\abs*{(p_h-\wb p_h^{k-1})(\cdot \mid s,a) (\wt{V}^{\pi_k}_{h+1} - V^{\pi_k}_{h+1})}  \nonumber\\
    &\leq  \sum_{k,h,s,a} q_h^{\pi_k}(s,a;p)\sum_{s'}\abs*{(p_h-\wb p_h^{k-1})(s' \mid s,a)} \abs*{\wt{V}^{\pi_k}_{h+1}(s') - V^{\pi_k}_{h+1}(s')}  \nonumber\\
    &\lesssim  \underbrace{\sum_{k,h,s,a} q_h^{\pi_k}(s,a;p)\sum_{s'}\frac{\sqrt{p_h(s'\mid s,a)}}{\sqrt{n_h^k(s,a)}} \abs*{\wt{V}^{\pi_k}_{h+1}(s') - V^{\pi_k}_{h+1}(s')}}_{(i)} + \underbrace{\sum_{k,h,s,a} q_h^{\pi_k}(s,a;p)\frac{H^2S^2}{n_h^k(s,a)}}_{(ii)}. \label{eq: bonus lower order relation 1}
\end{align}

In the third relation we used  assumption (2) of Lem.~\ref{lemma: on policy errors bonus based optimistic} as well as bounding 
\begin{align} \label{eq: bound on differnce of V with optmistic bonus}
    \abs*{\wt{V}^{\pi_k}_{h+1}(s) - V^{\pi_k}_{h+1}(s)}\lesssim SH^2
\end{align}
since $\wt{V}^{\pi_k}_{h+1}(s) \in [-SH^2,H]$ by the assumption on its instantaneous cost (assumption (1) of Lem.~\ref{lemma: on policy errors bonus based optimistic}). Note that ${V^{\pi_k}_{h+1}(s)\in[0,H]}$ as usual.

Term $(ii)$ is bounded as follows
\begin{align}
    (ii)  = H^2S^2 \sum_{k,h} \E\brs*{ \frac{1}{n_h^k(s_h^k,a_h^k)}\mid s_1,\pi_k,p}  = H^2S^2 \sum_{k,h} \E\brs*{ \frac{1}{n_h^k(s_h^k,a_h^k)}\mid \filt} \lesssim H^4S^3A, \label{eq: bonus lower order relation 2}
\end{align}
by Lem.~\ref{lemma: supp 1 1/N  factor and lograthimic factors}.

We now bound term $(i)$ as follows.
\begin{align}
     &(i)\leq  \sum_{k}\sum_{s,a,h} q_h^{\pi_k}(s,a;p)  \frac{ \sqrt{ \mathcal{N}\sum_{s'} p_h(s'\mid s,a)(\wt{V}^{\pi_k}_{h+1}(s') - V^{\pi_k}_{h+1}(s'))^2}}{\sqrt{n_h^k(s,a)}} \nonumber\\
    & \leq \sqrt{\mathcal{N}}\sqrt{\sum_{k}\sum_{s,a,h} q_h^{\pi_k}(s,a;p)  \frac{1}{n_h^k(s,a)}}\sqrt{\sum_{k}\sum_{s,a,h} \sum_{s'} q_h^{\pi_k}(s,a;p)   p_h(s'\mid s,a)(\wt{V}^{\pi_k}_{h+1}(s') - V^{\pi_k}_{h+1}(s'))^2 } \nonumber\\
    & = \sqrt{\mathcal{N}}\sqrt{\sum_{k}\sum_{s,a,h} q_h^{\pi_k}(s,a;p)  \frac{1}{n_h^k(s,a)}}\sqrt{\sum_{k}\sum_{s',a,h} q_{h+1}^{\pi_k}(s',a;p) (\wt{V}^{\pi_k}_{h+1}(s') - V^{\pi_k}_{h+1}(s'))^2 } \nonumber\\
    &  \lesssim \sqrt{\mathcal{N}}S  H^{2}\sqrt{A} \sqrt{\sum_{k}\sum_{s,a,h} q_{h+1}^{\pi_k}(s,a;p) (V^{\pi_k}_{h+1}(s) - \wt{V}^{\pi_k}_{h+1}(s)) } \nonumber\\
    & \leq \sqrt{\mathcal{N}}S H^{5/2}\sqrt{A}\sqrt{\sum_{k}(V^{\pi_k}_{1}(s_1) - \wt{V}^{\pi_k}_{1}(s_1)) + \sum_{k,h,s,a} q_{h}^{\pi_k}(s,a;p)\abs*{(p_h-\wb p_h)(\cdot \mid s,a) (\wt{V}^{\pi_k}_{h+1} - V^{\pi_k}_{h+1})}} \nonumber\\
    & \leq \sqrt{\mathcal{N}}S  H^{5/2}\sqrt{A}\sqrt{\sum_{k}(V^{\pi_k}_{1}(s_1) - \wt{V}^{\pi_k}_{1}(s_1))} \nonumber\\
    &\quad  + \sqrt{\mathcal{N}}S H^{5/2}\sqrt{A}\sqrt{ \sum_{k,h,s,a} q_{h}^{\pi_k}(s,a;p)\abs*{(p_h-\wb p_h)(\cdot \mid s,a) (\wt{V}^{\pi_k}_{h+1} - V^{\pi_k}_{h+1})}}. \label{eq: bonus lower order relation 3}
\end{align}

The first relation holds by Jensen's inequality while using the fact that $p_h(\cdot\mid s,a)$ has at most $\mathcal{N}$ non-zero terms. The second relation holds by Cauchy-Schwartz inequality.
The third relation follows from properties of the occupancy measure (see Eq.~\ref{eq:occupancy_space}). In particular, $\sum_{s,a} p_h(s'|s,a) q_h(s,a;p) = \sum_a q_{h+1}(s',a;p)$.
The forth relation holds by applying Lem.~\ref{lemma: supp 1 1/N  factor and lograthimic factors} and bounding $(\wt{V}^{\pi_k}_{h+1}(s) - V^{\pi_k}_{h+1}(s))^2 \lesssim SH^2 (V^{\pi_k}_{h+1}(s)-\wt{V}^{\pi_k}_{h+1}(s))$ due to~\eqref{eq: bound on differnce of V with optmistic bonus} and $V^{\pi_k}_{h+1}(s)-\wt{V}^{\pi_k}_{h+1}(s)\geq 0$ due to optimism (assumption (4) of Lem.~\ref{lemma: on policy errors bonus based optimistic}). The fifth relation holds by Lemma~\ref{lemma: future differences to initial difference} (see that its assumption holds by assumption $(5)$). The sixth relation holds by $\sqrt{a+b}\leq \sqrt{a}+\sqrt{b}$.


Plugging the bounds on term $(i)$,~\eqref{eq: bonus lower order relation 2}, and term $(ii)$,~\eqref{eq: bonus lower order relation 3}, into \eqref{eq: bonus lower order relation 1} we get
\begin{align*}
    &\sum_{k,h,s,a}  q_h^{\pi_k}(s,a;p)\abs*{(p_h-\wb p_h)(\cdot \mid s,a) (\wt{V}^{\pi_k}_{h+1} - V^{\pi_k}_{h+1})}\\
    &\leq H^4S^3A + \sqrt{\mathcal{N}}S H^{5/2}\sqrt{A}\sqrt{\sum_{k}(V^{\pi_k}_{1}(s_1) - \wt{V}^{\pi_k}_{1}(s_1))} \nonumber\\
    &\quad  + \sqrt{\mathcal{N}}S H^{5/2} \sqrt{A}\sqrt{ \sum_{k,h,s,a}  q_h^{\pi_k}(s,a;p)\abs*{(p_h-\wb p_h)(\cdot \mid s,a) (\wt{V}^{\pi_k}_{h+1} - V^{\pi_k}_{h+1})}}.
\end{align*}

Denoting $X = \sum_{k,h,s,a}  q_h^{\pi_k}(s,a;p)\abs*{(p_h-\wb p_h)(\cdot \mid s,a) (\wt{V}^{\pi_k}_{h+1} - V^{\pi_k}_{h+1})}$ this bound has the form $0\leq X\leq a + b\sqrt{X}$, where
\begin{align*}
    &a = H^4S^3A + \sqrt{\mathcal{N}}S H^{5/2}\sqrt{A}\sqrt{\sum_{k}(V^{\pi_k}_{1}(s_1) - \wt{V}^{\pi_k}_{1}(s_1))}\\
    &b =  \sqrt{\mathcal{N}}S H^{5/2} \sqrt{A}.
\end{align*}
Applying Lem.~\ref{lemma: self bounding property}, by which $X\leq a + b^2$, we get
\begin{align*}
    &\sum_{k,h,s,a} q_h^{\pi_k}(s,a;p)\abs*{(p_h-\wb p_h)(\cdot \mid s,a) (\wt{V}^{\pi_k}_{h+1} - V^{\pi_k}_{h+1})}\\
    &\leq H^4S^3A + \sqrt{\mathcal{N}}S H^{5/2}\sqrt{A}\sqrt{\sum_{k}(V^{\pi_k}_{1}(s) - \wt{V}^{\pi_k}_{1}(s))} + \mathcal{N}S^2 H^{5}A\\
    &\leq S^2 H^{4}A(\mathcal{N}H+S) + \sqrt{\mathcal{N}}S H^{5/2}\sqrt{A}\sqrt{\sum_{k}(V^{\pi_k}_{1}(s) - \wt{V}^{\pi_k}_{1}(s))}
\end{align*}

\end{proof}

\begin{lemma}\label{lemma: future differences to initial difference}
Let $l_h(s,a),  \wt l_h(s,a)$ be a cost function and its optimistic cost. Let $p,\wb p$ be two transition probabilities. Let $V^{\pi}_h(s)\eqdef V^{\pi}_h(s;l,p)$ and $\wt V^\pi_h(s)\eqdef V^{\pi}_h(s;\wt l_k,\wb p)$ be the value of a policy $\pi$ according to the cost and transition model $l,p$ and $\wt l,\wb p_k$, respectively. Assume that
\begin{align}
     l_h(s,a)-\wt l_h(s,a) + (p_h(\cdot\mid s,a)-\wb p_h(\cdot \mid s,a))V^{\pi}_{h+1} \geq 0, \label{eq: assumption recursion sum h to initial state}
\end{align}
for any $s,a,h$. Then, for any $\pi$ and $s$
\begin{align*}
    &\sum_{h=2}^H \E\brs*{ V^{\pi}_h(s_h) - \wt V^{\pi}_h(s_h) \mid s_1=s,\pi,p}\\
    &\leq  H\br*{V^{\pi}_1(s) -  \wt V^{\pi}_1(s)} +H\sum_{h=1}^{H}\E\brs*{ \abs*{(p_{h}(\cdot\mid s_{h},a_{h})-\wb p_{h}(\cdot \mid s_{h},a_{h'}))(\wt V^{\pi}_{h+1}-V^{\pi}_{h+1})}\mid s_1=s,\pi,p}
\end{align*}
\end{lemma}
\begin{proof}
By definition
\begin{align}
    &V^{\pi}_1(s) -  \wt V^{\pi}_1(s) \nonumber \\
    &=\E\brs*{ V^{\pi}_1(s_1) - l_1(s_1,a_1) - p_1(\cdot\mid s_1,a_1)\wt V^{\pi}_2\mid s_1=s,\pi,P} \nonumber \\
    &\quad + \E\brs*{ l_1(s_1,a_1) + p_1(\cdot\mid s_1,a_1)\wt V^{\pi}_2 - \wt V^{\pi}_1(s)\mid s_1=s,\pi,P}  \nonumber\\
    &=\E\brs*{ V^{\pi}_2(s_2) - \wt V^{\pi}_2(s_2) \mid s_1=s,\pi,P} \nonumber\\
    &\quad + \E\brs*{l_1(s_1,a_1)-\wt l_1(s_1,a_1) + (p_1(\cdot\mid s_1,a_1)-\wb p_1(\cdot \mid s_1,a_1))\wt V^{\pi}_2\mid s_1=s,\pi,P}\nonumber \\
    &=\E\brs*{ V^{\pi}_2(s_2) - \wt V^{\pi}_2(s_2) \mid s_1=s,\pi,P} \nonumber\\
    &\quad + \E\brs*{(p_1(\cdot\mid s_1,a_1)-\wb p_1(\cdot \mid s_1,a_1))(\wt V^{\pi}_2 - V^{\pi}_2)\mid s_1=s,\pi,P} \nonumber \\
    &\quad + \E\brs*{l_1(s_1,a_1)-\wt l_1(s_1,a_1) + (p_1(\cdot\mid s_1,a_1)-\wb p_1(\cdot \mid s_1,a_1))V^{\pi}_2\mid s_1=s,\pi,P}\nonumber\\
    &\geq\E\brs*{ V^{\pi}_2(s_2) - \wt V^{\pi}_2(s_2) \mid s_1=s,\pi,P} \nonumber\\
    &\quad + \E\brs*{(p_1(\cdot\mid s_1,a_1)-\wb p_1(\cdot \mid s_1,a_1))(\wt V^{\pi}_2 - V^{\pi}_2)\mid s_1=s,\pi,P},
\end{align}
where the first relation holds by the value difference lemma~\ref{lemma: value difference} and the last relation holds due to the assumption~\ref{eq: assumption recursion sum h to initial state}.


Iterating on this relation we get that for any $h\in\brc*{2,..H}$
\begin{align*}
     &V^{\pi}_1(s) -  \wt V^{\pi}_1(s) \\
     &\geq \E\brs*{ V^{\pi}_h(s_h) - \wt V^{\pi}_h(s_h) \mid s_1=s,\pi,P}\\
     &\quad+\sum_{h'=1}^{h-1}\E\brs*{ (p_{h'}(\cdot\mid s_{h'},a_{h'})-\wb p_{h'}(\cdot \mid s_{h'},a_{h'}))(\wt V^{\pi}_{h'+1}-V^{\pi}_{h'+1})\mid s_1=s,\pi,P}.
\end{align*}

By summing this relation for $h\in\brc*{2,..H}$ and rearranging we get
\begin{align*}
    &H\br*{V^{\pi}_1(s) -  \wt V^{\pi}_1(s)} -\sum_{h=2}^H\sum_{h'=1}^{h-1}\E\brs*{ (p_{h'}(\cdot\mid s_{h'},a_{h'})-\wb p_{h'}(\cdot \mid s_{h'},a_{h'}))(\wt V^{\pi}_{h'+1}-V^{\pi}_{h'+1})\mid s_1=s,\pi,P}\\
    &\geq \sum_{h=2}^H \E\brs*{ V^{\pi}_h(s_h) - \wt V^{\pi}_h(s_h) \mid s_1=s,\pi,P}.
\end{align*}

Thus,
\begin{align*}
    &\sum_{h=2}^H \E\brs*{ V^{\pi}_h(s_h) - \wt V^{\pi}_h(s_h) \mid s_1=s,\pi,P}\\
    &\leq H\br*{V^{\pi}_1(s) -  \wt V^{\pi}_1(s)} +\sum_{h=2}^H\sum_{h'=1}^{h-1}\E\brs*{\br*{- (p_{h'}(\cdot\mid s_{h'},a_{h'})-\wb p_{h'}(\cdot \mid s_{h'},a_{h'}))(\wt V^{\pi}_{h'+1}-V^{\pi}_{h'+1})}\mid s_1=s,\pi,P}\\
    &\leq H\br*{V^{\pi}_1(s) -  \wt V^{\pi}_1(s)} +\sum_{h=2}^H\sum_{h'=1}^{H}\E\brs*{ \abs*{(p_{h'}(\cdot\mid s_{h'},a_{h'})-\wb p_{h'}(\cdot \mid s_{h'},a_{h'}))(\wt V^{\pi}_{h'+1}-V^{\pi}_{h'+1})}\mid s_1=s,\pi,P}\\
    &\leq H\br*{V^{\pi}_1(s) -  \wt V^{\pi}_1(s)} +H\sum_{h=1}^{H}\E\brs*{ \abs*{(p_{h}(\cdot\mid s_{h},a_{h})-\wb p_{h}(\cdot \mid s_{h},a_{h'}))(\wt V^{\pi}_{h+1}-V^{\pi}_{h+1})}\mid s_1=s,\pi,P}.
\end{align*}
\end{proof}


\section{Useful Lemmas}
We start stating the value difference lemma (a.k.a. simulation lemma). This lemma has been used in several papers~\citep[\eg][]{cai2019provably,efroni2020optimistic}.
The following lemma is central for the analysis of \primaldualcmdp.
\begin{restatable}[Extended Value Difference]{lemma}{lemmaExtendedValueDiff}\label{lemma: extended value difference}
Let $\pi,\pi'$ be two policies, and $\mathcal{M} = (\sset, \aset, \brc*{p_h}_{h=1}^H, \brc*{c_h}_{h=1}^H)$ and $\mathcal{M}' = (\sset, \aset, \brc*{p'_h}_{h=1}^H, \brc*{c'_h}_{h=1}^H)$ be two MDPs.
Let $\wh Q_h^{\pi}(s,a;c,p)$ be an approximation of the $Q$-function of policy $\pi$ on the MDP $\mathcal{M}$ for all $h,s,a$, and let  ${\wh V_h^{\pi}(s;c,p) = \inner*{\wh Q_h^{\pi}(s,\cdot;c,p) ,\pi_h(\cdot\mid s)}}$.
Then,
\begin{align*}
    & \wh V_1^{\pi}(s_1;c,p) - V_1^{\pi'}(s_1;c',p')=
    \\
    &
    \sum_{h=1}^H  \E \brs*{ \inner*{\wh Q_h^{\pi}(s_h,\cdot;c,p), \pi'_h(\cdot \mid s_h) - \pi_h(\cdot 
    \mid s_h)} \mid s_1,\pi',p'}+
    \\
    &   
    \sum_{h=1}^H \E \brs*{\wh Q_h^{\pi}(s_h,a_h;c,p)  - c_h' (s_h,a_h) - p'_h(\cdot | s_h,a_h) \wh V_{h+1}^{\pi}(\cdot ; c,p)\mid  s_1,\pi', p'}
\end{align*}
where $V_1^{\pi'}(s;c',p')$ is the value function of $\pi'$ in the MDP $\mathcal{M}'$.
\end{restatable}

The following lemma is standard~\citep[see \eg][Lem. E.15]{dann2017unifying}, and can be seen as a corollary of the extended value difference lemma.
\begin{lemma}[Value difference lemma] \label{lemma: value difference}
Consider two MDPs $\mathcal{M} = (\sset, \aset, \brc*{p_h}_{h=1}^H, \brc*{c_h}_{h=1}^H)$ and $\mathcal{M}' = (\sset, \aset, \brc*{p'_h}_{h=1}^H, \brc*{c'_h}_{h=1}^H)$. For any policy $\pi$ and any  $s,h$ the following relation holds.
\begin{align*}
    &V^\pi_h(s; c,p) - V^\pi_h(s;c',p') \\
    &= \E[\sum_{h'=h}^H (c_h(s_h,a_h) - c'_h(s_h,a_h)) + (p_h - p'_h)(\cdot \mid s_h,a_h)V^{\pi}_{h+1}(\cdot; c,p)\mid s_h=s,\pi,p'] \\
    &= \E[\sum_{h'=h}^H (c'_h(s_h,a_h) - c_h(s_h,a_h)) + (p'_h - p_h)(\cdot \mid s_h,a_h)V^{\pi}_{h+1}(\cdot; c',p')\mid s_h=s,\pi,p].
\end{align*}
\end{lemma}

The following lemmas are standard. There proof can be found in~\citep{dann2017unifying,zanette2019tighter,efroni2019tight} (\eg \citealt{efroni2019tight}, Lem. 38).

\begin{lemma}\label{lemma: bounding on trajectory visitation}
Assume that for all $s,a,h,k\in[K]$
\begin{align*}
     n_h^{k-1}(s,a) > \frac{1}{2} \sum_{j<k} q_h^{\pi_k}(s,a; p)-H\ln\frac{SAH}{\delta'},
\end{align*}
then
$$\sum_{k=1}^K\sum_{h=1}^H \E\brs*{ \sqrt{\frac{1}{n_{h}^{k-1}(s_h^k,a_h^k)\vee 1}} \mid \F_{k-1} }\leq \Olog(\sqrt{SAH^2K} +SAH)$$
\end{lemma}

\begin{lemma}[\eg \cite{zanette2019tighter}, Lem. 13]
\label{lemma: supp 1 1/N  factor and lograthimic factors}
Assume that for all $s,a,h,k\in[K]$
\begin{align*}
     n_h^{k-1}(s,a) > \frac{1}{2} \sum_{j<k} q_h^{\pi_k}(s,a; p)-H\ln\frac{SAH}{\delta'},
\end{align*}
then
$$\sum_{k=1}^K\sum_{t=1}^H \E\brs*{ {\frac{1}{n_{k-1}(s_t^k,a_t^k)\vee 1}} \mid \F_{k-1} }\leq \Olog\br*{SAH^2}.$$
\end{lemma}

\begin{lemma}[Consequences of Self Bounding Property]\label{lemma: self bounding property}
Let $0\leq X\leq a + b\sqrt{X}$ where $X,a,b\in \mathbb{R}$. Then,
\begin{align*}
    X\lesssim a + b^2.
\end{align*}
\end{lemma}
\begin{proof}
We have that
\begin{align*}
    X - b\sqrt{X} - a \leq 0.
\end{align*}
Since $X\geq 0$ this implies that
\begin{align*}
    \sqrt{X} &\leq \frac{b}{2} +\sqrt{\frac{1}{4}b^2 + 4a}\\
    & \leq \frac{b}{2} + \sqrt{\frac{b^2}{4}} + \sqrt{4a} \leq b + 2\sqrt{a},
\end{align*}
where we used the relation $\sqrt{a+b}\leq \sqrt{a}+\sqrt{b}$.

Since $\sqrt{X}\geq 0$ by squaring the two sides of the later inequality we get 
\begin{align*}
    X \leq (b + 2\sqrt{a})^2 \leq 2b^2 + 4a\lesssim b^2+a,
\end{align*}
where in the second relation we used the relation $(a+b)^2\leq 2a^2+2b^2$.

\end{proof}

\subsection{Online Mirror Descent}

In each iteration of Online Mirror Descent (OMD) the following problem is solved:

\begin{align}\label{eq: OMD iterates}
x_{k+1} \in \arg\min_{x\in C} t_K \inner*{g_k, x - x_k } + \bregman{x}{x_k},
\end{align}
where $t_K$ is a stepsize, and $\bregman{x}{x_k}$ is the bregman distance.

When choosing $\bregman{x}{x_k}$ as the KL-divergence, and the set $C$ is the unit simplex OMD has the following closed form,
\begin{align*}
 &x_{k+1}\in \arg\min_{x \in C} \{ t_K\inner*{\nabla f_k(x_k), x - x_k} +  \dkl{x}{x_k} \},
\end{align*}

The following lemma~\citep[Theorem 10.4]{orabona2019modern}  provides a fundamental inequality which will be used in our analysis.
\begin{lemma}[Fundamental inequality of Online Mirror Descent]\label{lemma: OMD orabona}
Assume $g_{k,i} \geq 0$ for $k=1,...,K$ and $i=1,...,d$. Let $C = \Delta_d$. Using OMD with the KL-divergence, learning rate $t_K$, and with uniform initialization, $x_1=[1/d,...,1/d]$, the following holds for any $u\in \Delta_d$,
$$ \sum_{k=1}^K \inner*{ g_t, x_k - u } \leq \frac{\log d}{t_K} + \frac{t_K}{2} \sum_{k=1}^K \sum_{i=1}^d x_{k,i} g_{k,i}^2 $$
\end{lemma}

In our analysis, we will be solving the OMD problem for each time-step $h$ and state $s$ separately,
\begin{align}\label{eq: RL OMD iterates}
     \pi_h^{k+1}(\cdot \mid s) \in \arg\min_{\pi \in \Delta_A} t_K \inner*{Q_h^k(s,\cdot), \pi - x_h^k(\cdot \mid s) } + \dkl{\pi}{\pi_h^k(\cdot\mid s)}.
\end{align}

Therefore, by adapting the above lemma to our notation, we get the following lemma,

\begin{lemma}[Fundamental inequality of Online Mirror Descent for RL]\label{lemma: fundamental inequality of OMD}
Let $t_K>0$. Let $\pi_h^1(\cdot \mid s)$ be the uniform distribution for any $h\in[H]$ and $s\in \sset$. Assume that $Q^k_h(s,a)\in[0,M]$ for all $s,a,h,k$. Then, by solving \eqref{eq: RL OMD iterates} separately for any $k\in[K], h\in[H]$ and $s\in \sset$, the following holds for any stationary policy $\pi$,
$$\sum_{k=1}^K \inner*{ Q_h^k( \cdot \mid  s), \pi_h^k(\cdot \mid s) - \pi_h(\cdot \mid s) } \leq \frac{\log A}{t_K} + \frac{t_K M^2K}{2}$$
\end{lemma}
\begin{proof}
First, observe that for any $k,h,s$, we solve the optimization problem defined in \eqref{eq: RL OMD iterates} which is the same as \eqref{eq: OMD iterates}. By the fact that the estimators used in our analysis are non-negative, we can apply Lemma~\ref{lemma: OMD orabona} separately for each $h,s$ with $g_k = Q_h^k(s,\cdot)$ and $x_k = \pi_h^k(s,\cdot) $. Lastly, bounding $(Q_h^k(s,a))^2\leq M^2$ and $\sum_{a}\pi^k_h(a\mid s)=1$ for all $s$ concludes the result.
\end{proof}





\section{Useful Results from Constraint Convex Optimization}\label{supp: convex optimization review}

In this section we enumerate several results from constraint convex optimization which are central to establish the bounds for the dual algorithms. To keep the generality of discussion, we follow results from~\cite{beck2017first}, Chapter 3, and consider a general constraint convex optimization problem
\begin{align}
    \fopt = \min_{\xv\in X} \brc*{f(\xv) : \gv(\xv)\leq {\bf 0}, \A\xv + \bv = {\bf 0}},\label{supp: convex cos problem}
\end{align}
where $\gv(\xv) \eqdef \br*{g_1(\xv),..,g_I(\xv)}^T,$ and $f,g_1,..,g_{m}: \mathbb{E}\rightarrow (-\infty,\infty)$ are convex real valued functions, $\A\in \mathbb{R}^{p\times n},\bv\in \mathbb{R}^p$. By defining the vector of constraints 

We define a value function associated with~\eqref{supp: convex cos problem} 
\begin{align*}
    v(\uv,\tv) =\min_{\xv\in X} \brc*{f(\xv) : \gv(\xv)\leq \uv, \A\xv + \bv = \tv},
\end{align*}
Furthermore, we define the dual problem to~\eqref{supp: convex cos problem}. The dual function is
\begin{align*}
    q(\lambda,\mu) = \min_{x\in X} \brc*{L(\xv,\lambda,\mu) = f(\xv) +\lambda^T\gv(\xv) +\mu^T(\A\xv+\bv)},
\end{align*}
where $\lambda\in \mathbb{R}_+^m,\mu\in \mathbb{R}^p$ and the dual problem is 
\begin{align}
    \qopt = \max_{\lambda\in \mathbb{R}_+^m,\mu\in \mathbb{R}^p}\brc*{q(\lambda,\mu):(\lambda,\mu)\in \mathrm{dom}(-q)}. \label{supp: convex cos problem dual}
\end{align}
Where $\mathrm{dom}(-q) = \brc*{(\lambda,\mu)\in \mathbb{R}_+^m,\mu\in \mathbb{R}^p: q(\lambda,\mu)>-\infty}$. Furthermore, denote an optimal solution of~\eqref{supp: convex cos problem dual} by $\lambda^*,\mu^*$.

We make the following assumption which will be verified to hold. The assumption implies strong duality, i.e., $\qopt=\fopt$.
\begin{assumption}\label{supp: convex con optimization assumption non empty}
The optimal value of~\eqref{supp: convex cos problem} is finite and exists a slater point $\wb \xv$ such that $g(\wb \xv)<0$ and exists a point $\wh \xv \in \mathrm{ri}(X)$ satifying $\A\wh\xv +\bv = 0$, where $\mathrm{ri}(X)$ is the relative interior of $X$.
\end{assumption}

The following theorem is proved in~\citealt{beck2017first}.
\begin{theorem}[\citealt{beck2017first}, Theorem 3.59.]\label{theorem: beck subdifferential}
$(\lambda^*,\mu^*)$ is an optimal solution of~\eqref{supp: convex cos problem dual} iff $$-(\lambda^*,\mu^*)\in \partial v({\bf 0},{\bf 0}).$$ Where $\partial f(\xv)$ denotes the set of all sub-gradients of $f$ at $\xv$.
\end{theorem}

Using this result we arrive to the following theorem, which is a variant of \citealt{beck2017first},~Theorem~3.60.

\begin{theorem}\label{theorem: thoeorem constraint bound beck}
Let $\lambda^*$ be an optimal solution of the dual problem~\eqref{supp: convex cos problem dual} and assume that $ 2\norm{\lambda^*}_1\leq \rho$. Let $\wt \xv$ satisfy $\A\wt\xv+\bv=0$ and
\begin{align*}
    f(\wt \xv) - f_{opt} + \rho \norm{\brs{g(\wt \xv)}_+}_\infty \leq \delta,
\end{align*}
then
\begin{align*}
   \norm{\brs{g(\wt \xv)}_+}_\infty \leq \frac{\delta}{\rho}.
\end{align*}
\end{theorem}
\begin{proof}
Let
\begin{align*}
    v(\uv,\tv) = \min_{\xv\in X}\brc*{f(\xv): g(\xv)\leq \uv, \A\xv+\bv =\tv}.
\end{align*}

Since $(-\lambda^*,\mu^*)$ is an optimal solution of the dual problem it follows by Theorem~\ref{theorem: beck subdifferential} that $(-\lambda^*,\mu^*)\in \partial v({\bf 0},{\bf 0})$. Therefore, for any $(\uv,{\bf 0})\in \mathrm{dom}(v)$
\begin{align}
    v(\uv,{\bf 0}) - v({\bf 0},{\bf 0}) \geq \inner{-\lambda^*,\uv}. \label{eq: beck proof rel 1}
\end{align}
Set $\uv= \wt \uv \eqdef \brs*{g(\wt \xv)}_+$. See that $\wb \uv\geq 0$ which implies that
\begin{align*}
    v(\wt \uv ,{\bf 0})\leq  v({\bf 0} ,{\bf 0}) = \fopt \leq f(\wt \xv). 
\end{align*}
Thus, ~\eqref{eq: beck proof rel 1} implies that
\begin{align}
    f(\wt \xv) - \fopt \geq \inner{-\lambda^*,\wt \uv}. \label{eq: beck proof rel 2}
\end{align}
We obtain the following relations.
\begin{align*}
     (\rho - \norm{\lambda^*}_1)\norm{\wt\uv}_\infty &=-\norm{\lambda^*}_1\norm{\wt\uv}_\infty + \rho \norm{\wt\uv}_\infty\\
     &\leq \inner{-\lambda^*,\wb\uv}+ \rho \norm{\wt\uv}_\infty\\
     & = f(\wt \xv) - \fopt +\rho \norm{\wb\uv}_\infty\leq \delta,
\end{align*}
where the last relation holds by~\eqref{eq: beck proof rel 2}. Rearranging, we get
\begin{align*}
    \norm{\brs*{\gv(\wt\xv)}_+}_\infty = \norm{\wb\uv}_\infty\leq \frac{\delta}{\rho - \norm{\lambda^*}_1}\leq \frac{2}{\rho}\delta,
\end{align*}
by using the assumption $2\norm{\lambda^*}_1 \leq \rho$.
\end{proof}

Lastly, we have the following useful result by which we can bound the optimal dual parameter by the properties of a slater point. This result is an adjustment of \citealt{beck2017first}, Theorem 8.42.

\begin{theorem}\label{theorem: thoeorem slater point beck}
Let $\wb \xv \in X$ be a point satisfying $\gv(\wb x)<{\bf 0}$ and $\A\wb\xv+\bv ={\bf 0}$. Then, for any $\lambda,\mu\in\brc*{\lambda\in \mathbb{R}^m_+, \mu\in \mathbb{R}^p_+: q(\lambda,\mu)\geq M}$
\begin{align*}
    \norm{\lambda}_1 \leq \frac{f(\wb x)-M}{\min_{j=1,..,m} -g_j(\wb x)}.
\end{align*}
\end{theorem}
\begin{proof}
Let 
\begin{align*}
    S_M = \brc*{\lambda\in \mathbb{R}^m_+, \mu\in \mathbb{R}^p_+: q(\lambda,\mu)\geq M}.
\end{align*}
By definition, for any $\lambda,\mu\in S_M$ we have that
\begin{align*}
    &M \leq q(\lambda,\mu) \\
    & = \min_{x\in X}\brc*{f(\xv) + \lambda^T \gv(\xv) + \mu^T(\A\xv +\bv)}\\
    &\leq f(\wb \xv) + \lambda^T \gv(\wb\xv) + \mu^T(\A\wb\xv +\bv)\\
    & =  f(\wb \xv) + \sum_{j=1}^m\lambda_j g_j(\wb\xv).
\end{align*}
Therefore,
\begin{align*}
      -\sum_{j=1}^m\lambda_j g_j(\wb\xv)\leq f(\wb \xv) - M,
\end{align*}
which implies that for any $(\lambda,\mu)\in S_M$
\begin{align*}
    \sum_{j=1}^m \lambda_j = \norm{\lambda}_1 \leq \frac{f(\wb \xv) - M}{\min_{j=1,..,m} (-g_j(\wb \xv))}.
\end{align*}
\end{proof}

From this theorem we get the following corollary.
\begin{corollary}\label{corollary: bound on a dual variable}
Let $\wb \xv \in X$ be a point satisfying $\gv(\wb x)<{\bf 0}$ and $\A\wb\xv+\bv ={\bf 0}$, and$\lambda^*$ be an optimal dual solution. Then,
\begin{align*}
    \norm{\lambda^*}_1 \leq \frac{f(\wb x)-M}{\min_{j=1,..,m} -g_j(\wb x)}
\end{align*}
\end{corollary}
\begin{proof}
Since $(\lambda^*,\mu^*)\in S_{\fopt}$ be an optimal solution of the dual problem~\eqref{supp: convex cos problem dual}.
\end{proof}

\end{document}